\documentclass{article}

\usepackage{microtype}
\usepackage{graphicx}
\usepackage{subcaption}
\usepackage{booktabs} % for professional tables

\usepackage{hyperref}

\usepackage[preprint]{icml2026}

\usepackage{amsmath}
\usepackage{amssymb}
\usepackage{mathtools}
\usepackage{amsthm}

\usepackage[capitalize,noabbrev]{cleveref}

\theoremstyle{plain}
\newtheorem{theorem}{Theorem}[section]
\newtheorem{proposition}[theorem]{Proposition}

\theoremstyle{definition}
\newtheorem{definition}[theorem]{Definition}

\theoremstyle{remark}
\newtheorem{remark}[theorem]{Remark}
\theoremstyle{remark}
\newtheorem{example}[theorem]{Example}

\usepackage[textsize=tiny]{todonotes}

\icmltitlerunning{Stepwise Variational Inference with Vine Copulas}

\usepackage[utf8]{inputenc} % allow utf-8 input
\usepackage[T1]{fontenc}    % use 8-bit T1 fonts
\usepackage{hyperref}       % hyperlinks
\usepackage{url}            % simple URL typesetting
\usepackage{booktabs}       % professional-quality tables
\usepackage{amsfonts}       % blackboard math symbols
\usepackage{nicefrac}       % compact symbols for 1/2, etc.
\usepackage{microtype}      % microtypography
\usepackage[dvipsnames]{xcolor}         % colors

\usepackage{tikz}
\usetikzlibrary{matrix}
\usetikzlibrary{shapes, arrows, calc, arrows.meta, fit, positioning}

\usepackage{graphicx}
\usepackage{placeins}
\usepackage{bm}
\newcommand{\B}{\boldsymbol}
\newcommand{\Bb}{\mathbf}

\begin{document}

\twocolumn[
  \icmltitle{Stepwise Variational Inference with Vine Copulas}



  \icmlsetsymbol{equal}{*}

  \begin{icmlauthorlist}
      \icmlauthor{Elisabeth Griesbauer}{nr,uiocbe,intg}
      \icmlauthor{Leiv Rønneberg}{uiomat}
      \icmlauthor{Arnoldo Frigessi}{uiocbe,intg}
      \icmlauthor{Claudia Czado}{tum,mucd}
      \icmlauthor{Ingrid Hobæk Haff}{uiomat,intg}
    % \icmlauthor{Firstname1 Lastname1}{equal,yyy}
    % \icmlauthor{Firstname2 Lastname2}{equal,yyy,comp}
    % \icmlauthor{Firstname3 Lastname3}{comp}
  \end{icmlauthorlist}

  \icmlaffiliation{uiocbe}{Oslo Centre for Biostatistics and Epidemiology (OCBE), University of Oslo, Norway}
  \icmlaffiliation{uiomat}{Department of Mathematics, University of Oslo, Norway}
  \icmlaffiliation{nr}{Norwegian Computing Center, Oslo, Norway}
  \icmlaffiliation{intg}{Integreat - Norwegian Centre for Knowledge-driven Machine Learning, Oslo, Norway}
  \icmlaffiliation{tum}{Technical University of Munich, Germany}
  \icmlaffiliation{mucd}{Munich Data Science Institute, Germany}

  \icmlcorrespondingauthor{Elisabeth Griesbauer}{elismg@uio.no}

  \icmlkeywords{Machine Learning, ICML}

  \vskip 0.3in
]



\printAffiliationsAndNotice{ }  % no special notice (required even if empty)

\begin{abstract}
    We propose stepwise variational inference (VI) with vine copulas: a universal VI procedure that combines vine copulas with a novel stepwise estimation procedure of the variational parameters. 
    Vine copulas consist of a nested sequence of trees built from copulas, where more complex latent dependence can be modeled with increasing number of trees. 
    We propose to estimate the vine copula approximate posterior in a stepwise fashion, tree by tree along the vine structure. 
    Further, we show that the usual backward Kullback-Leibler divergence cannot recover the correct parameters in the vine copula model, thus the evidence lower bound is defined based on the R\'enyi divergence. 
    Finally, an intuitive stopping criterion for adding further trees to the vine eliminates the need to pre-define a complexity parameter of the variational distribution, as required for most other approaches. 
    Thus, our method interpolates between mean-field VI (MFVI) and full latent dependence. 
    In many applications, in particular sparse Gaussian processes, our method is parsimonious with parameters, while outperforming MFVI.

\end{abstract}

\begin{figure}[!t]
    \centering
    \includegraphics[width=\linewidth]{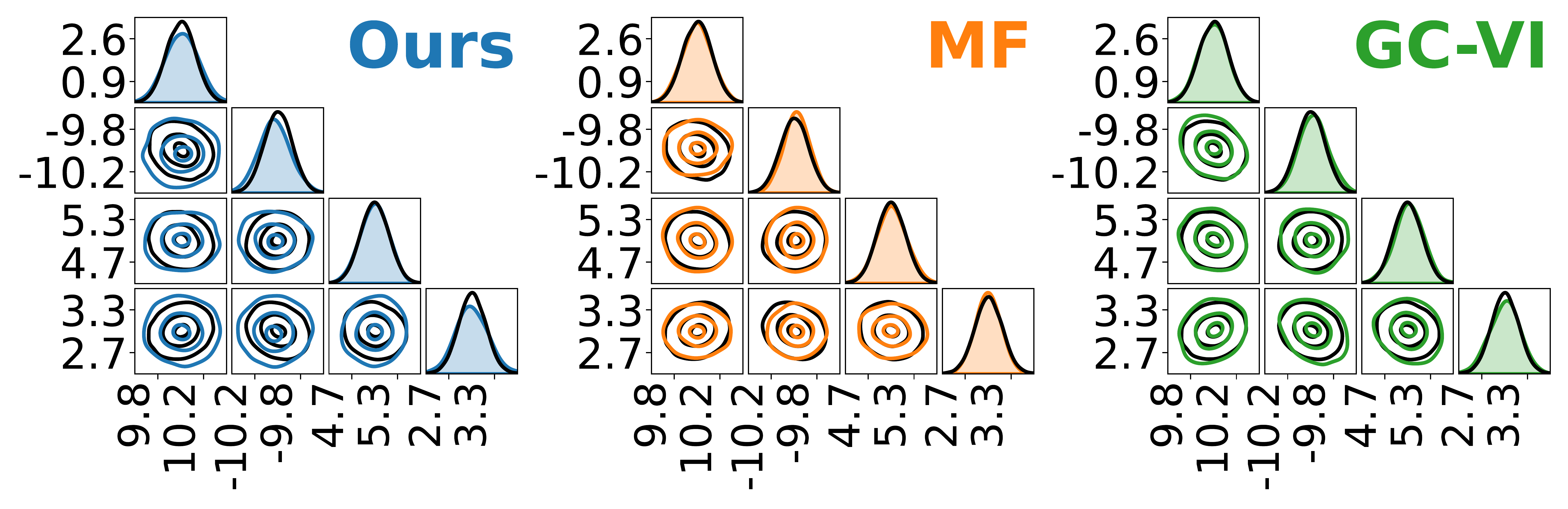}
    \includegraphics[width=\linewidth]{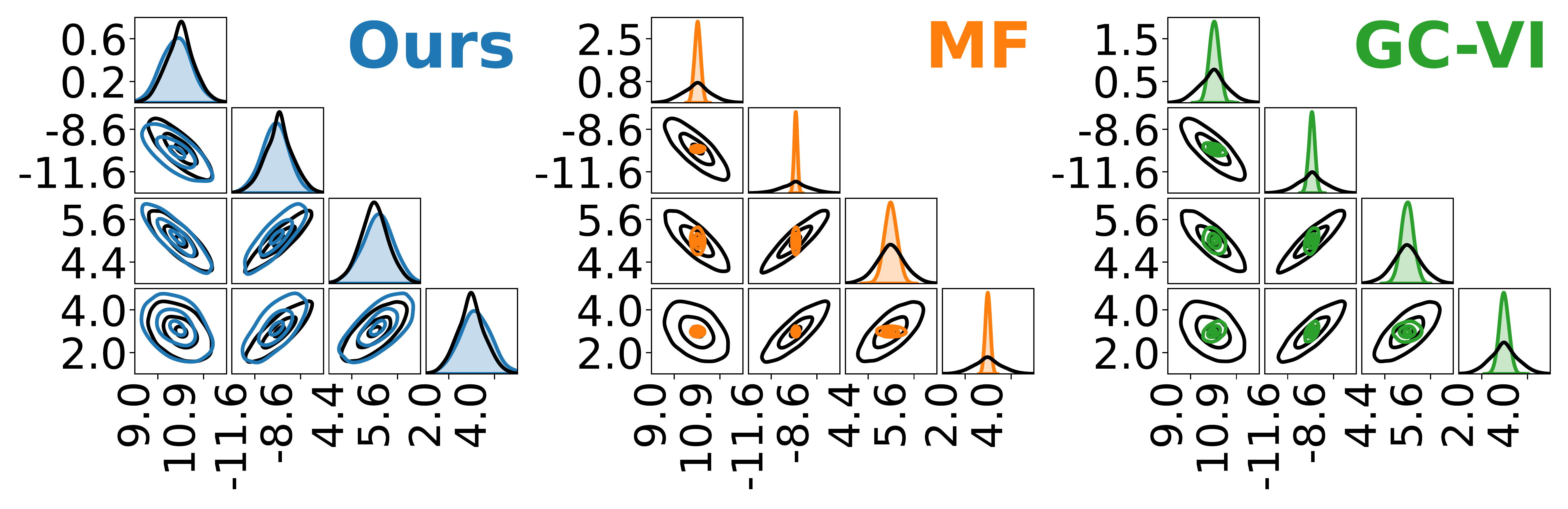}
    \caption{Contour plots of samples from NUTS (black) regarded as ground truth and variational approximations obtained with stepwise VI with vines (blue), MFVI (orange) and GC-VI (green).}
    \label{fig:first_page}
\end{figure}

\section{Introduction}
% general on VI
Variational inference allows approximate posterior inference by framing it as an optimization problem. It is most useful when sampling based methods struggle or fail. The candidate distribution is found using stochastic optimization from a tractable, parametric distribution family, that lies closest to the true posterior with respect to a divergence. 
% choice of Q
The choice of the variational family is critical -- its flexibility and complexity determines how well the true posterior can be approximated.
% rank of covariance matrix/dependence structure
The complexity of variational distributions ranges from fully factorized MFVI, structured MF approaches, variational distributions with sparse (low-rank + diagonal) or full covariance matrices, and (un-)truncated vine copulas. In each case, a hyperparameter of the variational distribution, which determines the complexity of the variational distribution, needs to be pre-defined : the structure of the structured MF \citep{saul1995exploiting, fu2025vector}, the sparsity of the covariance matrix of the variational distribution \citep{smith2020high,smith2023implicit}, the truncation level of the vine copula \citep{tran2015copula, chi2022fast}, the number of components in mixture variational distribution \citep{guo2016boosting, campbell2019universal, gunawan2024flexible} or the number of flow parameters in normalizing flow based VI \citep{rezende2015variational}. Without any prior knowledge, this is hard. It either gives a too simplistic variational model that fails to capture important aspects of the true posterior, or an over-parametrized variational model, potentially making the optimization more costly than necessary. 
% different variational distribution families

We propose a variational family that is expressive while parsimonious, and eliminates the need to pre-specify a complexity hyperparameter. Our approach combines a vine copula distribution as the approximate posterior with a stepwise estimation of the variational parameters along the sequence of vine copula trees and a global stopping criterion, that automatically selects the complexity of the variational distribution on the fly. We term this approach \textit{stepwise VI with vine copulas}.

% vines
Vine copulas are highly flexible models that model joint dependence and marginal behavior separately. For this reason, they have gained popularity both within and outside VI \citep{tran2015copula, chi2022fast, tagasovska2019copulas, tagasovska2023retrospective, huk2024quasi}. Vine copulas are built from copulas in a structure that is graphically represented by a nested sequence of trees: each edge in a tree corresponds to a copula which captures the (conditional) dependence between a pair of variables. Consequently, a vine with more trees, and thus more copulas, can model more complex dependence.
% stepwise estimation
Due to their nested tree structure, the standard approach for estimating vine copulas from observed data is to proceed tree by tree using stepwise maximum likelihood estimation \citep{dissmann2013selecting}. Existing vine based VI approaches estimate the vine parameters of all trees simultaneously. 
We propose a novel procedure for estimating the parameters of the vine approximate posterior along its tree sequence, which is natural from its structure. The stepwise estimation is assisted by a natural global stopping criterion for adding further trees to the vine approximate posterior. If all copulas of the current tree are close to independence, no further trees are added. This stopping criterion automatically chooses the complexity parameter of the vine copula, i.e. its number of trees, on the fly. Thus, stepwise VI with vines gives an expressive and flexible approximate posterior, that is parsimonious in the number of variational parameters and represents a compromise between MFVI and full-rank variational models.
% Renyi-divergence instead of backward KL
Finally, we show that an evidence lower bound based on the commonly used backward Kullback-Leibler (KL) divergence cannot recover the correct parameters in the vine copula approximate posterior. For this reason we optimize the variational parameters using a R\'enyi divergence based lower bound.

% % advantages
% This gives a major advantage over other (copula-based) variational models: there is no need to pre-define the complexity parameter of the variational model. Instead this is automatically inferred in form of the truncation level of the vine copula using the intuitive and interpretable global stopping criterion in the stepwise estimation procedure. Thus stepwise VI with vines avoids estimating more variational parameters than needed for good posterior approximation. Additionally, the stepwise estimation avoids estimating all variational parameters simultaneously which is computationally more demanding.

\paragraph{Related Work}\label{sec:related_work}
The idea of using vine copulas as approximate posteriors in VI has been explored before. \citet{tran2015copula} and \citet{chi2022fast} alternate between optimizing the MF parameters and optimizing the vine copula parameters until convergence. For reduced sampling cost, \citet{chi2022fast} additionally propose to formulate the ELBO gradients as an expectation over the MF. In these approaches all pair copula parameters are updated simultaneously and numerous alternating steps are needed for convergence. In their experiments \citet{tran2015copula} explore the route of learning the vine tree structure and pair copula families from synthetic data of the latent variables, but abandon this idea.\footnote{Either there is a high-quality variational posterior available from which we could generate synthetic data and meaningfully learn the vine tree structure, but then there is no need to estimate a vine copula variational posterior, or one only has access to a low-quality variational posterior, resulting in synthetic data that are useless for vine structure learning.} 
Consequently, they have to fix the truncation level. \citet{han2016variational} propose a $d$-dimensional Gaussian copula on latent variables transformed to normality with Bernstein polynomials. Their approach requires the estimation of the full correlation matrix or its Cholesky decomposition as variational parameters of the Gaussian copula. 
Similarly, \citet{smith2020high} and \citet{smith2023implicit} approximate transformed latent variables with implicit copulas defined by normal, skew-normal and elliptical distributions. They set the structure of the corresponding covariance matrix to a factor structure (low-rank + diagonal) and pre-define the rank as a complexity parameter.
\citet{gunawan2024flexible} combine ideas from variational boosting \citep{guo2016boosting, miller2017variational} and implicit (skew-)Gaussian copulas as variational models \citep{smith2020high}. They model the transformed latent variables with a $K$-mixture of multivariate Gaussians with correlation matrix in factor structure, pre-specifying the number of mixing components and the rank of the covariance matrices within each mixture component. \citet{fu2025vector} model block-dependence in style of structured MF with a vector copula, updating all vector copula parameters simultaneously and pre-specifying the number of blocks and their dimension.
An extended literature review can be found in Appendix \ref{app:extended_literature_review}.

\paragraph{Contributions}
Our contributions are as follows:
\begin{itemize}
    \item We propose \textit{stepwise VI with vine copulas}, a VI procedure based on vine copulas as a flexible and expressive variational family, and a novel stepwise estimation procedure of the variational parameters along the tree sequence of the vine copula. 
    \item \textit{Stepwise VI with vines} automatically infers the complexity parameter of the variational family by using an intuitive global stopping criterion for adding further trees to the vine. This yields a parsimonious \textit{and} expressive approximate posterior, and eliminates the need to pre-define the complexity parameter of the variational distribution. 
    \item Theoretically, we show that an evidence lower bound based on the backward KL divergence cannot recover the correct parameters in the vine copula approximate posterior and therefore optimize the variational parameters using a R\'enyi divergence based lower bound.
\end{itemize}

\section{Background}

\subsection{Variational Inference (VI)}

Let $\mathbf{Z} \in \mathbb{R}^d$ be a latent random vector and $\mathbf{x}$ be realizations of the random vector $\mathbf{X} \in \mathbb{R}^p$. Let $\pi(\mathbf{z})$ be the prior and $p(\mathbf{x} | \mathbf{z})$ the likelihood, which both can be evaluated and sampled from. 
The core idea of VI is to approximate the true posterior $p(\mathbf{z} | \mathbf{x})$ with a variational distribution $q(\mathbf{z}; \bm{\phi})$\footnote{The notation $q( \cdot; \bm{\phi})$ and $q_{\bm{\phi}}(\cdot)$ will be used exchangeably. Notational conventions can be found in Appendix \ref{app:notation}.} where $q$ is an element of some tractable distribution family $\mathcal{Q}$, parametrized by $\bm{\phi} \in \Phi$. The approximate posterior $q( \cdot; \bm{\phi}^*)$ is obtained by minimizing, some divergence, most commonly the KL-divergence, from the variational distribution to the true posterior:
\begin{align}\label{eq:argmin_KL_VI}
    q(\mathbf{z}; \bm{\phi}^*) := \arg \min_{q \in \mathcal{Q}} KL\big( q(\mathbf{z}; \bm{\phi}) || p(\mathbf{z} | \mathbf{x}) \big)
\end{align}
Minimizing $KL\big( q(\mathbf{z}; \bm{\phi}) || p(\mathbf{z} | \mathbf{x}) \big)$ is equivalent to maximizing the the evidence lower bound (ELBO):
\begin{align}\label{eq:KL_ELBO}
    \mathcal{L}(\bm{\phi}) &:= E_{q_{\bm{\phi}}}\big[ \log p(\mathbf{z}, \mathbf{x}) - \log q(\mathbf{z}; \bm{\phi}) \big] \; 
\end{align}

usually obtained with stochastic gradient descent (SGD) on $-\mathcal{L}(\bm{\phi})$.

A standard choice for $\mathcal{Q}$ is MFVI \citep{anderson1987mean, jordan1999introduction},  $q(\mathbf{z}; \bm{\phi}) = \prod_{j = 1}^d q_j(z_j; \phi_j)$, which is fast to optimize but does not allow any latent dependence. 
% \eg{Add? Minimum KL divergence any approximate posterior distribution without dependence to the true posterior can achieve is the KL between the independence copula and the copula of the true posterior \citep{li2025bayesian}.}

KL-divergence based VI, optimizing the objective in \eqref{eq:KL_ELBO}, suffers from several drawbacks, e.g. zero-forcing behavior. We prove in Section \ref{sec:stepwiseVIvines} that it cannot recover the correct parameters in a vine copula, and for this reason we instead make use of R\'enyi $\alpha$-divergence based VI.

\subsection{R\'enyi $\alpha$-divergence and VI}\label{sec:Renyi-divergence_VI}

\paragraph{R\'enyi $\alpha$-divergence}

The R\'enyi $\alpha$-divergence is defined as:
\begin{align*}
    R_{\alpha}(q||p) := \frac{1}{\alpha-1} \log \Big( \int q(z)^{\alpha} p(z)^{1-\alpha} dz  \Big)
\end{align*}
for $\alpha \in (0, 1) \cup (1, \infty)$ \citep{van2014renyi} and can be extended to:
\begin{align*}
    &\lim_{\alpha \rightarrow 1} R_{\alpha}(q||p) =: R_1(q||p) = KL(q||p) \; ,\\
    &R_0(q||p) = - \log \Big( \int_{q(z) > 0} p(z) dz \Big) \; .
\end{align*}
% Additionally, the R\'enyi $\alpha$-divergence has skew symmetry for $\alpha \in (0,1)$:
% \begin{align}
%     &R_{\alpha}(q||p) = \frac{a}{1-\alpha} R_{1 - \alpha}(p||q)
% \end{align}
% and is jointly convex in $(q,p)$ for $\alpha \in [0,1]$. 

\citet{li2016renyi} propose a R\'enyi $\alpha$-divergence based VI framework  by optimizing the variational R\'enyi bound (VR):
\begin{align*}
    & \text{VR}^{(\alpha)}(\bm{\phi}; \mathbf{x}):= \frac{1}{1 - \alpha } \log \int q(\mathbf{z}; \bm{\phi})^{\alpha} p(\mathbf{x}, \mathbf{z})^{1 - \alpha} dz
\end{align*}
that lower bounds the log evidence through:
\begin{align}
    &R_{\alpha}(q(\mathbf{z}; \bm{\phi}) || p(\mathbf{z} | \mathbf{x})) \nonumber  \\
    =& \frac{1}{\alpha -1} \log  \int q(\mathbf{z}; \bm{\phi})^{\alpha} \Big( \frac{p(\mathbf{x}, \mathbf{z})}{p(\mathbf{x})}\Big)^{1 - \alpha} dz \nonumber \\
    =& \frac{1}{\alpha -1} \log \int q(\mathbf{z}; \bm{\phi})^{\alpha} p(\mathbf{x}, \mathbf{z})^{1 - \alpha}  dz + \log \{p(\mathbf{x})\} \; . \nonumber
\end{align} 
The parameter $\alpha$ allows control over the amount of weight put on the true posterior,  overcoming drawbacks of the backward KL-divergence $KL(q||p)$ \citep{daudel2023alpha}.

% The VR bound possibly provides a tighter lower bound on the log evidence than the ELBO according to Theorem 1 in \citet{li2016renyi}, which states that:
% \begin{align}
%     \mathcal{L}_{KL}(q; x) = \lim_{\alpha' \rightarrow 1} VR_{\alpha'}(q; x) \leq VR^{(\alpha)}(q; x) \leq \log p(\mathbf{x}) \; ,
% \end{align}
% for $\alpha \in (0,1)$ and $supp(p(z|x)) \subset supp(q(z))$. 

\paragraph{R\'enyi divergence VI and VR-IWAE bound}

The Monte-Carlo (MC) estimator for the VR bound and its gradients proposed by \citet{li2016renyi} is biased for all $\alpha \notin \{0,1\}$, but shown to work well empirically. However, SGD with a learning rate sequence fulfilling the Robbins-Monro conditions is guaranteed to converge to its optimum only for unbiased estimators of the evidence lower bound gradients \citep{robbins1951stochastic}.

\citet{daudel2023alpha} show that the expectation of the biased VR bound gradient estimator in \citet{li2016renyi} can be used as a variational lower bound itself. And that this represents a generalization of the importance weighted auto-encoder (IWAE) bound of \citet{burda2015importance}, based on $N$ importance samples. They term it the VR-IWAE bound and define it to be:
\begin{align}\label{eq:VR_IWAE_def}
    &l_N^{(\alpha)}(\bm{\phi}; \mathbf{x}) := \frac{1}{1 - \alpha} \int \prod_{i=1}^N \nonumber \\
    & \quad q(\mathbf{z}_i; \phi) \log \Bigg( \frac{1}{N} \sum_{k=1}^N  \Big(\frac{p(\mathbf{x},\mathbf{z}_j)}{q(\mathbf{z}_j; \phi)} \Big)^{1 - \alpha} \Bigg) d\mathbf{z}_{1:N} \; .
\end{align}

Applying the reparametrization trick (see Section \ref{subsec:reparamtrick}) to the VR-IWAE bound yields the same SGD procedure as the reparametrized VR bound \citep{daudel2023alpha}. This means that existing implementations of the VR bound gradient estimators, e.g. in \verb|pyro| \citep{bingham2018pyro}, provide an unbiased estimator for VR-IWAE gradients.

\paragraph{VR-IWAE bound and $\alpha$}

The VR-IWAE recovers the ELBO for $N=1$ and $\alpha \rightarrow 1$, and the IWAE bound for $\alpha=0$. The VR-IWAE can be expressed \citep{daudel2023alpha, daudel2024learning} as:
\begin{gather*}
    l_N^{(\alpha)}(\bm{\phi}; \mathbf{x}) = \text{VR}^{(\alpha)}(\bm{\phi}; \mathbf{x}) - \frac{\gamma^{(\alpha)}(\bm{\phi}; \mathbf{x})^2}{2N} + o\Big(\frac{1}{N} \Big) \; , \\
    \text{where } \gamma^{(\alpha)}(\bm{\phi}; \mathbf{x})^2 := (1 - \alpha) \text{Var}_{\mathbf{Z} \sim q_{\phi}} (\bar{w}^{(\alpha)}_{\bm{\phi}}(\mathbf{Z})) \\
    \bar{w}^{(\alpha)}_{\bm{\phi}}(\mathbf{Z}) := w_{\bm{\phi}}(\mathbf{Z})^{1 - \alpha}/E_{\mathbf{Z} \sim q_{\phi}} [ w_{\bm{\phi}}(\mathbf{Z})^{1 - \alpha}] \\
    w_{\bm{\phi}}(\mathbf{z}) := p(\mathbf{x}, \mathbf{z})/q(\mathbf{z}; \bm{\phi}).
\end{gather*}
This gives two things: Firstly, the VR-IWAE bound converges to the VR bound at a rate of $1/N$, and secondly this gives a decomposition of the VR-IWAE bound into a bias and a variance term which depend on $\alpha$. For $\alpha \rightarrow 1$ the variance term vanishes, bringing the VR-IWAE closer to the VR bound, while the latter at the same time converges to the backward KL-divergence based ELBO, with drawbacks we want to overcome. 
On the other hand, a value of $\alpha$ closer to 0 puts more weight on the true posterior in the VR bound due to the skew symmetry of the R\'enyi $\alpha$-divergence \citep{li2016renyi}, which is favorable.
How fast $\gamma^{(\alpha)}(\bm{\phi}; \mathbf{x})^2/2N$ goes to 0 depends on the behavior of $\gamma^{(\alpha)}(\bm{\phi}; \mathbf{x})^2$, which is not straight-forward to quantify. 
This encourages to find a trade-off based on $\alpha$ for good empirical performance \citep{daudel2023alpha}.
% The presence of $\gamma^{(\alpha)}(\bm{\phi}; \mathbf{x})^2$ in the VR-IWAE bound might decrease this effect. Even though it might be possible to compute or estimate $\gamma^{(\alpha)}(\bm{\phi}; \mathbf{x})^2$, neither the VR bound nor the VR-IWAE bound have a unified scale. It is therefore not possible to quantify how far $\gamma^{(\alpha)}(\bm{\phi}; \mathbf{x})^2$ brings us away from $VR^{(\alpha)}(\bm{\phi}; \mathbf{x})$. Additionally, $\gamma^{(\alpha)}(\bm{\phi}; \mathbf{x})^2/2N$ can be controlled by increasing $N$.
\citet{daudel2024learning} give more refined results in similar fashion on the gradient level. \citet{margossian2024variational} analyze R\'enyi divergence VI in a Gaussian setting, which we further discuss in Appendix \ref{app:Renyi_divergence_Gaussian_setting}.

\subsection{Reparametrization trick}\label{subsec:reparamtrick}
Lower bound gradient estimators typically suffer from high variance that can limit their practical applicability. Gradient estimators based on the reparametrized lower bound \citep{kingma2013auto, rezende2014stochastic} exhibit lower variance.
If applicable, the latent variable $\mathbf{z}$ is expressed as a deterministic, differentiable transformation $\mathbf{z} := g(\bm{\epsilon}, \bm{\phi})$ of some random variable $\bm{\epsilon} \sim q(\bm{\epsilon})$. 
The reparametrized gradient estimator of the VR-IWAE bound is given by \citep{daudel2023alpha}:
\begin{align*}
    &\nabla_{\bm{\phi}} l_N^{(\alpha)}(\bm{\phi}; \mathbf{x}) = \int \prod_{i=1}^N q(\bm{\epsilon}_i) \cdot  \Bigg( \sum_{j=1}^N  \nonumber \\
    & \; \frac{w_{\bm{\phi}}\big(g(\bm{\epsilon}_j, \bm{\phi})\big)^{1 - \alpha}}{\sum_{k=1}^N w_{\bm{\phi}}\big(g(\bm{\epsilon}_k, \bm{\phi})\big)^{1 - \alpha}} \nabla_{\bm{\phi}} \log w_{\bm{\phi}}\big(g(\bm{\epsilon}_j, \bm{\phi})\big) \Bigg) d\bm{\epsilon}_{1:N} \; ,
\end{align*}
where $w_{\bm{\phi}}(\mathbf{z}) := p(\mathbf{x},\mathbf{z})/q(\mathbf{z}; \bm{\phi})$\footnote{Here we have left out the subscript $\theta$ in $w_{\phi}(z)$ as the parameter of the model $p(x,z)$ is assumed to be a constant.}, and its unbiased estimator given by:
\begin{align}\label{eq:vr_iwae_bound_grad_estimator}
    \widehat{\nabla_{\bm{\phi}} l_N^{(\alpha)}(\bm{\phi}; \mathbf{x})} &:= \sum_{j=1}^N \frac{w_{\bm{\phi}}\big(g(\bm{\epsilon}_j, \bm{\phi})\big)^{1 - \alpha}}{\sum_{k=1}^N w_{\bm{\phi}}\big(g(\bm{\epsilon}_k, \bm{\phi})\big)^{1 - \alpha}} \nonumber \\
    & \quad \cdot \nabla_{\bm{\phi}} \log w_{\bm{\phi}}\big(g(\bm{\epsilon}_j, \bm{\phi})\big) \; .
\end{align}

\subsection{Vine Copulas}
A $d$-dimensional copula $C:[0,1]^d \rightarrow [0,1]$ is a $d$-dimensional distribution on the unit cube with uniform marginals and (if existing) corresponding copula density $c$. \citet{sklar1959fonctions} shows that any $d$-dimensional distribution $F$ can be expressed in terms of a $d$-dimensional copula $C$:
\begin{align*}
    F(x_1, \dots, x_d) = C\big(F_1(x_1), \dots, F_d(x_d)\big) \; .
\end{align*}
If all densities exist, a $d$-dimensional density $f$ can be expressed as a product of the corresponding $d$-dimensional copula density $c$ and the $d$ marginal densities: 
\begin{align}\label{eq:Sklar_density}
    f(x_1, \dots, x_d) &= c\big(F_1(x_1), \dots, F_d(x_d)\big)  \nonumber \\
    &\; \cdot f_1(x_1) \cdot ... \cdot f_d(x_d) \; .
\end{align}
Together with the fact that a copula uniquely describes dependence of random variables \citep{geenens2023towards}, this allows completely separate modeling of marginal behavior and joint dependence, which gives a recipe for building highly flexible models.
There are different parametric copula families $\{c(\cdot, \cdot; \eta) : \; \eta \in H\}$ that model different types of dependence, e.g. upper or lower tail dependence, both or none. However, it is hard to estimate a $d$-dimensional copula, and one is limited to the dependence type inherent to the corresponding copula family. Using density factorization combined with Sklar's theorem \citep{sklar1959fonctions}, a $d$-dimensional copula $c$ can be deconstructed into a product of bivariate (conditional) copulas, so called pair copulas.

A vine copula \citep{joe1997multivariate, bedford2001probability, bedford2002vines, aas2009pair, joe2014dependence, czado2019analyzing} is a probabilistic model built on the idea of reversing the copula decomposition, constructing flexible $d$-dimensional distributions from univariate marginals and bivariate (conditional) copulas. The vine tree structure $\mathcal{V} = (T_1, \dots, T_{d-1})$, is a nested sequence of $d-1$ trees $T_k = (V_k, E_k), \; k \in [d-1]$, which serves as a construction plan of the vine copula.  
An edge in $T_1$ represents a bivariate copula $c_{a_e, b_e}$ of the unconditional pair of random variables $(X_{a_e}, X_{b_e}), \; a_e, b_e \in [d]$, and an edge $e$ in $T_k, \; k \in \{2, \dots, d-1\}$ represents a bivariate copula $c_{a_e, b_e ; D_e}$ of a pair $(X_{a_e}, X_{b_e})$, conditioned on $k-1$ random variables $X_j, \; j \in D_e \subset [d]$. 
Hence, the vine copula is: 
\begin{align}\label{eq:vine}
    c = \prod_{k \in [d-1]} \prod_{e \in E_k} c_{a_e, b_e ; D_e} \; ,
\end{align}
where we left out arguments and the pair copula parameters $\eta$ for notational ease.
A way to simplify a vine copula is to truncate it at a specific tree level $\tau \in [d-1]$. This is equivalent to setting all pair copulas of trees $T_{\tau+1}, \dots, T_{d-1}$ to independence. 

\begin{definition}[Truncation of the Vine Copula at Level $\tau$]
    Let $c$ be a vine copula as given in Equation \eqref{eq:vine}. We define the vine copula truncated at truncation level $\tau \in [d-1]$ as: $\prod_{t \in [\tau]} \prod_{e \in E_k} c_{a_e, b_e ; D_e}$.
\end{definition}

Thus, in the resulting vine copula, only trees $T_1, ... T_{\tau}$ are left in the model.  For $\tau=d-1$, we obtain the un-truncated vine copula, while for $\tau=1$, only the first tree is retained.
% D-vine
Special shapes of trees in the vine tree structure lead to certain sub-classes of vines. In particular, in a \textit{D-vine}, each tree is a path, i.e. the degree of all nodes in all trees it smaller than or equal to 2. An illustration of a D-vine can be found in Figure \ref{picture:dvine4_notation}.
% order
The vine tree structure together with the order, in which the random variables appear in each tree $T_t$, determines which pairs of random variables (conditioned on other random variables) are modeled with a copula in the vine. Figure \ref{picture:dvine4_notation} depicts a D-vine tree sequence on 4 random variables with fixed order.

\vspace{0.2cm}
\begin{figure}[H]
    % \centering
    \begin{tikzpicture}[font = \footnotesize, node distance=5mm, main/.style = {draw, shape = circle}, baseline=+6mm]
        
        \node[draw=none,fill=none] at (0,0) {(T1)};
        \node[main] (1) at (1, 0) {1};
        \node[main] (2) at (3, 0) {2};
        \node[main] (3) at (5, 0) {3};
        \node[main] (4) at (7, 0) {4};
        \draw (1) -- (2) node[midway, above] {$1, 2$};
        \draw (2) -- (3) node[midway, above] {$2, 3$};
        \draw (3) -- (4) node[midway, above] {$3, 4$};
    \end{tikzpicture}

\vspace{0.3cm}

    \begin{tikzpicture}[font = \footnotesize, node distance=5mm, main/.style = {draw}, state/.style ={ellipse, draw, minimum width = 0.8 cm}, baseline=+6mm]
    
        \node[draw=none,fill=none] at (0,0) {(T2)};
        \node[state] (1) at (1.5, 0) {$1, 2$};
        \node[state] (2) at (4, 0) {$2, 3$};
        \node[state] (3) at (6.5, 0) {$3, 4$};
        \draw (1) -- (2) node[midway, above] {$1, 3; 2$};
        \draw (2) -- (3) node[midway, above] {$2, 4; 3$};
    \end{tikzpicture}
 
 \vspace{0.3cm}
 
    \begin{tikzpicture}[font = \footnotesize, node distance=5mm, main/.style = {draw}, state/.style ={ellipse, draw, minimum width = 0.8 cm}, baseline=+6mm]
    
        \node[draw=none,fill=none] at (0,0) {(T3)};
        \node[state] (1) at (2.3, 0) {$1, 3; 2$};
        \node[state] (2) at (5.7, 0) {$2, 4; 3$};
        \draw (1) -- (2) node[midway, above] { $1, 4; 2, 3$ };
    \end{tikzpicture}
    \caption{A D-vine tree sequence on 4 elements.}
    \label{picture:dvine4_notation}
\end{figure}
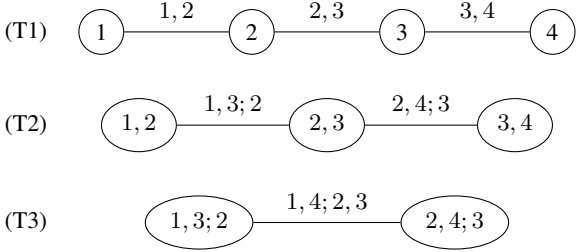

We write $c^t(\mathbf{u}; \bm{\eta})$ as a shorthand for the product of pair copulas belonging to tree $t$ in the D-vine, $c^t(\mathbf{u}(\mathbf{z}; \bm{\lambda}, \bm{\eta}^1, ..., \bm{\eta}^{t-1}); \bm{\eta}^t)$, and denote a $\tau$-truncated D-vine copula as $\prod_{t=1}^{\tau} c^t(\mathbf{u}; \bm{\eta}^t)$. Note that $c^t(\mathbf{u}; \bm{\eta})$ can be written:
\begin{align*}
     &\prod_{j=1}^{d-t} 
      c_{j, j+t; j+1 : j+t-1} \Big( F_{j|j+1 : j+t-1}(z_j | \mathbf{z}_{j+1 : j+t-1}), \nonumber  \\
    &\;  \cdot F_{j+t|j+1 : j+t-1}(z_{j+t} | \mathbf{z}_{j+1 : j+t-1}) ; \eta_{j, j+t; j+1 : j+t-1} \Big) \; ,
\end{align*}
where $\bm{\eta}^t=[\eta_{1, t+1; 2 : t},\ldots, \eta_{d-t, d; d-t+1 : d-1}]^T \in\mathbb{R}^{d-t}$ is the vector of the parameters of pair copulas in tree $t$. $\mathbf{u}=[\bm{F}_{1, t+1|2 : t}^T,\ldots,\bm{F}_{d-t,d|d-t+1 : d-1}^T]^T\in \mathbb{R}^{(d-t)\times 2} $ is the matrix of input data, so-called copula data, to the copulas of tree $t$, with the $j$-th row containing the entries $F_{j|j+1 : j+t-1}(z_j | \mathbf{z}_{j+1 : j+t-1})$ and $F_{j+t|j+1 : j+t-1}(z_{j+t} | \mathbf{z}_{j+1 : j+t-1})$.
For an extended introduction to vine copulas please consult Appendix \ref{app:intro_to_vines}.

\section{Stepwise Variational Inference with Vine Copulas}\label{sec:stepwiseVIvines}
We propose a D-vine distribution as the variational model:
\begin{align*}
    q(\mathbf{z} ; \bm{\lambda}, \bm{\eta}) &:= \prod_{j=1}^d q_j(z_j; \lambda_j) \nonumber \\
    &\quad \cdot \prod_{t=1}^{d-1} c^t(\mathbf{u}(\mathbf{z}; \bm{\lambda}, \bm{\eta}^1, ..., \bm{\eta}^{t-1}); \bm{\eta}^t) \; ,
\end{align*}

with marginal distributions $q'(\mathbf{z}; \bm{\lambda}) := \prod_{j=1}^d q_j(z_j; \lambda_j)$, i.e. the MF.

The D-vine tree structure and order of the random variables, as well as the pair copula family associated with each edge, are assumed to be given. We propose to optimize the parameters of the variational D-vine model, using the VR-IWAE bound with reparameterization for variance reduction. Here $\alpha$ is set to a low value, for which we show in a simulation study in Appendix \ref{app:simulation_study_alpha} that it overall works well across a range of simulation settings and examples. We perform the optimization of the variational parameters in a stepwise manner, tree by tree, as described in Algorithm \ref{algo:stepwiseVIvinesPseudoCode}.
The truncation level $\tau \leq d-1$ is inferred using a global stopping criterion based on the latent dependence that is present in the data. This yields a variational model with $d(d-1)/2 - (d - \tau)(d - \tau - 1)/2$ parameters contrary to estimating $d(d-1)/2$ parameters by default in a full D-vine. As a local stopping criterion for when the current variational parameter estimates have converged, we follow \citet{dhaka2020robust} and use $\hat{R}$: The trajectory of variational parameter estimates is regarded as a Markov chain (MC) and convergence of the latter is detected with the $\hat{R}$ statistics used in Markov-Chain Monte Carlo (MCMC) methods \cite{gelman2013bayesian}. Our proposed method is presented in Algorithm \ref{algo:stepwiseVIvinesPseudoCode}. Implementation details can be found in Appendix \ref{app:implementation}.

\paragraph{Theoretical justification for low $\alpha$}
An important justification for using the R\'enyi $\alpha$-divergence with a rather low value of $\alpha$, is that when the true posterior is Gaussian, it is an approximation of the forward KL divergence $KL(p||q)$. In the non-Gaussian case the $\alpha$ parameter allows to control the amount of weight put on the true posterior in the R\'enyi $\alpha$-divergence \citep{daudel2023alpha}. 
Using the forward KL divergence $KL(p||q)$ as VI objective, the approximate posterior recovers the correct parameters of the true posterior when the latter is Gaussian, as we show in Theorem \ref{thm:forwardKL}.
However, this is generally infeasible. On the other hand, VI based on the standard backward KL divergence $KL(q||p)$ does not result in variational parameter estimates that match the true posterior parameters unless the true posterior exhibits latent independence. This follows from Theorem \ref{thm:backwardKL} below.

Assume that the true posterior distribution $p(\Bb{z}|\Bb{x})$ is the multivariate Gaussian distribution $\mathcal{N}(\B{\mu},\B{\Sigma})$, which we approximate with $q(\Bb{z}; \bm{\phi})$, the multivariate Gaussian distribution $\mathcal{N}(\B{\nu},\B{\Psi})$ with $\bm{\phi} := (\bm{\nu}, \B{\Psi})$, constructed from Gaussian univariate marginals combined with a Gaussian D-vine, a D-vine with only Gaussian copulas. Further, let $\Bb{R}_{\Sigma}$ and $\Bb{R}_{\Psi}$ be the correlation matrices corresponding to $\B{\Sigma}$ and
$\B{\Psi}$, respectively, so that $\B{\Sigma} = \Bb{D}_{\Sigma}\Bb{R}_{\Sigma}\Bb{D}_{\Sigma}$
and $\B{\Psi} = \Bb{D}_{\Psi}\Bb{R}_{\Psi}\Bb{D}_{\Psi}$, where $\Bb{D}_{\Sigma}$ and 
$\Bb{D}_{\Psi}$ are the $d\times d$ diagonal matrices with diagonal elements $\sqrt{\Sigma_{ii}}$ and
$\sqrt{\Psi_{ii}}$, respectively. The following theorem establishes that by minimizing the forward KL-divergence the approximate posterior recovers true posterior.
 
\begin{theorem}\label{thm:forwardKL}
    The parameters $\bm{\phi} = (\B{\nu},\Bb{D}_{\Psi},\Bb{R}_{\Psi})$ of $q(\Bb{z}; \bm{\phi})$ obtained when minimizing the forward KL in the proposed stepwise manner are the true parameters, i.e $\B{\nu}=\B{\mu}$, $\Bb{D}_{\Psi}=\Bb{D}_{\Sigma}$ and $\Bb{R}_{\Psi}=\Bb{R}_{\Sigma}$.    
\end{theorem}
When instead minimizing the backward KL divergence, the approximate posterior will only recover the mean of the true posterior, and can only recover the correlation matrix in a special case.

\begin{theorem}\label{thm:backwardKL}
    If the parameters $\bm{\phi} = (\B{\nu},\Bb{D}_{\Psi},\Bb{R}_{\Psi})$ of $q(\Bb{z}; \bm{\phi})$ are obtained by minimizing the backward KL in the proposed stepwise manner, the true mean vector is recovered $\B{\nu}=\B{\mu}$, but the standard deviations $\Bb{D}_{\Psi}$ and correlation matrix $\Bb{R}_{\Psi}$ will \textbf{not} be equal to the true values, unless all correlations are $0$.    
\end{theorem}

The proofs of Theorems \ref{thm:forwardKL} and \ref{thm:backwardKL} can be found in Appendix \ref{app:fw_bw_KL_divergence_Gaussian_setting}.
Note that the marginal distributions estimated in tree 0 of Algorithm \ref{algo:stepwiseVIvinesPseudoCode} are used to transform the data to the copula scale, i.e. to obtain copula data for the following steps, but that is not the only reason why the stepwise procedure with the backward KL fails to recover the true correlations. As shown in the Proposition \ref{prop:backward_KL_known_std} in Appendix \ref{app:backwardKL_proof_and_proposition}, the stepwise procedure does not recover the true correlations even when the true standard deviations are known, unless a Gaussian D-vine with only one tree is the true model, so that all the vine-copula parameters are optimized simultaneously in the second step of the procedure.

\paragraph{VI with VR-IWAE bound and reparameterization}
We performed a simulation study to assess the effect of $\alpha \in (0,1)$ in the VR-IWAE on the approximation capacity of the D-vine and the MF as variational distributions. We found that an $\alpha =0.1$ consistently yields good performance in several scenarios and examples, see Appendix \ref{app:simulation_study_alpha}. We apply the reparameterization trick for variance reduction. \citet{tran2015copula} note that the reparameterization trick is always applicable for continuous $\mathbf{z}$ following a vine distribution, as it can be expressed as a deterministic transformation, the inverse marginal cdf, of $\mathbf{u} \sim \text{Unif}([0,1]^d)$.

\paragraph{Global stopping criterion}
The parameter $\eta$ of a Gaussian pair copula in tree $t$ of the D-vine is the (partial) correlation $\text{Cor}(Y_j ,  Y_{j+t}) = \eta \in [-1,1]$ where $Y_{j} := \Phi^{-1}\big( F(Z_j | \mathbf{z}_{j+1 : j+t-1}) \big)$\footnote{N.B.: Here, $\Phi(\cdot)$ is not a Gaussian marginal distribution, but comes from the definition of the Gaussian pair copula, which can be found in Equation \eqref{def:Gaussian_copula}.}
with $j \in [d-t]$, 
and therefore has an interpretable scale. If $| \eta | < 0.1$ for all pair copulas in the current tree $t$ in the variational D-vine model , we can assume that there is no more latent dependence to capture. We stop adding further trees to the variational model, and consider the $t-1$-truncated D-vine as our final variational model. For pair copula families other than the Gaussian, we propose to use the Kendall's $\tau$ rank correlation coefficient, which again has an interpretable scale.

\paragraph{$\hat{R}$ as local stopping criterion}
\citet{dhaka2020robust} points out that the standard stopping criterion $\Delta$ ELBO $< \epsilon$ (or any other lower bound) for optimization is flawed. As the scale of the ELBO changes with the parametrization of the model for the observed data, the choice of $\epsilon$ decides whether the optimization is stopped prematurely or the stopping criterion will ever be invoked. Instead, they propose to view the sequence of variational parameter estimates as a MC and use a MCMC diagnostic tool to assess convergence. They propose to use the rank-normalized $\hat{R}$ \citep{vehtari2021rank} as a stopping criterion for optimization. Due to its missing implementation in \verb|pyro|, which we will leave for future work, we resort to the split-$\hat{R}$ proposed by \citet{gelman2013bayesian}, to which we will simply refer as $\hat{R}$. It is defined as the square root of the ratio of between- ($\mathbb{V}$) and within-chain ($\mathbb{W}$) variance: $\hat{R} := \sqrt{\mathbb{V}/\mathbb{W}}$.  

\begin{algorithm}[tb]
   \caption{Stepwise VI with vine copulas}
   \label{algo:stepwiseVIvinesPseudoCode}
        \begin{algorithmic}
           \STATE {\bfseries Input:} $\gamma$, $R>1$, $STOP=$False, $t=1$, $\tau \leq d-1$ 
           \STATE {\bfseries Output:} $\hat{\bm{\lambda}}, \hat{\bm{\eta}}^1, \dots, \hat{\bm{\eta}}^{\tau}$
           \STATE \textit{Tree 0 (MF):} 
           \WHILE{$\hat{R} > R$}
               \STATE \[ \hat{\bm{\lambda}} \leftarrow \hat{\bm{\lambda}} + \gamma \widehat{\nabla_{\bm{\lambda}} l_N^{(\alpha)}(\hat{\bm{\lambda}}; \mathbf{x})}  \; .\]
               \STATE Compute $\hat{R}$.
           \ENDWHILE
           \WHILE{$STOP$ is False \AND $t \leq \tau$}
                \STATE \textit{Tree $t$:} Current variational model is:
                \[ q(\mathbf{z}; \bm{\lambda}, \bm{\eta}^1, \dots, \bm{\eta}^t) := q'(\mathbf{z}; \bm{\lambda}) \prod_{l=1}^t c^l(\mathbf{u}; \bm{\eta}^l) \; .\]
                with $\hat{\bm{\lambda}}, \hat{\bm{\eta}}^1, \dots, \hat{\bm{\eta}}^{t-1}$ fixed.
                \WHILE{$\hat{R} > R$}
                   \STATE \[ \hat{\bm{\eta}}^t \leftarrow \hat{\bm{\eta}}^t + \gamma \widehat{\nabla_{\bm{\eta}^t} l_N^{(\alpha)}(\hat{\bm{\eta}}^t; \mathbf{x})} \]
                   \STATE Compute $\hat{R}$.
                \ENDWHILE
                \STATE $STOP \leftarrow$ True \textbf{if} for all pair copulas of tree $t$ it is $|\rho| < 0.1$, \textbf{else} $STOP \leftarrow False$.
                \STATE
                $ t \leftarrow t+1$
            \ENDWHILE
            % \RETURN $\hat{\bm{\lambda}}, \hat{\bm{\eta}}^1, \dots, \hat{\bm{\eta}}^{\tau}$
        \end{algorithmic}
\end{algorithm}

\section{Results}

\subsection{Competitor models}\label{sec:competitor_models}
We compare stepwise VI with vine copulas to the following competitors: Gaussian MF (MF) with transformations for constrained latent variables (this corresponds to ADVI \citep{kucukelbir2017automatic} used as the default variational approximation in Stan \citep{carpenter2017stan});  Gaussian copula VI (GC-VI) as proposed by \citet{tran2015copula}; and lastly to masked auto-regressive flows (MAF) \citep{papamakarios2017masked}. We also compare these methods to samples from the true posterior, obtained with the \texttt{pyro} implementation of the No-U-Turn Sampler (NUTS). See Appendix \ref{app:competitor_models} for details on and a discussion of the competitor models.

\subsection{Simulated Examples}\label{sec:simulated_examples_results}

\paragraph{Stepwise D-vine recovers MF as correct posterior}
We start with an example where the true posterior is a MF and evaluate the degree to which stepwise VI with vines can correctly infer the complexity of the true posterior. That is, recover the latent independence without implicitly specifying independence as a hyperparameter in the variational model. 

We set this up with a regression example following \citet{shen2025wild}. First we sample $n=50$ i.i.d. draws from $(X_1, ..., X_4)^T \sim \mathcal{N}(\bm{0}, C)$ where $C := I_4$ and set:
\begin{align*}
    Y := \beta_1 X_1 + \beta_2 X_2 + \beta_3 X_3 + \beta_4 X_4 \;
\end{align*}
where the true values of the coefficients are set to $(\beta_1, ..., \beta_4)^T := (10, -10, 5, 3)$. We take $p(\bm{\beta})=\mathcal{N}(\bm{0},I_4)$ as the prior distribution, and $\prod_{i=1}^np(y_i | \mathbf{x}_i, \bm{\beta})=\prod_{i=1}^n\mathcal{N}(\mathbf{x}_i^T \bm{\beta}, 1)$ as the likelihood.

Our stepwise VI procedure correctly invoked its stopping criteria at tree 1, and recovered the contour plots obtained with NUTS, see Figure \ref{fig:first_page} and Figure \ref{fig:example_independence_pairsplots} in Appendix \ref{app:simulated_examples_results}. While the GC-VI and MAF also recover the contour plots, both methods incorrectly estimate a slight posterior correlation greater than 0 in absolute value. Only stepwise VI with vines with the global stopping criterion correctly invoking at tree 1, and MF recover \textit{exact} posterior independence.

% Even if we do not run GC-VI until convergence and decrease the number of MC samples by a factor of $100$ compared to what is proposed by \citet{tran2015copula}, GC-VI runs 8 times longer than stepwise VI with vines. 

% \begin{figure}[!t]
%     \centering
%     \includegraphics[width=\linewidth]{images/Dvine_recovers_MF/pairsplots.pdf}
%     \includegraphics[width=\linewidth]{images/Neelde/pairsplots.pdf}
%     \caption{Contour plots of latent vector comparing NUTS (black) as ground truth to variational approximations obtained with stepwise VI with vines (blue), MFVI (orange), GC-VI (green) and MAF (red) and an independence example (top) and needle example (bottom).}
%     \label{fig:example_independence_pairsplots}
% \end{figure}

\paragraph{Needle Example}
When the data only provides information about the sum or difference of two latent variables, the latent variables are only weakly identifiable and show a needle-shaped dependence. This is the case in linear regression with collinearity as proposed by \citet{shen2025wild} in their needle example.
We simulate a data set in the same way as for the previous example, except for setting:
\begin{align*}
    C &:= \begin{pmatrix}
        1&  0.9&  0.14& -0.85 \\
        0.9&  1 & -0.2& -0.9 \\
        0.14& -0.2&  1 & -0.1 \\
        -0.85& -0.9& -0.1 &  1
    \end{pmatrix}, \; 
\end{align*}
when simulating $X_1,\ldots,X_4$. We use the same likelihood and prior, and the same number of observations.

The results are displayed on the bottom row of Figure \ref{fig:first_page}. We see that regular MF and GC-VI fail to capture the dependency structure in the posterior, underestimating both variance and covariance, while our model can correctly identify both. Similarly MAFs did well in this setting, as can be seen in Appendix \ref{app:simulated_examples_results}.

% GC-VI runs 7 times slower than our proposed approach. Our approach and MAF recover the correct posterior dependence, while MFVI and GC-VI fail to do so. MAF was taking the least time to achieve this. \eg{Should we report time, if we do not report if of MF and MAF?} \lr{It is nice to note that we do better than the main competitor though, even if we have to make a note that MAFs are very fast.}

% \paragraph{Banana}
% When the data only provides information about the product of two latent variables, the latent variables are only weakly identifiable and show a banana-shaped dependence. \citet{shen2025wild} note that this is the case in an N-mixture model used by ecologists to count unmarked animals. Data $y_{i, t}$ is simulated as follows: set the latent variables $\lambda=30$ and $p=0.1$, sample $N_i \sim \mathrm{Poisson} (\lambda)$ for $i \in [20]$ and draw $y_{i, t} \sim \mathrm{Binomial}(N_i,p)$ for $t \in [5]$. 

\subsection{Inducing Points of Gaussian Processes}
We also apply our method to the setting of learning \textit{inducing points} in sparse Gaussian Process regression (SGPR) \cite{Titsias2009, hensman2013gaussian}. 

We consider a Gaussian process (GP) regression model with inputs $\mathbf{x}_i\in\mathcal{X}\subset\mathbb{R}^d$, and noisy observations $y_i=f(\mathbf{x}_i)+\epsilon_i$, with $f\sim\mathcal{GP}(0,\kappa_\theta)$, $\epsilon_i \stackrel{iid}{\sim}\mathcal{N}(0,\sigma^2)$ and $\kappa_\theta:\mathcal{X}\times\mathcal{X}\to\mathbb{R}$ a positive-definite covariance function, or \textit{kernel}, parametrized by some hyperparameters $\boldsymbol{\theta}$. For a finite set of input locations, the latent vector of function evaluations $\mathbf{f}=\{f(\mathbf{x}_i)\}_{i=1}^n$ takes on a multivariate normal distribution $\mathbf{f}\sim\mathcal{N}(\mathbf{0},K)$, where $K$ is the $n\times n$ matrix with entries $K_{ij}=\kappa(\mathbf{x}_i,\mathbf{x}_j)$. Conditioning on data yields Gaussian posterior and predictive distributions, but inference and tuning of the hyperparameters $\{\boldsymbol{\theta},\sigma^2\}$ scale cubically in the number of datapoints.

In SGPR, a set of \textit{inducing variables} $\mathbf{v}=\{f(\mathbf{z}_j)\}_{j=1}^m$ are introduced, which correspond to function evaluations at a new set of input locations $\{\mathbf{z}_j\}_{j=1}^m$, where $\mathbf{z}_j\in\mathcal{X}$ and crucially $m \ll n$. The joint probability model in the SGPR framework takes the form $p(\mathbf{y},\mathbf{f},\mathbf{v}\vert \boldsymbol{\theta})=p(\mathbf{y}\vert\mathbf{f},\boldsymbol{\theta})p(\mathbf{f}\vert\mathbf{v},\boldsymbol{\theta})p(\mathbf{v}\vert\boldsymbol{\theta})$, and variational SGPR approximates the posterior over the unknown latent function evaluations $p(\mathbf{f},\mathbf{v}\vert\mathbf{y})\approx q(\mathbf{f},\mathbf{v})=p(\mathbf{f}\vert\mathbf{v})q(\mathbf{v})$, with $q(\mathbf{v})=\mathcal{N}(\mathbf{m},S)$, with mean and covariance matrix $\{\mathbf{m},S\}$ learned from data as variational parameters alongside the inducing point locations $\{\mathbf{z}_j\}_{j=1}^m$ and hyperparameters $\{\boldsymbol{\theta},\sigma^2\}$.

While other distributions can be used for $q$ rather than the Gaussian, this choice is optimal for the ELBO defined using the KL-divergence, and one can derive closed form expressions for the ELBO, as well as the optimal values of $\{\mathbf{m},S\}$ \cite{Titsias2009}. Because of this optimality, most of the literature has focused on the structure of $S$, the two most common parameterizations being a full-rank Cholesky $S=LL^T$, where $L$ is an $m \times m$ lower-triangular matrix, or the mean-field $S=\text{diag}\{s_1^2,\ldots,s_m^2\}$. We employ our sequential VI framework to this problem, parameterizing $q$ as a Gaussian vine copula, starting from a MF and sequentially adding more dependencies as more trees are added.

We evaluate our model on a real-world benchmark dataset; the \texttt{pumadyn32nm} dataset consisting of 7168 training samples and 1024 test samples, with 32 features. We use the RBF kernel, with automatic relevance determination, i.e. $\kappa(\mathbf{x},\mathbf{x}')=\sigma_0^2\exp(-\sum_{k=1}^d \vert x_k-x_k'\vert^2 /2\ell_k^2)$. The hyperparameters $\{\sigma_0^2,\boldsymbol{\ell}\}$ and the noise variance $\sigma^2$ are fixed at values obtained from an initial (non-sparse) GP fit. This follows the setup in \cite{lazaro2009inter, snelson2005sparse} and is done to better showcase the properties of the variational posterior rather than hyperparameter tuning.

For the vine copula, we first fit the MF allowing all variational parameters to move freely. Then, moving on to the first tree level, in addition to fixing the MF parameters $\{s_1^2,\ldots,s_m^2\}$ and $\mathbf{m}$, we also fix the inducing point locations $\{\mathbf{z}_j\}_{j=1}^m$, as once $\mathbf{m}$ is fixed, it makes little sense to move them about. Because the inducing point inputs are allowed to move freely in the MF, computation of $\hat{R}$ in this first step can suffer due to "label-switching". We therefore opt to compute our convergence criterion $\hat{R}$ in this step based on $\Vert\mathbf{m}\Vert_2$ and $\Vert S \Vert_F$ instead of $\mathbf{m}$ and $\{s_1^2,\ldots,s_m^2\}$ directly, where $\Vert \cdot\Vert_2$ denotes the $L_2$ norm and $\Vert\cdot\Vert_F$ the Frobenius norm. This allows us to sidestep the "label-switching" issue as the $L_2$ norm is invariant to the ordering of the elements in $\mathbf{m}$, and the Frobenius norm invariant to a reordering of rows and columns of the covariance matrix. Once inducing point locations are fixed we run a greedy nearest-neighbor algorithm on the inducing point locations to set the tree-structure, ensuring that the covariance structure is built iteratively from nearest neighbors in the inducing point space. And we compute $\hat{R}$ again directly using copula parameters, $\boldsymbol{\eta}$.

Denoting the test dataset by $\mathbf{y}_*$ and the predictive mean and covariance by $\hat{\mathbf{y}}_*$ and $K_{**}$ respectively, we compute the root mean squared error (RMSE) $\sum_{i=1}^{n_{test}}(\mathbf{y}_{*,i}-\hat{\mathbf{y}}_{*,i})^2 / n_{test}$ and the negative log-predictive density (NLPD) $- \log \mathcal{N}(\mathbf{y}_*| \hat{\mathbf{y}},K_{**}+\sigma^2I)/n_{test}$.

We compare our stepwise procedure against full-rank (SGPR) and mean-field SGPR (MF-SGPR) methods, and display our results in Figure \ref{fig:GP_results_nlpd} and Figure \ref{fig:GP_results_rmse} in Appendix \ref{app:GP_results}. We note that only small improvements were seen in our experiments past tree one, and thus for visibility we limit our figures to this case. Compared to MF-SGPR and SGPR our method is equivalent in terms of RMSE, but the NLPD shows that our method interpolates between these two extremes. We further display in Figure \ref{fig:GP_cov} the evolution of the correlation matrix associated with the covariance matrix $S$ for our stepwise vine at different tree levels for the setting of $50$ inducing points. Our global stopping criterion did not trigger until $t=46$, indicating perhaps that the greedy procedure we used to set the tree structure was sub-optimal.

\begin{figure}[h]
\centering\includegraphics[width=\linewidth]{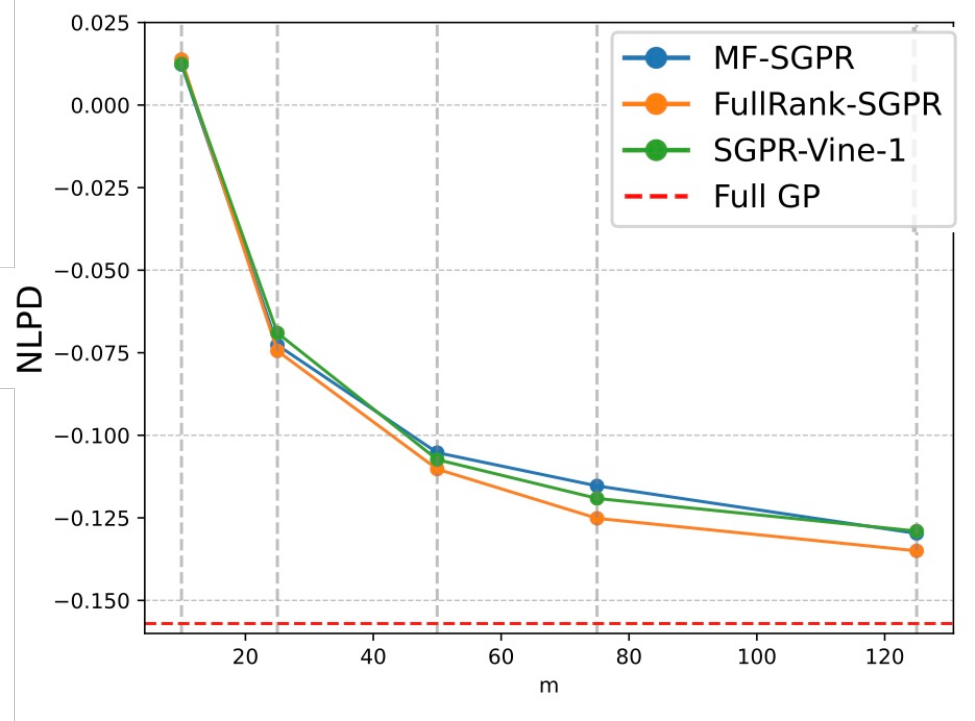}
    \caption{NLPD across a range of inducing point values for the \texttt{pumadyn32nm} dataset. We compare the mean-field SGPR (MF-SGPR), a full rank SGPR (FullRank-SGPR), and the vine at tree level $\tau=1$ to a full (non-sparse) GP fit (red dashed line).}
    \label{fig:GP_results_nlpd}
\end{figure}
\begin{figure}[h]
    \centering\includegraphics[width=\linewidth]{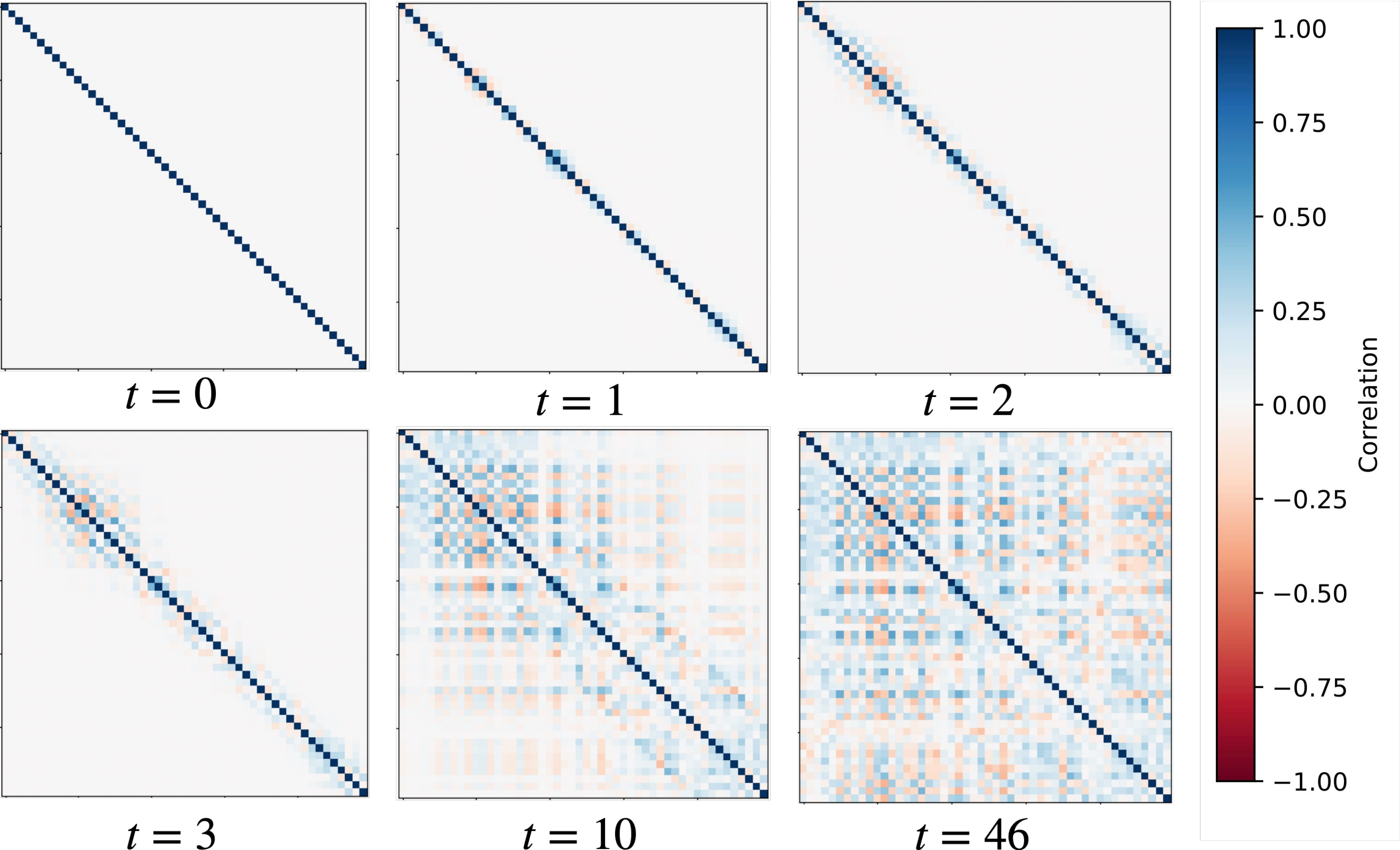}
    \caption{Evolution of the corresponding correlation matrix associated with the variational matrix $S$ at different tree levels $t$ of vine copula approximate posterior, for the \texttt{pumadyn32nm} example using 50 inducing points.}
    \label{fig:GP_cov}
\end{figure}

\section{Conclusion}
In this paper, we have proposed a novel general-purpose framework for variational inference based on vine copulas. Paired with a stepwise procedure for estimating variational parameters, our model captures MFVI as a special case, but is able to adapt, and automatically select the complexity of the approximating distribution. We illustrate our approach on real and simulated examples, and show that our model provides a middle ground between full-rank variational methods and MFVI.

\paragraph{Future Work}
Our work motivates several directions for future research. Learning the vine tree structure, which \citet{tran2015copula} unsuccessfully attempted, and selecting pair copula types would make stepwise VI with vines even more flexible. Our approach can be extended to jointly model continuous and discrete latent variables with the use of discrete copulas \citep{panagiotelis2012pair}, which is difficult within other frameworks such as ADVI \citep{kucukelbir2015automatic} or normalizing flows. In the GP application, analytic forms of the R\'enyi divergence based lower bound could be exploited as done by \citet{yue2019renyi}, which would give tighter bounds and potentially faster convergence. One could also explore other procedures for setting the tree-structure based on inducing-point location.

\FloatBarrier

\section*{Impact statement} 
This paper is a contribution to the field of variational inference and proposes a flexible and parsimonious variational family and estimation procedure. More accurate approximate posterior inference that is parsimonious in parameters, allows more accurate and scalable uncertainty quantification, which positively impacts decision making under uncertainty in fields like healthcare or climate modeling.

% This paper presents work whose goal is to advance the field of machine learning. There are many potential societal consequences of our work, none of which we feel must be specifically highlighted here.

% for camera-ready
\section*{Acknowledgment}
This work was supported by the Research Council of Norway, Integreat - Norwegian centre for knowledge-driven machine learning, project number 332645. Leiv Rønneberg was supported by the European Union’s Horizon Europe research and innovation programme under the Marie Sklodowska-Curie grant agreement No. 101126636. The work of Claudia Czado is supported in part by the Deutsche Forschungsgemeinschaft (DFG CZ-86/6-1 CZ-86/10-1).

\bibliography{references}
\bibliographystyle{icml2026}

\newpage
\appendix
\onecolumn
%%%%%%%%

\section{Notation}\label{app:notation}

Please find commonly used notation in Table \ref{table:notation}. In general, bold letters $\bm{v}$ denote vectors; capital letters such as $U$, $\mathbf{X}$ or $M$ denote random variables ($U$), random vectors ($\mathbf{X}$) or matrices ($M$); small letters denote (if not specified otherwise) real scalars ($x$) or vectors ($\mathbf{x} \in \mathbb{R}^d$).

\begin{table}[H]
  \caption{Notation used throughout the paper.}
  \label{table:notation}
  \begin{center}
    \begin{small}
      \begin{sc}
        \begin{tabular}{lll}
          \toprule
          Notation & Definition & Description  \\
          \midrule
          $[d], \; d \in \mathbb{N}$ & $[d] := \{1, ...., d \}, \; d \in \mathbb{N}$ & natural numbers from 1 to $d$ \\
          $j:k, \; j,k \in \mathbb{N}, \; j<k$ & $j:k := \{j, j+1, ..., k \}$ & natural numbers from $j$ to $k$ \\
          $X$ & -- & 1-dimensional random variable $X$ \\
          $\mathbf{X} \in \mathbb{R}^d, \; d \in \mathbb{N}$ & $\mathbf{X} := (X_1, ..., X_d)^T$ & $d$-dimensional random vector \\
          $\mathbf{X} =\mathbf{x}$ & -- & $d$-dimensional random vector $\mathbf{X}$ takes on value $\mathbf{x} \in \mathbb{R}^d$ \\
          $\mathbf{x}_{<j}, \; j \leq d$ & $\mathbf{x}_{<j} := (x_1, ..., x_{j-1})^T$ & $j-1$-dimensional sub-vector of $\mathbf{x} \in \mathbb{R}^d$ \\
          $\mathbf{x}_{\geq j}, \; j \leq d$ & $\mathbf{x}_{\geq j} := (x_j, ..., x_d)^T$ & $(d-j+1)$-dimensional sub-vector of $\mathbf{x} \in \mathbb{R}^d$ \\
          $\mathbf{x}_{j:k}, \; 1 \leq j < k \leq d$ & $\mathbf{x}_{j:k} := (x_j, x_{j+1}, ...,  x_k)^T$ & $(k-j+1)$-dimensional sub-vector of $\mathbf{x} \in \mathbb{R}^d$ \\
          $I_d$ & $I_d := diag(1, ..., 1) \in \mathbb{R}^d$ & $d$-dimensional identity matrix \\
          \bottomrule
        \end{tabular}
      \end{sc}
    \end{small}
  \end{center}
  \vskip -0.1in
\end{table}

\section{Extended Literature Review}\label{app:extended_literature_review}

Stepwise VI with vines is a copula based approach to VI, that refines the variational approximation in each step.

\paragraph{Copula based VI methods} Copula based VI methods have been discussed in the main part of this paper.

\paragraph{Boosting VI}

Boosting VI (BVI) \citep{guo2016boosting} approximates the posterior with a finite mixture of parametric base distributions in a stepwise manner: Borrowing ideas from gradient boosting, \citet{guo2016boosting} add a new component to the mixture in each boosting step. The result is a flexible variational distribution that can approximate various shapes of posteriors. In detail, in each step $t$ \citet{guo2016boosting} propose to mix the current variational distribution $q_{t-1}$ with $h_t$, a member from the chosen base distribution class with some mixing weight $\alpha_t$, i.e. $q_t := (1 - \alpha_t) q_{t-1} + \alpha_t h_t$. Here, $h_t$ is chosen with (Laplacian) gradient boosting and, keeping $h_t$ fixed, $\alpha_t$ is chosen with SGD. This is repeated  for a fixed number of steps. The authors focus on the Gaussian base distribution class and provide a closed form solution for the parameters of $h_t$ based on a heuristic. If the dimension $d$ of the latent space is high, \citet{guo2016boosting} are restricted to a diagonal $\Sigma$ due to computational burden.
As \citet{guo2016boosting}, \citet{miller2017variational} propose a finite Gaussian mixture as the variational model and formulate the reparametrization trick for lower variance ELBO gradient estimation for mixtures. This allows them to treat the mixing weight as well as the component distribution parameters as variational parameters and optimize for them with SGD. Despite the reparametrization trick, \citet{miller2017variational} need many samples - 400 in their experiments - to reliably estimate the ELBO gradients. \citet{miller2017variational} allow to model posterior correlation to varying degree by utilizing a low-rank plus diagonal covariance in the mixture components. As a stopping criterion for when to include higher-rank components into the mixture they propose to monitor the average absolute change in marginal variance.
By showing that the KL-divergence has bounded curvature on the set of mixtures of truncated, non-degenerate distributions, \citet{locatello2018aboostingopt} draw a connection between boosting VI and the functional Frank-Wolfe algorithm. Through this they provide a proof of why approximating the posterior with VI in a boosting style works for mixtures of truncated, non-degenerate distributions and give convergence rates. Additionally, they propose a variant of the Frank-Wolfe algorithm for boosting VI that updates all mixture weights in each iteration, leading to a higher computational load but fast convergence. 
Assuming a truncated support of the base distribution class in \citet{locatello2018aboostingopt} leads to in irreducible error term for the variational approximation to the true posterior. Additionally, special care needs to be taken during the optimization to avoid degenerate component distributions prohibiting the use of black-box VI methods. \citet{locatello2018boosting} mitigate these limitations. They prove that it suffices to have a bounded parameter space (instead of truncated support of the base distribution class) for convergence of the Frank-Wolfe algorithm and propose a Residual ELBO with regularization to guarantee non-degenerate component densities. Finally, they propose a stopping criterion based on the duality from the Frank-Wolfe algorithm.
\citet{campbell2019universal} show that regularization in boosting VI to avoid degeneracy of the approximating posterior \citep{guo2016boosting, locatello2018boosting} can lead to un-intuitive behavior of the approximation or loss of convergence guarantees. Instead, \citet{campbell2019universal} propose VI method using a Hellinger distance based  objective termed universal boosting VI. Through this they avoid degeneracy and the difficult joint optimization of mixture component and weight of \citet{miller2017variational} without any need for hyperparameter tuning while providing theoretical convergence guarantees. The authors refine their variational distribution for a fixed number of steps.

% \paragraph{Structured MF}
% \begin{itemize}
%     \item \citet{saul1995exploiting}
%     \item \citet{ghahramani1997structured}
%     \item \cite{ranganath2016hierarchical}: not really boosting, not really structured VI, something in between
%     \item 
% \end{itemize}

% \paragraph{Automatic VI methods}
% \citet{kucukelbir2017automatic} propose Automatic Differentiation VI (ADVI), a method for automatic VI on continuous latent variables implemented in Stan \citep{carpenter2017stan}, where only the joint model, but no variational model needs to be provided as input. ADVI transforms any constrained latent variables to an unconstrained space and performs VI with a Gaussian full-rank distribution or Gaussian MF as variational distribution there. Due to the Jacobian adjustment, the variational model in the constrained space is non-Gaussian. With the automatic differentiation and list of transformations and corresponding Jacobians in Stan, mini-batching, adaptive step-size as well as the reparametrization trick to reduce ELBO gradient variance, full-rank Gaussian ADVI performs on par with NUTS \citep{hoffman2014no} in terms of predictive accuracy and MF ADVI outperforms sampling based approaches in terms of speed. Stepwise VI with a D-vine and Gaussian pair copulas as the variational model can be understood as a more flexible version of ADVI.

% \citet{ambrogioni2021automatic}

\section{Introduction to Vine Copulas}\label{app:intro_to_vines}
 This introduction to vine copulas is based on \citet{czado2019analyzing}. More details can be found there or for example in \cite{nelsen2006introduction} and \citet{joe2014dependence}.

\subsection{Copulas}

Vine copulas build on the concept of copulas, which represent a distribution class with a specific support and specific marginals.

\begin{definition}
Let $d \in \mathbb{N}$. The function $C: [0,1]^d \rightarrow [0,1]^d$ is a \textit{d-dimensional copula} if it is a $d$-dimensional cumulative distribution function with uniform marginal distributions $U[0,1]$.
\end{definition}

Sklar's Theorem, \cite{sklar1959fonctions} provides the link between copulas and distributions: Any $d$-dimensional probability distribution of a random vector $(X_1, \dots, X_d)$ can be expressed as its corresponding $d$-dimensional copula. 

\begin{theorem}[Sklar's Theorem] \label{thm:Sklar}
    Let $\mathbf{X}$ be a $d$-dimensional random vector with distribution function $F$ and marginal distributions $F_1, \dots F_d$. Then $F$ can be expressed as:
    \begin{align} \label{thm:Sklar1}
        F(x_1, \dots, x_d) = C(F_1(x_1), \dots, F_d(x_d)) \; , \quad  (x_1, \dots, x_d) \in \mathbb{R}^d \; . 
    \end{align} 
    where $C$ is a copula. If $F$ is absolutely continuous, the copula $C$ is unique. We then say that the copula $C$ is corresponding to the distribution $F$. In the case of absolute continuity all densities exist and we can express the joint density $f$ of $\mathbf{X}$ as:
    \begin{align} \label{thm:Sklar1_density}
        f(x_1, \dots, x_d) = c(F_1(x_1), \dots, F_d(x_d)) \cdot f_1(x_1) \cdot ... \cdot f_d(x_d) \; . 
    \end{align} 
    Conversely, let $C$ be the $d$-dimensional copula corresponding to the joint distribution function $F$ of $\mathbf{X}$ with marginal distributions $F_1, \dots F_d$. Then we can express $C$ as:
    \begin{align}\label{thm:Sklar2}
        C(u_1, \dots, u_d) = F(F_1^{-1}(u_1), \dots, F_d^{-1}(u_d)) 
    \end{align}
    with copula density:
    \begin{align}\label{thm:Sklar2_density}
        c(u_1, \dots, u_d) = \frac{f(F_1^{-1}(u_1), \dots, F_d^{-1}(u_d))}{f_1(F_1^{-1}(u_1)) \cdot ... \cdot f_d(F_d^{-1}(u_d))} \; .
    \end{align}
\end{theorem}

Equation \eqref{thm:Sklar1_density} illustrates how the joint density $f$ of a random vector $(X_1, \dots, X_d)$ can be split into the joint copula density, which captures the dependence structure of $X_1, \dots X_d$, and the marginal densities $f_1, \dots f_d$.

The inverse Sklar's Theorem \ref{thm:Sklar} gives the construction of the \textit{elliptical copulas}, to which the Gaussian copula belongs.

\begin{definition}[bivariate Gaussian copula]
    Let $\Phi_2 ( \cdot, \cdot; \rho)$ be the $2$-dimensional standard normal distribution with mean vector $\bm{\mu} = 0$ and correlation parameter $\rho \in (0, 1)$, and let $\Phi^{-1}(\cdot)$ be the inverse of the univariate standard normal distribution. Then by Sklar's Theorem  \ref{thm:Sklar} we obtain the \textit{bivariate Gaussian copula} by:
    \begin{align}\label{def:Gaussian_copula}
        C(u_1, u_2; \rho) = \Phi_2(\Phi^{-1}(u_1), \Phi^{-1}(u_2); \rho) \; .
    \end{align}
\end{definition}

Another class of copulas, the Archimedean copulas, is defined through generator functions and has members such as the Clayton, Gumbel, Frank or Joe copulas. Please find more details in \cite{nelsen2006introduction}.

\subsection{From Copulas to Vines: Pair Copula Decomposition and Construction}
Equation \eqref{thm:Sklar1_density} of Sklar's Theorem \ref{thm:Sklar} provides a recipe to estimate flexible multivariate densities by modeling $d$-dimensional dependence and marginals separately. However, estimating a $d$-dimensional copula is challenging. Additionally, (parametric) copula families, such as elliptical or Archimedean, do not allow combine different types of dependence, e.g. upper, lower tail dependence or both. They are thus limited in their modeling capacity.

\citet{aas2009pair}, which the rest of this section is based on, decompose a multivariate density by using a cascade of pair copulas as bivariate building blocks. This decomposition can then be reversed in order to construct multivariate copulas and distribution functions respectively. These are flexible and their construction is simple. This is the idea of \textit{pair copula construction}.

We define the following notation:

\begin{definition}\label{def:notation_copula_cond_dist}
    Let $\mathbf{X}_D \in \mathbb{R}^d$ be a random vector and $\mathbf{x}_D \in \mathbb{R}^d$, let $i, j, d \in \mathbb{N}$ and $D \subset \mathbb{N}$ with $i, j \notin D$ and $| D | = d$. Let $F_{i j \, | \, D}(\cdot, \cdot \, | \, \mathbf{X}_D = \mathbf{x}_D)$ be the conditional distribution of $(X_i, X_j)$ given that $\mathbf{X}_D = \mathbf{x}_D $. The copula distribution associated with $F_{i j \, | \, D}(\cdot, \cdot \, | \, \mathbf{X}_D = \mathbf{x}_D)$ is denoted by:
    \begin{align*}
        C_{i j ; D}( \cdot, \cdot ; \mathbf{x}_D) \; .
    \end{align*}
    If existing, its corresponding density is denoted by:
    \begin{align*}
        c_{i j ; D}( \cdot, \cdot ; \mathbf{x}_D) \; .
    \end{align*}
\end{definition}

We make a 3-dimensional example to illustrate a pair copula decomposition.

\begin{example}[Pair copula decomposition]
    Let $\mathbf{X} = (X_1, X_2, X_3)$ be a random vector with joint density function $f_{1 2 3}$ and marginal density functions $f_1, f_2$ and $f_3$. Using conditioning we can rewrite the joint density function:
    
    \begin{align} \label{pcc:start_term}
        f_{1 2 3}(x_1, x_2, x_3) = f_{1 | 2 3}(x_1 \, | \, x_2, x_3) f_{2 | 3}(x_2 \, | \, x_3) f_3(x_3) \; ,
    \end{align}
    
    with:
    
    \begin{align}
        f_{2 | 3}(x_2 \, | \, x_3) &= \frac{f_{2 3}(x_2, x_3)}{f_3(x_3)} \; , \label{pcc:umformung1}\\
        f_{1 | 2 3}(x_1 \, | \, x_2, x_3) &= \frac{f_{1 2 3}(x_1, x_2, x_3)}{f_{2 3}(x_2, x_3)} = \frac{f_{1 3 | 2}(x_1, x_3 \, | \, x_2)}{f_{3 | 2}(x_3 \, | \, x_2)}\; . \label{pcc:umformung2}
    \end{align}
    
    By Sklar's Theorem \ref{thm:Sklar} we know, that:
    
    \begin{align*}
       f_{2 3}(x_2, x_3) = c_{2 3} (F_2(x_2), F_3(x_3)) f_2(x_2) f_3(x_3) \; , 
    \end{align*}
    
    and thus \eqref{pcc:umformung1} becomes:
    
    \begin{equation}\label{pcc:zwischen_ergebnis1}
        \begin{aligned}[b]
            f_{2 | 3}(x_2 \, | \, x_3) &:= \frac{f_{2 3}(x_2, x_3)}{f_3(x_3)} = c_{2 3} (F_2(x_2), F_3(x_3)) f_2(x_2) \; . 
        \end{aligned}
    \end{equation}
    
    In the same manner we obtain (\ref{pcc:umformung2}):
    
    \begin{equation}\label{pcc:zwischen_ergebnis2}
        \begin{aligned}[b]
            f_{1 | 2 3}(x_1 \, | \, x_2, x_3) &= \frac{f_{1 3 | 2}(x_1, x_3 \, | \, x_2)}{f_{3 | 2}(x_3 \, | \, x_2)} \\
            &= \frac{c_{1 3; 2}(F(x_1 \, | \, x_2), F(x_3 \, | \, x_2); x_2) f_{1 | 2}(x_1 \, | \, x_2) f_{3 | 2}(x_3 \, | \, x_2)}{f_{3 | 2}(x_3 \, | \, x_2)} \\
            &= c_{1 3; 2}(F(x_1 \, | \, x_2), F(x_3 \, | \, x_2); x_2) f_{1 | 2}(x_1 \, | \, x_2) \\
            &= c_{1 3; 2}(F(x_1 \, | \, x_2), F(x_3 \, | \, x_2); x_2) c_{1 2} (F_1(x_1), F_2(x_2)) f_1(x_1) \; . 
        \end{aligned}
    \end{equation}
    
    Combining (\ref{pcc:zwischen_ergebnis1}) and (\ref{pcc:zwischen_ergebnis2}) we can decompose \eqref{pcc:start_term} into a product of pair copulas and marginal distributions:
    
    \begin{equation}
        \begin{aligned}[b]
            f_{1 2 3}(x_1, x_2, x_3) = \, & c_{1 3; 2}(F(x_1 \, | \, x_2), F(x_3 \, | \, x_2); x_2) \\
            & c_{1 2} (F_1(x_1), F_2(x_2)) \, c_{2 3} (F_2(x_2), F_3(x_3)) \\
            & f_1(x_1) \, f_2(x_2) \, f_3(x_3) \; . \label{pcc:decomposition}
        \end{aligned}
    \end{equation}
\end{example}

The decomposition with conditioning in \eqref{pcc:start_term} is not unique. Neither is therefore \eqref{pcc:decomposition}. As a second remark, we note, that $c_{1 3; 2}(\cdot, \cdot; x_2)$, the pair copula associated with the conditional distribution of $(X_1, X_3)$ given $X_2 = x_2$ depends on the value $x_2$ of $X_2$. If we ignore this dependence, i.e.:
\begin{align*}
    \forall x_2 \in \mathbb{R}:& \quad c_{1 3; 2}(u_1, u_3; x_2) = c_{1 3; 2}(u_1, u_3), \quad u_1 \in [0,1], \; u_3 \in [0,1] \ ,
\end{align*}
we make the \textit{simplifying assumption}: Copulas associated with conditional distributions do not depend on the value(s) of the conditioning variable(s).
With the simplifying assumption, the decomposition of \eqref{pcc:decomposition} can be used as a construction of the three dimensional density $f_{123}$ from pair copula densities, conditional distributions and marginal densities. In this case we speak of \textit{pair copula construction}.
The construction of the 3-dimensional example above can be generalized to $d$ dimensions.

\subsection{Regular Vines}
For a $d$-dimensional probability distribution there exist several pair copula constructions. \citet{bedford2001probability} and \citet{bedford2002vines} introduced \textit{regular vines (R-vines)} and the \textit{R-vine specification} to efficiently represent the pair copula constructions. The \textit{R-vine specification} captures the structure of the pair copula construction: Each bivariate copula is associated with an edge in a tree in a sequence of nested trees, the \textit{R-vine tree sequence}. This compact notation facilitates the estimation and sampling procedures on R-vines. \citet{bedford2001probability} and \citet{bedford2002vines} also show, that each R-vine specification represents a unique $d$-dimensional distribution $F$.

\begin{definition}[(Regular) Vine tree sequence] \label{def:rvine}
    A set of trees $\mathcal{V} = (T_1, ..., T_{d-1})$ is a (regular) vine tree sequence (R-vine tree sequence) on $d$ elements if:
    \begin{enumerate}
        \item[(i)] $T_1$ is a tree with edge set $E_1$ and node set $V_1 = \{1, ..., d\}$.
        \item[(ii)] For $t \in \{2, ..., (d-1)\}$ it holds that $T_t$ is a tree with edge set $E_t$ and node set $V_t = E_{t-1}$.
        \item[(iii)] For $t \in \{2, ..., (d-1)\}$ and $\{a, b\} \in E_t$ with $a = \{a_1, a_2\}$ and $b = \{b_1, b_2\}$ we have that $|a \cap b| = 1$ (proximity condition).
    \end{enumerate}
\end{definition}

The proximity condition ensures that nodes $a$ and $b$ are only then joined by an edge in tree $T_t$ if they share a common node in tree $T_{t-1}$, where $a, b \in E_{t-1}$. 

R-vines can be divided into sub-classes depending on the shape of each tree in the vine tree sequence. One of these sub-classes is the class of D-vines.

\begin{definition}[D-vine]
    An R-vine tree sequence $\mathcal{V}$ on $d$ elements is called a \textit{D-vine}, if for each node $v$ of each tree $T_t \in \mathcal{V} , \; t \in [d-1]$ it holds that $deg(v) \leq 2$, i.e. each tree $T_t \in \mathcal{V} , \; t \in [d-1]$ is a path.
\end{definition}

Depicted in Figure \ref{picture:dvine4_notation2} is an example of a 4-dimensional D-vine using a notation consistent with Definition \ref{def:rvine}.

\vspace{0.2cm}
\begin{figure}[H]
    % \centering
    \begin{tikzpicture}[font = \footnotesize, node distance=5mm, main/.style = {draw, shape = circle}, baseline=+6mm]
        
        \node[draw=none,fill=none] at (0,0) {(T1)};
        \node[main] (1) at (4.5, 0) {1};
        \node[main] (2) at (6.5, 0) {2};
        \node[main] (3) at (8.5, 0) {3};
        \node[main] (4) at (10.5, 0) {4};
        \draw (1) -- (2) node[midway, above] {$1, 2$};
        \draw (2) -- (3) node[midway, above] {$2, 3$};
        \draw (3) -- (4) node[midway, above] {$3, 4$};
    \end{tikzpicture}

\vspace{0.3cm}

    \begin{tikzpicture}[font = \footnotesize, node distance=5mm, main/.style = {draw}, state/.style ={ellipse, draw, minimum width = 0.8 cm}, baseline=+6mm]
    
        \node[draw=none,fill=none] at (0,0) {(T2)};
        \node[state] (1) at (5, 0) {$1, 2$};
        \node[state] (2) at (7.5, 0) {$2, 3$};
        \node[state] (3) at (10, 0) {$3, 4$};
        \draw (1) -- (2) node[midway, above] {$1, 3; 2$};
        \draw (2) -- (3) node[midway, above] {$2, 4; 3$};
    \end{tikzpicture}
 
 \vspace{0.3cm}
 
    \begin{tikzpicture}[font = \footnotesize, node distance=5mm, main/.style = {draw}, state/.style ={ellipse, draw, minimum width = 0.8 cm}, baseline=+6mm]
    
        \node[draw=none,fill=none] at (0,0) {(T3)};
        \node[state] (1) at (5.8, 0) {$1, 3; 2$};
        \node[state] (2) at (9.2, 0) {$2, 4; 3$};
        \draw (1) -- (2) node[midway, above] { $1, 4; 2, 3$ };
    \end{tikzpicture}
    \caption{A D-vine tree sequence on 4 elements.}
    \label{picture:dvine4_notation2}
\end{figure}
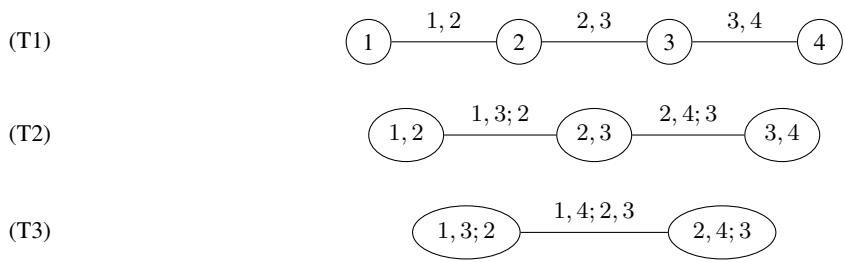

To draw the final link between the vine tree sequence and regular vine construction to obtain regular vine distributions, some notation is needed. 
% By \cite{bedford2002vines} and results of \cite{kurowicka2003parameterization} the edges of each tree of the vine tree sequence $\mathcal{V}$ can be uniquely identified by two \textit{conditioned nodes} and a set of \textit{conditioning nodes}.

\begin{definition}[Complete union, conditioning set, conditioned set] \label{def:complete_union}
    Let $\mathcal{V}$ be an vine tree sequence. The \textit{complete union} $U_e$ of the edge $e \in E_t$ is defined as:
    \begin{align*}
        U_e := \{ j \in V_1 \ | \ \exists e_1 \in E_1, ..., e_{t-1} \in E_{t-1} \quad  \text{s.th.} \quad  j \in e_1 \in ... \in e_{t-1} \in e \} .
    \end{align*}
    The set:
    \begin{align*}
        D_e := U_a \cap U_b
    \end{align*}
    is called \textit{conditioning set} $D_e$ of an edge $e = \{a, b\}$ and the \textit{conditioned sets} $\mathcal{C}_{e, a}$, $\mathcal{C}_{e, b}$ and $\mathcal{C}_e$ are given by:
    \begin{align*}
        \mathcal{C}_{e, a} := U_a \setminus D_e \, , \quad  \mathcal{C}_{e, b} := U_b \setminus D_e \quad \text{and}  \quad \mathcal{C}_e  := \mathcal{C}_{e, a} \cup \mathcal{C}_{e, b} \ .
    \end{align*}
\end{definition}

% \begin{definition}[Constraint set] \label{def:constraint_set}
%     The \textit{constraint set} $\mathcal{C} \mathcal{V}$ for the R-vine tree sequence $\mathcal{V}$ is defined as:
%     \begin{align}
%         \mathcal{C} \mathcal{V} := \big\{ ( \mathcal{C}_{e, a}, \mathcal{C}_{e, b} ; D_e ) \ | \ e = \{a, b\}, \ e \in E_i \quad \text{for} \quad i \in  [d-1] \big\} \ .
%     \end{align}
%     Here the edge $e = ( \mathcal{C}_{e, a}, \mathcal{C}_{e, b} ; D_e )$ of the R-vine tree sequence will often be abbreviated by $e = (e_a, e_b ; D_e)$.
% \end{definition}

Finally, we can piece everything together.

\begin{definition}[R-vine specification]
The triple $(\bm{F}, \mathcal{V}, B)$ is called \textit{R-vine specification} if:

\begin{enumerate}
    \item[(i)] $\bm{F} = (F_1, ..., F_d)$ is a vector of continuous and invertible distribution functions,
    \item[(ii)] $\mathcal{V}$ is an R-vine tree sequence on $d$ elements and
    \item[(iii)] $B := \big\{C_e \ | \ e \in E_t \quad \text{for} \quad t \in [d-1] \big\}$ is the set of bivariate copulas $C_e$ with $E_t$ the edge set of tree $T_t$ of the R-vine tree sequence $\mathcal{V}$.
\end{enumerate}

\end{definition}

By this definition each edge $e \in E_t$ of a tree $T_t$ in $\mathcal{V}$ corresponds to a bivariate copula $C_e$.

\begin{definition}[Regular vine distribution]\label{def:regular_vine_distribution}

A joint distribution $F$ of the random vector $\mathbf{X} = (X_1, ..., X_d)$ has a \textit{regular vine distribution}, if it realizes an R-vine specification $(\bm{F}, \mathcal{V}, B)$, i.e. if $C_e$ is the bivariate copula of $X_{\mathcal{C}_{e, a}}$ and $X_{\mathcal{C}_{e, b}}$ given $\mathbf{X}_{D_e}$ for each edge $e = \{a, b\} \in E_t$ and the marginal distribution of $X_j$ is $F_j$ for $j \in [d]$.

\end{definition}

\begin{theorem} \label{thm:bedford_cooke}

Let $(\bm{F}, \mathcal{V}, B)$ be an R-vine specification on $d$ elements where all pair copulas $C_e \in B$ satisfy the simplifying assumption and have densities $c_e$. There is a unique distribution $F$ that realizes this R-vine specification with density:
\begin{align*}
f_{1, ...d}(x_1, ..., x_d) &= \prod_{j = 1}^d f_j(x_j) \; \cdot \\
& \quad \ \prod_{t = 1}^{d-1} \prod_{e \in E_t} c_{ \mathcal{C}_{e, a}, \mathcal{C}_{e, b} ; D_e } \big( F_{\mathcal{C}_{e, a} | D_e}(x_{\mathcal{C}_{e, a}} | \mathbf{x}_{D_e} ), F_{\mathcal{C}_{e, b} | D_e}(x_{\mathcal{C}_{e, b}} | \mathbf{x}_{D_e}) \big) \; , 
\end{align*}
 
where $f_j$ denote the densities of $F_j$.
 
\end{theorem}

\begin{proof}
The proof of theorem can be found in \cite{bedford2001probability} and \cite{bedford2002vines}.
\end{proof}

\begin{definition}[Regular vine copula]
    A \textit{(regular) vine copula} is a regular vine distribution, where all margins are uniformly distributed on [0, 1].
\end{definition}

Note that for brevity, we often use the term vine copula when we mean a vine distribution as of Definition \ref{def:regular_vine_distribution}.

Vine copulas can be simplified by setting all pair copulas above a certain tree level to independence. This is called truncation.

\begin{definition}[Truncation of a vine copula at level $\tau$]
    The vine copula truncated at truncation level $\tau \in [d-1]$ is defined as: 
    \begin{align*}
        \prod_{t = 1}^{\tau} \prod_{e \in E_t} c_{ \mathcal{C}_{e, a}, \mathcal{C}_{e, b} ; D_e } \big( F_{\mathcal{C}_{e, a} | D_e}(x_{\mathcal{C}_{e, a}} | \mathbf{x}_{D_e} ), F_{\mathcal{C}_{e, b} | D_e}(x_{\mathcal{C}_{e, b}} | \mathbf{x}_{D_e}) \big) \; .
    \end{align*}
\end{definition}

\section{R\'enyi Divergence VI in the Gaussian Setting}\label{app:Renyi_divergence_Gaussian_setting}

For a slightly different $\alpha$-divergence, namely $D_{\alpha}(p||q) := \frac{1}{\alpha(\alpha-1)} \int \big(\frac{p(z)^{\alpha}}{q(z)^{\alpha}} -1 \big) q(z) dz$, \citet{margossian2024variational} prove the following results for $p(z) = p(z|x) = \mathcal{N}(\bm{\mu}, \Sigma)$ and $q(\mathbf{z}) = \mathcal{N}(\bm{\nu}, \Psi)$ with $\Psi=diag(\Psi_{jj})$: The $q$ that minimizes $D_{\alpha}(p||q)$ matches the mean of $p$ (Proposition 6), has finite and strictly positive variances $\Psi_{jj}$ (Proposition 7) and the covariance matrix $\Psi$ of $q$ satisfies the fixed point equations $\Psi_{jj} = [\alpha \Sigma^{-1}  + (1 - \alpha) \Psi^{-1}]^{-1}_{jj}$ (Proposition 8).

We prove the results of \citet{margossian2024variational} for the R\'enyi $\alpha$-divergence, which are relevant for the MF, that is part of the D-vine.

\begin{proposition}[Mean matching]\label{prop:mean_matching_Ralpah}
    Let $\alpha \in (0,1)$ and $p$ and $q$ be given as in Proposition 6 of \citet{margossian2024variational}. If $q$ minimizes $R_{\alpha}(q || p)$, then it matches the mean of $p$, i.e. $\bm{\nu} = \bm{\mu}$.
\end{proposition}

\begin{proof}
    \citet{gil2013renyi} give an explicit expression for the R\'enyi $\alpha$-divergence of two multivariate Gaussian distributions. Hence, we obtain:
    \begin{align}
        &\arg \min_{q \in Q} R_{\alpha}(q || p) = \arg \min_{q \in Q} \frac{\alpha}{2}(\bm{\nu} - \bm{\mu})^T [\alpha \Sigma + (1-\alpha)\Psi]^{-1} (\bm{\nu} - \bm{\mu}) - \frac{1}{2(\alpha - 1)} \log \Bigg( \frac{|\alpha \Sigma + (1-\alpha)\Psi|}{|\Psi|^{1 - \alpha} |\Sigma|^{\alpha}} \Bigg) \; . \label{eq:prop6_Ralpha}
    \end{align}
    The matrix $[\alpha \Sigma + (1-\alpha)\Psi]$ is positive definite and thus also its inverse. This makes the first summand of \eqref{eq:prop6_Ralpha} greater than or equal to 0; it is minimized at 0, which is the case for $\bm{\nu} = \bm{\mu}$.
\end{proof}

\begin{proposition}[Variance bounds]\label{prop:variance_bounds_Ralpah}
    Let $\alpha \in (0,1)$ and let $p$ and $q$ be given as in Proposition 7 of \citet{margossian2024variational}. If $q$ minimizes $R_{\alpha}(q || p)$, then its variances are strictly positive and finite, i.e. $0 < \Psi_{jj} < \infty$ for all $j$.
\end{proposition}

\begin{proof}
    Let $q$ minimize $R_{\alpha}(q || p)$. Then by Proposition \ref{prop:mean_matching_Ralpah} $\bm{\nu} = \bm{\mu}$ and the first term of \eqref{eq:prop6_Ralpha} vanishes. As $\alpha \in (0,1)$, $-\frac{1}{2(\alpha-1)}$ is positive and it suffices to consider the expression:
    \begin{align} \label{eq:variance_bounds_Ralpha_proof}
        &\log \Bigg( \frac{|\alpha \Sigma + (1-\alpha)\Psi|}{|\Psi|^{1 - \alpha} |\Sigma|^{\alpha}} \Bigg) = \log \big( C \cdot |\alpha \Sigma + (1-\alpha)\Psi| \cdot |\Psi|^{\alpha - 1} \big) \; ,
    \end{align}
    with $C := |\Sigma|^{-\alpha}$ constant. With similar arguments as in the proof of Proposition 7 in \citet{margossian2024variational} Equation \eqref{eq:variance_bounds_Ralpha_proof} diverges if any $\Psi_{jj} \rightarrow \infty$ due to $|\alpha \Sigma + (1-\alpha)\Psi|$, and likewise it diverges if any $\Psi_{jj} \rightarrow 0$ due to $|\Psi|^{\alpha - 1}$ and the fact that $\alpha - 1 < 0$.
\end{proof}

\begin{proposition}[Fixed-point equations]\label{prop:fixedpoint_equ_Ralpah}
    Let $\alpha \in (0,1)$ and let $p$ and $q$ be given as in Proposition 8 of \citet{margossian2024variational}. Then the $R_{\alpha}(q || p)$ is minimized when $\bm{\nu} = \bm{\mu}$ and the estimated variances from $\Psi$ satisfy the fixed-point equation:
    \begin{align*}
        diag(\Psi) = diag(\Phi_{\alpha}^{-1})
    \end{align*}
    with $\Phi_{\alpha} := \alpha \Psi^{-1} + (1-\alpha) \Sigma^{-1}$.
\end{proposition}

\begin{proof}
    With the same arguments as in the proof of Proposition \ref{prop:variance_bounds_Ralpah}, it suffices to consider Equation \eqref{eq:variance_bounds_Ralpha_proof} to find the $\Psi$ that minimizes $R_{\alpha}(q || p)$. Hence we consider:
    \begin{align*}
        &\arg \min_{\Psi} \log \big( C \cdot |\alpha \Sigma + (1-\alpha)\Psi| \cdot |\Psi|^{\alpha - 1} \big) \\
        &= \arg \min_{\Psi} \log \big( |\alpha \Sigma + (1-\alpha)\Psi| \big) + (\alpha - 1)\log \big( |\Psi| \big) \\
        % &= \arg \min_{\Psi} \log \big( |\alpha \Sigma + (1-\alpha)\Psi| \big) + (1 -\alpha) \log \big( |\Psi^{-1}| \big) \\
        &= \arg \min_{\Psi} \log \big( |\alpha \Sigma + (1-\alpha)\Psi| \big) + (1 -\alpha) \log \big( |\Psi^{-1}| \big) + \log \big( |\Psi^{-1}| \big) - \log \big( |\Psi^{-1}| \big)\\
        &= \arg \min_{\Psi^{-1}} \log \big( |\alpha \Sigma \Psi^{-1} + (1-\alpha)| \big) - \alpha \log \big( |\Psi^{-1}| \big) \\
        &= \arg \min_{\Psi^{-1}} \log \big( |\alpha \Psi^{-1} + (1-\alpha) \Sigma^{-1}| \big) - \alpha \log \big( |\Psi^{-1}| \big)
    \end{align*}
    where we have exploited that $\log\big( |\Sigma^{-1}| \big)$ is an additive constant. Solving for the minimum we find:
    \begin{align*}
        0 &= \frac{\partial}{\partial \Psi_{jj}^{-1}} \log \big( |\alpha \Psi^{-1} + (1-\alpha) \Sigma^{-1}| \big) - \alpha \log \big( |\Psi^{-1}| \big) = \alpha \big[ \Phi_{\alpha}^{-1} - \Psi \big]_{jj}
    \end{align*}
    with $\Phi_{\alpha} := \alpha \Psi^{-1} + (1-\alpha) \Sigma^{-1}$. Here we have used Jacobi's formula \citep{magnus2019matrix} for derivatives of determinants:
    \begin{align*}
        \frac{\partial}{\partial M_{jj}} \log |M| = \frac{1}{|M|} \frac{\partial |M|}{\partial M_{jj}} = \frac{1}{|M|} adj(M)_{jj} = \frac{1}{|M|} |M| (M^{-1})_{jj} \; ,
    \end{align*}
    where we have used in the last step that for an invertible matrix $M$ that $adj(M) = |M| M^{-1}$.
\end{proof}

This also explicitly proves the claim of \citet{margossian2024variational} that in the Gaussian setup of Propositions \ref{prop:mean_matching_Ralpah}, \ref{prop:variance_bounds_Ralpah} and \ref{prop:fixedpoint_equ_Ralpah}, $D_{\alpha}$ and $R_{\alpha}$ give the same VI optimization.\footnote{Note that we flipped $p$ and $q$ in the argument of the divergences. This is the reason for the flipped $\Sigma$ and $\Psi$ in $\Phi_{\alpha}$ compared to \citet{margossian2024variational}.} 
% This notably also means that as we decrease $\alpha$ towards 0, optimizing $R_{\alpha}(q||p)$ comes closer to optimizing $KL(p||q)$. 

\section{Forward and Backward Divergence VI in the Gaussian Setting}\label{app:fw_bw_KL_divergence_Gaussian_setting}

\subsection{Forward KL: Proof of Theorem \ref{thm:forwardKL}}\label{app:forwardKL_proof}

\begin{proof}[Proof of Theorem \ref{thm:forwardKL}]
First note that when $p$ and $q$ are Gaussian with the parameters specified before 
Theorem \ref{thm:forwardKL}, the forward KL divergence is given by 
\begin{align}
\begin{split}
KL(p||q) = &\frac{1}{2}\left(tr\left(\B{\Psi}^{-1}\B{\Sigma}\right)-d-\log|\B{\Sigma}|+\log|\B{\Psi}|+(\B{\mu}-\B{\nu})^{T}\B{\Psi}^{-1}(\B{\mu}-\B{\nu})\right).
\end{split}
\label{eqn:KLpqgauss}
\end{align}
Also note that since $q$ is a Gaussian D-vine, the parameters $\eta_{i,i+1}$, $i=1,\ldots,d-1$ 
of its first tree are the correlations $R_{\Psi,i,i+1}$ and the parameters 
$\eta_{i,i+t|i+1,\ldots,i+t-1}$ of tree $t$ are given by the partial correlations
of $(Z_{i},Z_{i+t})$ given $Z_{i+1},\ldots,Z_{i+t-1}$, for $t=2,\ldots,d-1$, 
$i=1,\ldots,d-j$, and the determinant of $\Bb{R}_{\Psi}$ is given by \citep{cooke2006}
\begin{align}
|\Bb{R}_{\Psi}| = \prod_{j=1}^{d-2}\prod_{i=1}^{d-j}\left(1-\eta_{i,i+j|i+1,\ldots,i+j-1}^{2}\right).
\label{eqn:detRpsi}
\end{align}

In the \textbf{first step}, the marginal parameters $\B{\nu}$ and $\Bb{D}_{\Psi}$ are 
estimated by mean-field, which consists in minimising the forward KL, assuming
independence between $Z_{1},\ldots,Z_{d}$. As shown by 
\cite{margossian2024variational}, one then obtains $\B{\nu}=\B{\mu}$  and
$\Bb{D}_{\Psi}=\Bb{D}_{\Sigma}$. The forward KL then reduces to
\[
KL(p||q) = \frac{1}{2}\left(tr\left(\Bb{R}_{\Psi}^{-1}\B{\Bb{R}_{\Sigma}}\right)-d-\log|\Bb{R}_{\Sigma}|+\log|\Bb{R}_{\Psi}|\right).
\]

The \textbf{second step} is to find the parameters of the first tree of the D-vine that 
maximize the forward KL when all the remaining $d-2$ trees are set to independence. The 
parameters of the first tree are $\eta_{i,i+1}=R_{\Psi,i,i+1}$, $i=1,\ldots,d-1$. Setting 
the copulas in the remaining trees to independence is the same as setting the 
partial correlations $\eta_{i,i+j|i+1,\ldots,i+j-1}$ to $0$, for $j=2,\ldots,d-2$, 
$i=1,\ldots,d-j$. The remaining elements of $\Bb{R}_{\Psi}$ are the constrained to be 
$R_{\Psi,i,i+j}=\prod_{k=1}^{j}R_{\Psi,i+k-1,i+k}$, for $j=2,\ldots,d-1$, $i=1,\ldots,d-i$.
Further, according to \eqref{eqn:detRpsi}, the determinant of $\Bb{R}_{\Psi}$ 
becomes
\[
|\Bb{R}_{\Psi}| = \prod_{i=1}^{d-1}\left(1-\eta_{i,i+1}^{2}\right).
\]
Let $\Bb A = \Bb{R}_{\Psi}^{-1}\B{\Bb{R}_{\Sigma}}$.The first term of the KL divergence is 
then given by
\[
\frac{1}{2}tr\left(\Bb{R}_{\Psi}^{-1}\B{\Bb{R}_{\Sigma}}\right) =  \frac{1}{2}\sum_{i=1}^{d}A_{i,i},
\]
where the $A_{i,i}$s are the diagonal elements of $\Bb A$, which are given by
\[
A_{i,i}=\left(\Bb{R}_{\Psi}^{-1}\right)_{i,i}+\sum_{j\neq i}R_{\Sigma,i,j}\left(\Bb{R}_{\Psi}^{-1}\right)_{i,j},
\]
where, according to Cramer's rule, 
$\left(\Bb{R}_{\Psi}^{-1}\right)_{i,j}=\frac{(-1)^{i+j}}{|\Bb{R}_{\Psi}|}|\Bb{R}_{\Psi,-j,-i}|$,
$\Bb{R}_{\Psi,-j,-i}$ being $\Bb{R}_{\Psi}$ with row $j$ and column $i$ deleted. It
is straightforward to show that 
$|\Bb{R}_{\Psi,-1,-1}| = \prod_{i=2}^{d-1}\left(1-\eta_{i,i+1}^{2}\right)$, so
that $\left(\Bb{R}_{\Psi}^{-1}\right)_{1,1}=\frac{1}{1-\eta_{12}^{2}}$, and correspondingly
$\left(\Bb{R}_{\Psi}^{-1}\right)_{d,d}=\frac{1}{1-\eta_{d-1,d}^{2}}$, and also
that for $i=2,\ldots,d-1$,\\
$|\Bb{R}_{\Psi,-i,-i}| = (1-\eta_{12}^{2})\cdot\ldots\cdot(1-\eta_{i-2,i-1}^{2})\cdot(1-\eta_{i-1,i}^{2})\cdot(1-\eta_{i,i+1}^{2})\cdot(1-\eta_{i+1,i+2}^{2})\cdot\ldots\cdot(1-\eta_{d-1,d}^{2})$, so that
$\left(\Bb{R}_{\Psi}^{-1}\right)_{i,i}=\frac{1}{1-\eta_{12}^{2}}=\frac{1-\eta_{i-1,i}^{2}\eta_{i,i+1}^{2}}{(1-\eta_{i-1,i}^{2})(1-\eta_{i,i+1}^{2})}$. Further, we have
\begin{align*}
|\Bb{R}_{\Psi,-2,-1}| = |\Bb{R}_{\Psi,-1,-2}| = &\eta_{12}|\Bb{R}_{\Psi,-1,-1}| = \eta_{12}\prod_{i=2}^{d-1}\left(1-\eta_{i,i+1}^{2}\right),
\end{align*}
so that $\left(\Bb{R}_{\Psi}^{-1}\right)_{1,2}=-\frac{\eta_{12}}{1-\eta_{1,2}^{2}}$,
and correspondingly
$\left(\Bb{R}_{\Psi}^{-1}\right)_{d-1,d}=-\frac{\eta_{d-1,d}}{1-\eta_{d-1,d}^{2}}$.
Moreover, using the Laplace expansion along the first row, we have
\begin{scriptsize}
\begin{align*}
&|\Bb{R}_{\Psi,-3,-2}| = |\Bb{R}_{\Psi,-2,-3}|\\ 
= &\eta_{23}|\Bb{R}_{\Psi,-(1,2),-(1,2)}|-\eta_{,23}\eta_{12}^{2}|\Bb{R}_{\Psi,-(1,2),-(1,2)}|,
\end{align*}
\end{scriptsize}
where the remaining terms becomne $0$, as all the corresponding determinants are
taken of matrices where the second column is propotional to the first by the factor
$\eta_{12}$. This gives
\[
|\Bb{R}_{\Psi,-3,-2}| = \eta_{2,3}(1-\eta_{1,2}^{2})\prod_{i=3}^{d-1}\left(1-\eta_{i,i+1}^{2}\right),
\]
such that $\left(\Bb{R}_{\Psi}^{-1}\right)_{2,3}=-\frac{\eta_{2,3}}{1-\eta_{2,3}^{2}}$,
and correspondingly for $i=3,\ldots,d-2$,
$\left(\Bb{R}_{\Psi}^{-1}\right)_{i,i+1}=-\frac{\eta_{i,i+1}}{1-\eta_{i,i+1}^{2}}$.
Finally, as $\eta_{i,i+j|i+1,\ldots,i+j-1}= 0$ for $j=1,\ldots,d-2$, $i=1,\ldots,d-j$,
$\left(\Bb{R}_{\Psi}^{-1}\right)_{i,j}=\left(\Bb{R}_{\Psi}^{-1}\right)_{j,i}=0$ for
$j=1,\ldots,d-2$, $i=1,\ldots,d-j$. Putting all this together, we obtain
\begin{align*}
A_{1,1} = &\frac{1-R_{\Sigma,1,2}\eta_{1,2}}{1-\eta_{1,2}^{2}}\\
A_{d,d} = &\frac{1-R_{\Sigma,d-1,d}\eta_{d-1,d}}{1-\eta_{d-1,d}^{2}}\\
A_{i,i} = &\frac{1-R_{\Sigma,i-1,i}\eta_{i-1,i}}{1-\eta_{i-1,i}^{2}}+\frac{1-R_{\Sigma,i,i+1}\eta_{i,i+1}}{1-\eta_{i,i+1}^{2}}, \quad i=2,\ldots,d-1,
\end{align*}
which gives
\begin{align*}
KL(p||q) = & \sum_{i=1}^{d-1}\frac{1-R_{\Sigma,i,i+1}\eta_{i,i+1}^{2}}{1-\eta_{i,i+1}^{2}}-\frac{d-2}{2}-\frac{1}{2}\log|\Bb{R}_{\Sigma}|+\frac{1}{2}\sum_{i=1}^{d-1}\log\left(1-\eta_{i,i+1}^{2}\right)
\end{align*}
Hence, the derivative of the KL divergence with respect to the parameters of the first 
tree are given by
\[
\frac{\partial KL(p||q)}{\partial \eta_{i,i+1}} = \frac{1+\eta_{i,i+1}^{2}}{(1-\eta_{i,i+1}^2)^2}(\eta_{i,i+1}-R_{\Sigma,i,i+1}), \quad i=1,\ldots,d-1
\]
which means that the minimum is obtained for 
$R_{\Psi,i,i+1}=\eta_{i,i+1}=R_{\Sigma,i,i+1}$, $i=1,\ldots,d-1$.

The \textbf{third step} (if $d \geq 3$, otherwise the second step is the last) is to find the 
parameters $\eta_{i,i+2|i+1}$, $i=1,\ldots,d-2$ of the second tree of the D-vine that 
maximize the forward KL when all the remaining $d-3$ trees are set to independence, 
which corresponds to setting $\eta_{i,i+j|i+1,\ldots,i+j-1}=0$, for $j=2,\ldots,d-1$, 
$i=1,\ldots,d-j$. This means that 
$\eta_{i,i+2}=R_{\Sigma,i,i+1}R_{\Sigma,i+1,i+2}+\eta_{i,i+2|i+1}\sqrt{(1-R_{\Sigma,i,i+1}^{2})(1-R_{\Sigma,i+1,i+2}^{2})}$, $i=1,\ldots,d-2$.
Further, we have
\[
|\Bb{R}_{\Psi}| = \prod_{i=1}^{d-1}\left(1-R_{\Sigma,i,i+1}^{2}\right)\cdot\prod_{k=1}^{d-2}\left(1-\eta_{k,k+2|k+1}^{2}\right)
\]
and
\begin{align*}
A_{1,1} = &\left(\Bb{R}_{\Psi}^{-1}\right)_{1,1}+R_{\Sigma,1,2}\left(\Bb{R}_{\Psi}^{-1}\right)_{1,2}+R_{\Sigma,1,3}\left(\Bb{R}_{\Psi}^{-1}\right)_{1,3}\\
A_{d,d} = &\left(\Bb{R}_{\Psi}^{-1}\right)_{d,d}+R_{\Sigma,d-1,d}\left(\Bb{R}_{\Psi}^{-1}\right)_{d-1,d}+R_{\Sigma,d-2,d}\left(\Bb{R}_{\Psi}^{-1}\right)_{d-2,d}\\
A_{2,2} = & \left(\Bb{R}_{\Psi}^{-1}\right)_{2,2}+R_{\Sigma,1,2}\left(\Bb{R}_{\Psi}^{-1}\right)_{1,2}+R_{\Sigma,2,3}\left(\Bb{R}_{\Psi}^{-1}\right)_{2,3}+I(d > 3)R_{\Sigma,2,4}\left(\Bb{R}_{\Psi}^{-1}\right)_{2,4}\\
A_{d-1,d-1} = & \left(\Bb{R}_{\Psi}^{-1}\right)_{d-1,d-1}+I(d > 3)R_{\Sigma,d-3,d-1}\left(\Bb{R}_{\Psi}^{-1}\right)_{d-3,d-1}\\
&+R_{\Sigma,d-2,d-1}\left(\Bb{R}_{\Psi}^{-1}\right)_{d-2,d-1}+R_{\Sigma,d-1,d}\left(\Bb{R}_{\Psi}^{-1}\right)_{d-1,d}\\
A_{i,i} = & \left(\Bb{R}_{\Psi}^{-1}\right)_{i,i}+R_{\Sigma,i-2,i}\left(\Bb{R}_{\Psi}^{-1}\right)_{i-2,i}+R_{\Sigma,i-1,i}\left(\Bb{R}_{\Psi}^{-1}\right)_{i-1,i}+R_{\Sigma,i,i+1}\left(\Bb{R}_{\Psi}^{-1}\right)_{i,i+1}\\
&+R_{\Sigma,i,i+2}\left(\Bb{R}_{\Psi}^{-1}\right)_{i,i+2}, \quad i=3,\ldots,d-2, \mbox{ for } d > 3.
\end{align*}
Once more using Cramer's rule combined with Laplace expansion, as well as the
constraints imposed by the partial correlations set to $0$, we obtain
\begin{align*}
\left(\Bb{R}_{\Psi}^{-1}\right)_{1,1} = &\frac{1}{(1-R_{\Sigma,1,2}^{2})(1-\eta_{1,3|2}^{2})}\\
\left(\Bb{R}_{\Psi}^{-1}\right)_{d,d} = &\frac{1}{(1-R_{\Sigma,d-1,d}^{2})(1-\eta_{d-2,d|d-1}^{2})}\\
\left(\Bb{R}_{\Psi}^{-1}\right)_{2,2} = &\frac{R_{\Sigma,1,2}^{2}}{(1-R_{\Sigma,1,2}^{2})(1-\eta_{1,3|2}^{2})}-\frac{2R_{\Sigma,1,2}R_{\Sigma,2,3}(\eta_{1,3}-R_{\Sigma,1,2}R_{\Sigma,2,3})}{(1-R_{\Sigma,1,2}^{2})(1-R_{\Sigma,2,3}^{2})(1-\eta_{1,3|2}^{2})}\\
&+I(d = 3)\frac{1-(1-R_{\Sigma,2,3}^{2})\eta_{1,3|2}^{2}}{(1-R_{\Sigma,2,3}^{2})(1-\eta_{1,3|2}^{2})}\\
&+I(d > 3)\frac{(1-R_{\Sigma,2,3}^{2})(1-\eta_{1,3|2}^{2})+R_{\Sigma,2,3}^{2}(1-\eta_{1,3|2}^{2}\eta_{2,4|3}^{2})}{(1-R_{\Sigma,2,3}^{2})(1-\eta_{1,3|2}^{2})(1-\eta_{2,4|3}^{2})}\\
\left(\Bb{R}_{\Psi}^{-1}\right)_{1,2} = &-\frac{R_{\Sigma,1,2}}{(1-R_{\Sigma,1,2}^{2})(1-\eta_{1,3|2}^{2})}+\frac{R_{\Sigma,2,3}(\eta_{1,3}-R_{\Sigma,1,2}R_{\Sigma,2,3})}{(1-R_{\Sigma,1,2}^{2})(1-R_{\Sigma,2,3}^{2})(1-\eta_{1,3|2}^{2})}\\
\left(\Bb{R}_{\Psi}^{-1}\right)_{d-1,d} = &-\frac{R_{\Sigma,d-1,d}}{(1-R_{\Sigma,d-1,d}^{2})(1-\eta_{d-2,d|d-1}^{2})}+\frac{R_{\Sigma,d-2,d-1}(\eta_{d-2,d}-R_{\Sigma,d-2,d-1}R_{\Sigma,d-1,d})}{(1-R_{\Sigma,d-2,d-1}^{2})(1-R_{\Sigma,d-1,d}^{2})(1-\eta_{d-2,d|d-1}^{2})}\\
\left(\Bb{R}_{\Psi}^{-1}\right)_{i,i+2} = &-\frac{\eta_{i,i+2}-R_{\Sigma,i,i+1}R_{\Sigma,i+1,i+2}}{(1-R_{\Sigma,i,i+1}^{2})(1-R_{\Sigma,i+1,i+2}^{2})(1-\eta_{i,i+2|i+1}^{2})}, \quad i=1,\ldots,d-2.
\end{align*}
and for $d > 3$
\begin{align*}
\left(\Bb{R}_{\Psi}^{-1}\right)_{d-1,d-1} = &\frac{R_{\Sigma,d-1,d}^{2}}{(1-R_{\Sigma,d-1,d}^{2})(1-\eta_{d-2,d|d-1}^{2})}-\frac{2R_{\Sigma,d-2,d-1}R_{\Sigma,d-1,d}(\eta_{d-2,d}-R_{\Sigma,d-2,d-1}R_{\Sigma,d-1,d})}{(1-R_{\Sigma,d-2,d-1}^{2})(1-R_{\Sigma,d-1,d}^{2})(1-\eta_{d-2,d|d-1}^{2})}\\
&\frac{(1-R_{\Sigma,d-2,d-1}^{2})(1-\eta_{d-2,d|d}^{2})+R_{\Sigma,d-2,d-1}^{2}(1-\eta_{d-3,d-1|d-2}^{2}\eta_{d-2,d|d-1}^{2})}{(1-R_{\Sigma,d-2,d-1}^{2})(1-\eta_{d-3,d-1|d-2}^{2})(1-\eta_{d-2,d|d-1}^{2})}\\
\left(\Bb{R}_{\Psi}^{-1}\right)_{i,i} = & \frac{(1-R_{\Sigma,i-1,i}^{2})(1-\eta_{i-1,i+1|i}^{2})+R_{\Sigma,i-1,i}^{2}(1-\eta_{i-2,i|i-1}^{2}\eta_{i-1,i+1|i}^{2})}{(1-R_{\Sigma,i-1,i}^{2})(1-\eta_{i-2,i|i-1}^{2})(1-\eta_{i-1,i+1|i}^{2})}\\
&+\frac{(1-R_{\Sigma,i,i+1}^{2})(1-\eta_{i-1,i+1|i}^{2})+R_{\Sigma,i,i+1}^{2}(1-\eta_{i-1,i+1|i}^{2}\eta_{i,i+2|i+1}^{2})}{(1-R_{\Sigma,i,i+1}^{2})(1-\eta_{i-1,i+1|i}^{2})(1-\eta_{i,i+2|i+1}^{2})}\\
&-\frac{2R_{\Sigma,i-1,i}R_{\Sigma,i,i+1}(\eta_{i-1,i}-R_{\Sigma,i-1,i}R_{\Sigma,i,i+1})}{(1-R_{\Sigma,i-1,i}^{2})(1-R_{\Sigma,i,i+1}^{2})(1-\eta_{i-1,i+1|i}^{2})}, \quad i=3,\ldots,d-2\\
\left(\Bb{R}_{\Psi}^{-1}\right)_{i,i+1} = &-\frac{R_{\Sigma,i-1,i}(1-\eta_{i-1,i+1|i}^2\eta_{i,i+2|i+1}^2)}{(1-R_{\Sigma,i,i+1}^{2})(1-\eta_{i-1,i+1|i}^{2})(1-\eta_{i,i+2|i+1}^{2})}\\
&+\frac{R_{\Sigma,i-1,i}(\eta_{i-1,i+1}-R_{\Sigma,i-1,i}R_{\Sigma,i,i+1})}{(1-R_{\Sigma,i-1,i}^{2})(1-R_{\Sigma,i,i+1}^{2})(1-\eta_{i-1,i+1|i}^{2})}\\
&+\frac{R_{\Sigma,i+1,i+2}(\eta_{i,i+2}-R_{\Sigma,i,i+1}R_{\Sigma,i+1,i+2})}{(1-R_{\Sigma,i,i+1}^{2})(1-R_{\Sigma,i+1,i+2}^{2})(1-\eta_{i,i+2|i+1}^{2})}, \quad i=2,\ldots,d-2.
\end{align*}
Now, let $\rho_{i,i+j|i+1,\ldots,i+j-1}$ be the partial correlation between $Z_{i}$
and $Z_{i+j}$, given $Z_{i+1},\ldots,Z_{i+j-1}$, from the correlation matrix 
$\Bb R_{\Sigma}$ of the true posterior. Then, we get 
\begin{align*}
A_{11} = &\frac{1-\rho_{1,3|2}\eta_{1,3|2}}{1-\eta_{1,3|2}^{2}}\\
A_{dd} = &\frac{1-\rho_{d-2,d|d-1}\eta_{d-2,d|d-1}}{1-\eta_{d-2,d|d-1}^{2}}\\
A_{22} = &\begin{cases}
1, & d = 3\\
\frac{1-\rho_{2,4|3}\eta_{2,4|3}}{1-\eta_{2,4|3}^{2}}, & d > 3
\end{cases}\\
A_{d-1,d-1} = & \frac{1-\rho_{d-3,d-1|d-2}\eta_{d-3,d-1|d-2}}{1-\eta_{d-3,d-1|d-2}^{2}}, \quad d > 3\\
A_{i,i} = & \frac{1-\rho_{i-2,i|i-1}\eta_{i-2,i|i-1}}{1-\eta_{i-2,i|i-1}^{2}}+\frac{1-\rho_{i,i+2|i+1}\eta_{i,i+2|i+1}}{1-\eta_{i,i+2|i+1}^{2}}, \quad i = 3,\ldots,d-2,
\end{align*}
which gives 
\begin{align*}
KL(p||q) = & \sum_{i=1}^{d-2}\frac{1-\rho_{i,i+2|i+1}\eta_{i,i+2|i+1}}{1-\eta_{i,i+2|i+1}^{2}}-\frac{d-4}{2}-\frac{1}{2}\log|\Bb{R}_{\Sigma}|+\frac{1}{2}\sum_{i=1}^{d-1}\log\left(1-R_{\Sigma,i,i+1}^{2}\right)\\
&+\frac{1}{2}\sum_{i=1}^{d-2}\log\left(1-\eta_{i,i+2|i+1}^{2}\right)
\end{align*}
Hence, the derivative of the KL divergence with respect to the parameters of the second 
tree are given by
\[
\frac{\partial KL(p||q)}{\partial \eta_{i,i+1}} = \frac{1+\eta_{i,i+2|i+1}^{2}}{(1-\eta_{i,i+2|i+1}^2)^2}(\eta_{i,i+2|i+1}-\rho_{i,i+2|i+1}), \quad i=1,\ldots,d-1
\]
which means that the minimum is obtained for 
$\eta_{i,i+2|i+1}=\rho_{i,i+2|i+1}$, $i=1,\ldots,d-2$, which again means that
\[
R_{\Psi,i,i+2}=R_{\Sigma,i,i+1}R_{\Sigma,i+1,i+2}+\rho_{i,i+2|i+1}\sqrt{(1-R_{\Sigma,i,i+1}^{2})(1-R_{\Sigma,i+1,i+2}^{2})}=R_{\Sigma,i,i+2}, \quad i=1,\ldots,d-2.
\]
The \textbf{fourth step} (if $d \geq 4$, otherwise the third step is the last) is to find the 
parameters $\eta_{i,i+3|i+1,i+2}$, $i=1,\ldots,d-3$ of the third tree of the D-vine that 
maximize the forward KL when all the remaining $d-4$ trees are set to independence, 
which corresponds to setting $\eta_{i,i+j|i+1,\ldots,i+j-1}=0$, for $j=3,\ldots,d-2$, 
$i=1,\ldots,d-i$. We have (consult for instance \cite{baba2004})
\begin{small}
\begin{align*}
&\eta_{i,i+3|i+1,i+2}\\
= & \frac{R_{\Psi,i,i+3}-\Bb R_{\Psi,(i+1,i+2),i}^{T}\Bb R_{\Psi,(i+1,i+2),(i+1,i+2)}^{-1}\Bb R_{\Psi,(i+1,i+2),i+3}}{\sqrt{(1- \Bb R_{\Psi,(i+1,i+2),i}^{T}\Bb R_{\Psi,(i+1,i+2),(i+1,i+2)}^{-1}\Bb R_{\Psi,(i+1,i+2),i})(1-\Bb R_{\Psi,(i+1,i+2),i+3}^{T}\Bb R_{\Psi,(i+1,i+2),(i+1,i+2)}^{-1}\Bb R_{\Psi(i+1,i+2),i+3})}}\\
= &\frac{R_{\Psi,i,i+3}-\Bb R_{\Sigma,(i+1,i+2),i}^{T}\Bb R_{\Sigma,(i+1,i+2),(i+1,i+2)}^{-1}\Bb R_{\Sigma,(i+1,i+2),i+3}}{\sqrt{(1-\Bb R_{\Sigma,(i+1,i+2),i}^{T}\Bb R_{\Sigma,(i+1,i+2),(i+1,i+2)}^{-1}\Bb R_{\Sigma,(i+1,i+2),i})(1-\Bb R_{\Sigma(i+1,i+2),i+3}^{T}\Bb R_{\Sigma,(i+1,i+2),(i+1,i+2)}^{-1}\Bb R_{\Sigma,(i+1,i+2),i+3})}},
\end{align*}
\end{small}
and
\begin{small}
\begin{align*}
&\eta_{i,i+4|i+1 \ldots i+3}\\
= & \frac{R_{\Psi,i,i+4}-\Bb R_{\Psi,i+1 \ldots i+3,i}^{T}\Bb R_{\Psi,i+1 \ldots i+3,i+1 \ldots i+3}^{-1}\Bb R_{\Psi,i+1\ldots i+3,i+4}}{\sqrt{(1- \Bb R_{\Psi,i+1 \ldots i+3,i}^{T}\Bb R_{\Psi,i+1 \ldots i+3,i+1 \ldots i+3}^{-1}\Bb R_{\Psi,i+1 \ldots i+3,i})(1-\Bb R_{\Psi,i+1 \ldots i+3,i+4}^{T}\Bb R_{\Psi,i+1 \ldots i+3,i+1 \ldots i+3}^{-1}\Bb R_{\Psi,i+1 \ldots i+3,i+4})}}\\
& = 0,
\end{align*}
\end{small}
so that
\begin{scriptsize}
\begin{align*}
&R_{\Psi,i,i+3} =\\ 
&\Bb R_{\Sigma,(i+1,i+2),i}^{T}\Bb R_{\Sigma,(i+1,i+2),(i+1,i+2)}^{-1}\Bb R_{\Sigma,(i+1,i+2),i+3}\\
&+\eta_{i,i+3|i+1,i+2}\sqrt{(1-\Bb R_{\Sigma,(i+1,i+2),i}^{T}\Bb R_{\Sigma,(i+1,i+2),(i+1,i+2)}^{-1}\Bb R_{\Sigma,(i+1,i+2),i})(1-\Bb R_{\Sigma(i+1,i+2),i+3}^{T}\Bb R_{\Sigma,(i+1,i+2),(i+1,i+2)}^{-1}\Bb R_{\Sigma,(i+1,i+2),i+3})}
\end{align*}
\end{scriptsize}
and
\begin{align*}
R_{\Psi,i,i+4} = &\Bb R_{\Psi,i+1 \ldots i+3,i}^{T}\Bb R_{\Psi,i+1 \ldots i+3,i+1 \ldots i+3}^{-1}\Bb R_{\Psi,i+1 \ldots i+3,i+4}, \quad i = 1,\ldots,d-4.
\end{align*}
Based on this it is straightforward to find that
and also the following relationships, that are 
straightforward to obtain
\begin{align*}
(1-R_{\Sigma,i,i+1}^2)(1-\rho_{i,i+2|i+1}^2) = &1-\Bb R_{\Sigma,(i+1,i+2),i}^{T}\Bb R_{\Sigma,(i+1,i+2),(i+1,i+2)}^{-1}\Bb R_{\Sigma,(i+1,i+2),i}\\
(1-R_{\Sigma,i+1,i+2}^2)(1-\rho_{i,i+2|i+1}^2) = &1-\Bb R_{\Sigma,(i,i+1),i+2}^{T}\Bb R_{\Sigma,(i,i+1),(i,i+1)}^{-1}\Bb R_{\Sigma,(i,i+1),i+2}.
\end{align*}
Further, we have
\[
|\Bb{R}_{\Psi}| = \prod_{i=1}^{d-1}\left(1-R_{\Sigma,i,i+1}^{2}\right)\cdot\prod_{k=1}^{d-2}\left(1-\rho_{k,k+2|k+1}^{2}\right)\cdot\prod_{k=1}^{d-3}\left(1-\eta_{k,k+3|k+1,k+2}^{2}\right)
\]
and
\begin{small}
\begin{align*}
A_{1,1} = &\left(\Bb{R}_{\Psi}^{-1}\right)_{1,1}+R_{\Sigma,1,2}\left(\Bb{R}_{\Psi}^{-1}\right)_{1,2}+R_{\Sigma,1,3}\left(\Bb{R}_{\Psi}^{-1}\right)_{1,3}+R_{\Sigma,1,4}\left(\Bb{R}_{\Psi}^{-1}\right)_{1,4}\\
A_{d,d} = &\left(\Bb{R}_{\Psi}^{-1}\right)_{d,d}+R_{\Sigma,d-1,d}\left(\Bb{R}_{\Psi}^{-1}\right)_{d-1,d}+R_{\Sigma,d-2,d}\left(\Bb{R}_{\Psi}^{-1}\right)_{d-2,d}+R_{\Sigma,d-3,d}\left(\Bb{R}_{\Psi}^{-1}\right)_{d-3,d}\\
A_{2,2} = & \left(\Bb{R}_{\Psi}^{-1}\right)_{2,2}+R_{\Sigma,1,2}\left(\Bb{R}_{\Psi}^{-1}\right)_{1,2}+R_{\Sigma,2,3}\left(\Bb{R}_{\Psi}^{-1}\right)_{2,3}+R_{\Sigma,2,4}\left(\Bb{R}_{\Psi}^{-1}\right)_{2,4}+I(d > 4)R_{\Sigma,2,5}\left(\Bb{R}_{\Psi}^{-1}\right)_{2,5}\\
A_{d-1,d-1} = & \left(\Bb{R}_{\Psi}^{-1}\right)_{d-1,d-1}+I(d > 4)R_{\Sigma,d-4,d-1}\left(\Bb{R}_{\Psi}^{-1}\right)_{d-4,d-1}+R_{\Sigma,d-3,d-1}\left(\Bb{R}_{\Psi}^{-1}\right)_{d-3,d-1}\\
&+R_{\Sigma,d-2,d-1}\left(\Bb{R}_{\Psi}^{-1}\right)_{d-2,d-1}+R_{\Sigma,d-1,d}\left(\Bb{R}_{\Psi}^{-1}\right)_{d-1,d}\\
A_{3,3} = & \left(\Bb{R}_{\Psi}^{-1}\right)_{3,3}+R_{\Sigma,1,3}\left(\Bb{R}_{\Psi}^{-1}\right)_{1,3}+R_{\Sigma,2,3}\left(\Bb{R}_{\Psi}^{-1}\right)_{2,3}+R_{\Sigma,3,4}\left(\Bb{R}_{\Psi}^{-1}\right)_{3,4}\\
&+I(d > 4)R_{\Sigma,3,5}\left(\Bb{R}_{\Psi}^{-1}\right)_{3,5}+I(d > 5)R_{\Sigma,3,6}\left(\Bb{R}_{\Psi}^{-1}\right)_{3,6}\\
A_{d-2,d-2} = & \left(\Bb{R}_{\Psi}^{-1}\right)_{d-2,d-2}+I(d > 5)R_{\Sigma,d-5,d-2}\left(\Bb{R}_{\Psi}^{-1}\right)_{d-5,d-2}+I(d > 4)R_{\Sigma,d-4,d-2}\left(\Bb{R}_{\Psi}^{-1}\right)_{d-4,d-2}\\
&+R_{\Sigma,d-3,d-2}\left(\Bb{R}_{\Psi}^{-1}\right)_{d-3,d-2}+R_{\Sigma,d-2,d-1}\left(\Bb{R}_{\Psi}^{-1}\right)_{d-2,d-1}+R_{\Sigma,d-2,d}\left(\Bb{R}_{\Psi}^{-1}\right)_{d-2,d}\\
A_{i,i} = & \left(\Bb{R}_{\Psi}^{-1}\right)_{i,i}+R_{\Sigma,i-3,i}\left(\Bb{R}_{\Psi}^{-1}\right)_{i-3,i}+R_{\Sigma,i-2,i}\left(\Bb{R}_{\Psi}^{-1}\right)_{i-2,i}+R_{\Sigma,i-1,i}\left(\Bb{R}_{\Psi}^{-1}\right)_{i-1,i}\\
&+R_{\Sigma,i,i+1}\left(\Bb{R}_{\Psi}^{-1}\right)_{i,i+1}+R_{\Sigma,i,i+2}\left(\Bb{R}_{\Psi}^{-1}\right)_{i,i+2}+R_{\Sigma,i,i+3}\left(\Bb{R}_{\Psi}^{-1}\right)_{i,i+3}, \quad i=4,\ldots,d-3, \mbox{ for } d > 6.
\end{align*}
\end{small}
Again, we use Cramer's rule combined with Laplace expansion, as well as the constraints 
imposed by the partial correlations set to $0$ to obtain
\begin{tiny}
\begin{align*}
&\left(\Bb{R}_{\Psi}^{-1}\right)_{1,1} = \frac{1}{(1-\Bb R_{\Sigma,(2,3),1}^{T}\Bb R_{\Sigma,(2,3),(2,3)}^{-1}\Bb R_{\Sigma,(2,3),1})(1-\eta_{1,4|2,3}^{2})}\\
&\left(\Bb{R}_{\Psi}^{-1}\right)_{d,d} = \frac{1}{(1-\Bb R_{\Sigma,(d-2,d-1),d}^{T}\Bb R_{\Sigma,(d-2,d-1),(d-2,d-1)}^{-1}\Bb R_{\Sigma,(d-2,d-1),d})(1-\eta_{d-3,d|d-2,d-1}^{2})}\\
&\left(\Bb{R}_{\Psi}^{-1}\right)_{1,2} =  -\frac{R_{\Sigma,1,2}-R_{\Sigma,1,3}R_{\Sigma,2,3}}{(1-\Bb R_{\Sigma,(2,3),1}^{T}\Bb R_{\Sigma,(2,3),(2,3)}^{-1}\Bb R_{\Sigma,(2,3),1})(1-R_{\Sigma,2,3}^{2})(1-\eta_{1,4|2,3}^{2})}\\
&\ \ \ \ \ \ \ \ \ \ \ \ \ \ +\frac{(R_{\Sigma,2,4}-R_{\Sigma,2,3}R_{\Sigma,3,4})(\eta_{1,4}-\Bb R_{\Sigma,(2,3),1}^{T}\Bb R_{\Sigma,(2,3),(2,3)}^{-1}\Bb R_{\Sigma,(2,3),4})}{(1-\Bb R_{\Sigma,(2,3),1}^{T}\Bb R_{\Sigma,(2,3),(2,3)}^{-1}\Bb R_{\Sigma,(2,3),1})(1-\Bb R_{\Sigma,(2,3),4}^{T}\Bb R_{\Sigma,(2,3),(2,3)}^{-1}\Bb R_{\Sigma,(2,3),4})(1-R_{\Sigma,2,3}^2)(1-\eta_{1,4|2,3}^{2})}\\
&\left(\Bb{R}_{\Psi}^{-1}\right)_{1,3} =  -\frac{R_{\Sigma,1,3}-R_{\Sigma,1,2}R_{\Sigma,2,3}}{(1-\Bb R_{\Sigma,(2,3),1}^{T}\Bb R_{\Sigma,(2,3),(2,3)}^{-1}\Bb R_{\Sigma,(2,3),1})(1-R_{\Sigma,2,3}^{2})(1-\eta_{1,4|2,3}^{2})}\\
&\ \ \ \ \ \ \ \ \ \ \ \ \ \ +\frac{(R_{\Sigma,3,4}-R_{\Sigma,2,3}R_{\Sigma,2,4})(\eta_{1,4}-\Bb R_{\Sigma,(2,3),1}^{T}\Bb R_{\Sigma,(2,3),(2,3)}^{-1}\Bb R_{\Sigma,(2,3),4})}{(1-\Bb R_{\Sigma,(2,3),1}^{T}\Bb R_{\Sigma,(2,3),(2,3)}^{-1}\Bb R_{\Sigma,(2,3),1})(1-\Bb R_{\Sigma,(2,3),4}^{T}\Bb R_{\Sigma,(2,3),(2,3)}^{-1}\Bb R_{\Sigma,(2,3),4})(1-R_{\Sigma,2,3}^2)(1-\eta_{1,4|2,3}^{2})}\\
&\left(\Bb{R}_{\Psi}^{-1}\right)_{i,i+2} = \\
&\frac{(R_{\Sigma,i+2,i+3}-R_{\Sigma,i+1,i+2}R_{\Sigma,i+1,i+3})(\eta_{i,i+3}-\Bb R_{\Sigma,(i+1,i+2),i}^{T}\Bb R_{\Sigma,(i+1,i+2),(i+1,i+2)}^{-1}\Bb R_{\Sigma,(i+1,i+2),i+3})/(1-\eta_{i,i+3|i+1,i+2}^{2})}{(1-\Bb R_{\Sigma,(i+1,i+2),i}^{T}\Bb R_{\Sigma,(i+1,i+2),(i+1,i+2)}^{-1}\Bb R_{\Sigma,(i+1,i+2),i})(1-\Bb R_{\Sigma,(i+1,i+2),i+3}^{T}\Bb R_{\Sigma,(i+1,i+2),(i+1,i+2)}^{-1}\Bb R_{\Sigma,(i+1,i+2),i+3})(1-R_{\Sigma,i+1,i+2}^2)}\\
&+\frac{(R_{\Sigma,i-1,i}-R_{\Sigma,i-1,i+1}R_{\Sigma,i,i+1})(\eta_{i-1,i+2}-\Bb R_{\Sigma,(i,i+1),i-1}^{T}\Bb R_{\Sigma,(i,i+1),(i,i+1)}^{-1}\Bb R_{\Sigma,(i,i+1),i+2})/(1-\eta_{i-1,i+2|i,i+1}^{2})}{(1-\Bb R_{\Sigma,(i,i+1),i-1}^{T}\Bb R_{\Sigma,(i,i+1),(i,i+1)}^{-1}\Bb R_{\Sigma,(i,i+1),i+2})(1-\Bb R_{\Sigma,(i+1,i+2),i}^{T}\Bb R_{\Sigma,(i+1,i+2),(i+1,i+2)}^{-1}\Bb R_{\Sigma,(i+1,i+2),i})(1-R_{\Sigma,i+1,i+2}^2)}\\
&-\frac{(R_{\Sigma,i,i+2}-R_{\Sigma,i,i+1}R_{\Sigma,i+1,i+2})(1-\eta_{i-1,i+2|i,i+1}^{2}\eta_{i,i+3|i+1,i+2}^{2})}{(1-\Bb R_{\Sigma,(i+1,i+2),i}^{T}\Bb R_{\Sigma,(i+1,i+2),(i+1,i+2)}^{-1}\Bb R_{\Sigma,(i+1,i+2),i})(1-R_{\Sigma,i+1,i+2}^2)(1-\eta_{i-1,i+2|i,i+1}^{2})(1-\eta_{i,i+3|i+1,i+2}^{2})},\\ 
&\quad i=2,\ldots,d-3\\
&\left(\Bb{R}_{\Psi}^{-1}\right)_{i,i+3} = \\
&-\frac{(\eta_{i,i+3}-\Bb R_{\Sigma,(i+1,i+2),i}^{T}\Bb R_{\Sigma,(i+1,i+2),(i+1,i+2)}^{-1}\Bb R_{\Sigma,(i+1,i+2),i+3})/(1-\eta_{i,i+3|i+1,i+2}^{2})}{(1-\Bb R_{\Sigma,(i+1,i+2),i}^{T}\Bb R_{\Sigma,(i+1,i+2),(i+1,i+2)}^{-1}\Bb R_{\Sigma,(i+1,i+2),i})(1-\Bb R_{\Sigma,(i+1,i+2),i+3}^{T}\Bb R_{\Sigma,(i+1,i+2),(i+1,i+2)}^{-1}\Bb R_{\Sigma,(i+1,i+2),i+3})},\\ 
&\quad i=1,\ldots,d-3\\
&\left(\Bb{R}_{\Psi}^{-1}\right)_{d-1,d} = \\ &-\frac{R_{\Sigma,d-1,d}-R_{\Sigma,d-2,d-1}R_{\Sigma,d-2,d}}{(1-\Bb R_{\Sigma,(d-2,d-1),d}^{T}\Bb R_{\Sigma,(d-2,d-1),(d-2,d-1)}^{-1}\Bb R_{\Sigma,(d-2,d-1),d})(1-R_{\Sigma,d-2,d-1}^{2})(1-\eta_{d-3,d|d-2,d-1}^{2})}\\
&+\frac{(R_{\Sigma,d-3,d-1}-R_{\Sigma,d-2,d-1}R_{\Sigma,d-1,d})(\eta_{d-3,d}-\Bb R_{\Sigma,(d-2,d-1),d-3}^{T}\Bb R_{\Sigma,(d-2,d-1),(d-2,d-1)}^{-1}\Bb R_{\Sigma,(d-2,d-1),d})/(1-\eta_{d-3,d|d-2,d-1}^{2})}{(1-\Bb R_{\Sigma,(d-2,d-1),d-3}^{T}\Bb R_{\Sigma,(d-2,d-1),(d-2,d-1)}^{-1}\Bb R_{\Sigma,(d-2,d-1),d-3})(1-\Bb R_{\Sigma,(d-2,d-1),d}^{T}\Bb R_{\Sigma,(d-2,d-1),(d-2,d-1)}^{-1}\Bb R_{\Sigma,(d-2,d-1),d})(1-R_{\Sigma,d-2,d-1}^2)}\\
&\left(\Bb{R}_{\Psi}^{-1}\right)_{d-2,d} =\\  &-\frac{R_{\Sigma,d-2,d}-R_{\Sigma,d-2,d-1}R_{\Sigma,d-1,d}}{(1-\Bb R_{\Sigma,(d-2,d-1),d}^{T}\Bb R_{\Sigma,(d-2,d-1),(d-2,d-1)}^{-1}\Bb R_{\Sigma,(d-2,d-1),d})(1-R_{\Sigma,d-2,d-1}^{2})(1-\eta_{d-3,d|d-2,d-1}^{2})}\\
&+\frac{(R_{\Sigma,d-1,d}-R_{\Sigma,d-2,d-1}R_{\Sigma,d-3,d-1})(\eta_{d-3,d}-\Bb R_{\Sigma,(d-2,d-1),d-3}^{T}\Bb R_{\Sigma,(d-2,d-1),(d-2,d-1)}^{-1}\Bb R_{\Sigma,(d-2,d-1),d})/(1-\eta_{d-3,d|d-2,d-1}^{2})}{(1-\Bb R_{\Sigma,(d-2,d-1),d-3}^{T}\Bb R_{\Sigma,(d-2,d-1),(d-2,d-1)}^{-1}\Bb R_{\Sigma,(d-2,d-1),d-3})(1-\Bb R_{\Sigma,(d-2,d-1),d}^{T}\Bb R_{\Sigma,(d-2,d-1),(d-2,d-1)}^{-1}\Bb R_{\Sigma,(d-2,d-1),d})(1-R_{\Sigma,d-2,d-1}^2)}\\
\end{align*}
\end{tiny}
For $d=4$, we have
\begin{tiny}
\begin{align*}
&\left(\Bb{R}_{\Psi}^{-1}\right)_{2,2} =\\ 
&\frac{(1-R_{\Sigma,1,3}^2)}{(1-\Bb R_{\Sigma,(2,3),1}^{T}\Bb R_{\Sigma,(2,3),(2,3)}^{-1}\Bb R_{\Sigma,(2,3),1})(1-R_{\Sigma,2,3}^2)(1-\Bb R_{\Sigma,(3,4),2}^{T}\Bb R_{\Sigma,(3,4),(3,4)}^{-1}\Bb R_{\Sigma,(3,4),2})}\\
&-\frac{(\Bb R_{\Sigma,(2,3),4}^{T}\Bb R_{\Sigma,(2,3),(2,3)}^{-1}\Bb R_{\Sigma,(2,3),4}-R_{\Sigma,1,3}R_{\Sigma,3,4})^2}{(1-\Bb R_{\Sigma,(2,3),1}^{T}\Bb R_{\Sigma,(2,3),(2,3)}^{-1}\Bb R_{\Sigma,(2,3),1})(1-R_{\Sigma,2,3}^2)(1-\Bb R_{\Sigma,(2,3),4}^{T}\Bb R_{\Sigma,(2,3),(2,3)}^{-1}\Bb R_{\Sigma,(2,3),4})(1-R_{\Sigma,2,3}^2)(1-\eta_{1,4|2,3}^{2})}\\
&-2\frac{(\Bb R_{\Sigma,(2,3),1}^{T}\Bb R_{\Sigma,(2,3),(2,3)}^{-1}\Bb R_{\Sigma,(2,3),4}-R_{\Sigma,1,3}R_{\Sigma,3,4})(\eta_{1,4}-\Bb R_{\Sigma,(2,3),1}^{T}\Bb R_{\Sigma,(2,3),(2,3)}^{-1}\Bb R_{\Sigma,(2,3),4})}{(1-\Bb R_{\Sigma,(2,3),1}^{T}\Bb R_{\Sigma,(2,3),(2,3)}^{-1}\Bb R_{\Sigma,(2,3),1})(1-R_{\Sigma,2,3}^2)(1-\Bb R_{\Sigma,(2,3),4}^{T}\Bb R_{\Sigma,(2,3),(2,3)}^{-1}\Bb R_{\Sigma,(2,3),4})(1-R_{\Sigma,2,3}^2)(1-\eta_{1,4|2,3}^{2})}\\
&\left(\Bb{R}_{\Psi}^{-1}\right)_{2,3} =\\ 
&\frac{R_{\Sigma,1,3}(R_{\Sigma,1,2}-R_{\Sigma,1,3}R_{\Sigma,2,3})+(1-\Bb R_{\Sigma,(2,3),1}^{T}\Bb R_{\Sigma,(2,3),(2,3)}^{-1}\Bb R_{\Sigma,(2,3),1})\eta_{1,4|2,3}^{2}}{(1-\Bb R_{\Sigma,(2,3),1}^{T}\Bb R_{\Sigma,(2,3),(2,3)}^{-1}\Bb R_{\Sigma,(2,3),1})(1-R_{\Sigma,2,3}^2)(1-\eta_{1,4|2,3}^{2})}\\
&-\frac{(R_{\Sigma,1,2}-R_{\Sigma,1,3}R_{\Sigma,2,3})(R_{\Sigma,3,4}-R_{\Sigma,2,3}R_{\Sigma,2,4})+(1-R_{\Sigma,1,3}^2)(1-R_{\Sigma,2,3}^2)(R_{\Sigma,2,3}-R_{\Sigma,2,4}R_{\Sigma,3,4})}{(1-\Bb R_{\Sigma,(2,3),1}^{T}\Bb R_{\Sigma,(2,3),(2,3)}^{-1}\Bb R_{\Sigma,(2,3),1})(1-\Bb R_{\Sigma,(3,4),2}^{T}\Bb R_{\Sigma,(3,4),(3,4)}^{-1}\Bb R_{\Sigma,(3,4),2})(1-R_{\Sigma,2,3}^2)(1-R_{\Sigma,3,4}^2)(1-\eta_{1,4|2,3}^{2})}\\
&-\frac{((R_{\Sigma,1,3}-R_{\Sigma,1,2}R_{\Sigma,2,3})(R_{\Sigma,2,4}-R_{\Sigma,2,3}R_{\Sigma,3,4})+(R_{\Sigma,1,2}-R_{\Sigma,1,3}R_{\Sigma,2,3})(R_{\Sigma,3,4}-R_{\Sigma,2,3}R_{\Sigma,2,4}))(\eta_{1,4}-\Bb R_{\Sigma,(2,3),1}^{T}\Bb R_{\Sigma,(2,3),(2,3)}^{-1}\Bb R_{\Sigma,(2,3),4})}{(1-\Bb R_{\Sigma,(2,3),1}^{T}\Bb R_{\Sigma,(2,3),(2,3)}^{-1}\Bb R_{\Sigma,(2,3),1})(1-R_{\Sigma,2,3}^2)(1-\Bb R_{\Sigma,(3,4),2}^{T}\Bb R_{\Sigma,(3,4),(3,4)}^{-1}\Bb R_{\Sigma,(3,4),2})(1-R_{\Sigma,2,3}^2)(1-R_{\Sigma,3,4}^2)(1-\eta_{1,4|2,3}^{2})}	
\end{align*}
\end{tiny}
and for $d > 4$, we have
\begin{tiny}
\begin{align*}
&\left(\Bb{R}_{\Psi}^{-1}\right)_{2,2} =\\ 
&\frac{(\Bb R_{\Sigma,(2,3),1}^{T}\Bb R_{\Sigma,(2,3),(2,3)}^{-1}\Bb R_{\Sigma,(2,3),1}-R_{\Sigma,1,3}^2)(1-\eta_{2,5|3,4}^{2})+(1-\Bb R_{\Sigma,(2,3),1}^{T}\Bb R_{\Sigma,(2,3),(2,3)}^{-1}\Bb R_{\Sigma,(2,3),1})(1-\eta_{1,4|2,3}^{2})}{(1-\Bb R_{\Sigma,(2,3),1}^{T}\Bb R_{\Sigma,(2,3),(2,3)}^{-1}\Bb R_{\Sigma,(2,3),1})(1-R_{\Sigma,2,3}^2)(1-\eta_{1,4|2,3}^{2})(1-\eta_{2,5|3,4}^{2})}\\
&+\frac{(\Bb R_{\Sigma,(2,3),4}^{T}\Bb R_{\Sigma,(2,3),(2,3)}^{-1}\Bb R_{\Sigma,(2,3),4}-R_{\Sigma,3,4}^2)(1-\eta_{1,4|2,3}^{2}\eta_{2,5|3,4}^{2})}{(1-\Bb R_{\Sigma,(3,4),2}^{T}\Bb R_{\Sigma,(3,4),(3,4)}^{-1}\Bb R_{\Sigma,(3,4),2})(1-R_{\Sigma,3,4}^2)(1-\eta_{1,4|2,3}^{2})(1-\eta_{2,5|3,4}^{2})}\\
&-2\frac{(\Bb R_{\Sigma,(2,3),1}^{T}\Bb R_{\Sigma,(2,3),(2,3)}^{-1}\Bb R_{\Sigma,(2,3),4}-R_{\Sigma,1,3}R_{\Sigma,3,4})(\eta_{1,4}-\Bb R_{\Sigma,(2,3),1}^{T}\Bb R_{\Sigma,(2,3),(2,3)}^{-1}\Bb R_{\Sigma,(2,3),4})}{(1-\Bb R_{\Sigma,(2,3),1}^{T}\Bb R_{\Sigma,(2,3),(2,3)}^{-1}\Bb R_{\Sigma,(2,3),1})(1-R_{\Sigma,2,3}^2)(1-\Bb R_{\Sigma,(2,3),4}^{T}\Bb R_{\Sigma,(2,3),(2,3)}^{-1}\Bb R_{\Sigma,(2,3),4})(1-R_{\Sigma,2,3}^2)(1-\eta_{1,4|2,3}^{2})}\\	
&\left(\Bb{R}_{\Psi}^{-1}\right)_{2,3} =\\ 
&\frac{(\Bb R_{\Sigma,(2,3),1}^{T}\Bb R_{\Sigma,(2,3),(2,3)}^{-1}\Bb R_{\Sigma,(2,3),1}-R_{\Sigma,1,2}R_{\Sigma,1,3})(1-\eta_{2,5|3,4}^{2})+R_{\Sigma,2,3}(1-\Bb R_{\Sigma,(2,3),1}^{T}\Bb R_{\Sigma,(2,3),(2,3)}^{-1}\Bb R_{\Sigma,(2,3),1})(1-\eta_{1,4|2,3}^{2})}{(1-\Bb R_{\Sigma,(2,3),1}^{T}\Bb R_{\Sigma,(2,3),(2,3)}^{-1}\Bb R_{\Sigma,(2,3),1})(1-R_{\Sigma,2,3}^2)(1-\eta_{1,4|2,3}^{2})(1-\eta_{2,5|3,4}^{2})}\\
&+\frac{(R_{\Sigma,2,4}-R_{\Sigma,2,3}R_{\Sigma,3,4})(R_{\Sigma,3,4}-R_{\Sigma,2,3}R_{\Sigma,2,4})(1-\eta_{1,4|2,3}^{2}\eta_{2,5|3,4}^{2})}{(1-\Bb R_{\Sigma,(3,4),2}^{T}\Bb R_{\Sigma,(3,4),(3,4)}^{-1}\Bb R_{\Sigma,(3,4),2})(1-R_{\Sigma,2,3}^2)(1-R_{\Sigma,3,4}^2)(1-\eta_{1,4|2,3}^{2})(1-\eta_{2,5|3,4}^{2})}\\
&+\frac{(R_{\Sigma,3,5}-R_{\Sigma,3,4}R_{\Sigma,4,5})(\eta_{2,5}-\Bb R_{\Sigma,(3,4),2}^{T}\Bb R_{\Sigma,(3,4),(3,4)}^{-1}\Bb R_{\Sigma,(3,4),5})}{(1-\Bb R_{\Sigma,(3,4),2}^{T}\Bb R_{\Sigma,(3,4),(3,4)}^{-1}\Bb R_{\Sigma,(3,4),2})(1-\Bb R_{\Sigma,(3,4),5}^{T}\Bb R_{\Sigma,(3,4),(3,4)}^{-1}\Bb R_{\Sigma,(3,4),5})(1-R_{\Sigma,3,4}^2)(1-\eta_{2,5|3,4}^{2})}\\
&-\frac{((R_{\Sigma,1,3}-R_{\Sigma,1,2}R_{\Sigma,2,3})(R_{\Sigma,2,4}-R_{\Sigma,2,3}R_{\Sigma,3,4})+(R_{\Sigma,1,2}-R_{\Sigma,1,3}R_{\Sigma,2,3})(R_{\Sigma,3,4}-R_{\Sigma,2,3}R_{\Sigma,2,4}))(\eta_{1,4}-\Bb R_{\Sigma,(2,3),1}^{T}\Bb R_{\Sigma,(2,3),(2,3)}^{-1}\Bb R_{\Sigma,(2,3),4})}{(1-\Bb R_{\Sigma,(2,3),1}^{T}\Bb R_{\Sigma,(2,3),(2,3)}^{-1}\Bb R_{\Sigma,(2,3),1})(1-R_{\Sigma,2,3}^2)(1-\Bb R_{\Sigma,(3,4),2}^{T}\Bb R_{\Sigma,(3,4),(3,4)}^{-1}\Bb R_{\Sigma,(3,4),2})(1-R_{\Sigma,2,3}^2)(1-R_{\Sigma,3,4}^2)(1-\eta_{1,4|2,3}^{2})}.
\end{align*}
\end{tiny}

This leads to 
\begin{align*}
A_{11} = &\frac{1-\rho_{1,4|2,3}\eta_{1,4|2,3}}{1-\eta_{1,4|2,3}^{2}}\\
A_{dd} = &\frac{1-\rho_{d-3,d|d-2,d-1}\eta_{d-3,d|d-2,d-1}}{1-\eta_{d-3,d|d-2,d-1}^{2}}\\
A_{22} = &\begin{cases}
1, & d = 4\\
\frac{1-\rho_{2,5|3,4}\eta_{2,5|3,4}}{1-\eta_{2,5|3,4}^{2}}, & d > 4
\end{cases}\\
A_{d-1,d-1} = & \frac{1-\rho_{d-4,d-1|d-3,d-2}\eta_{d-4,d-1|d-3,d-2}}{1-\eta_{d-4,d-1|d-3,d-2}^{2}}, \quad d > 4\\
A_{33} = &\begin{cases}
1, & d = 4\\
\frac{1-\rho_{1,4|2,3}\eta_{1,4|2,3}}{1-\eta_{1,3|2,3}^{2}}, & d = 5\\
\frac{1-\rho_{3,6|4,5}\eta_{3,6|4,5}}{1-\eta_{3,6|4,5}^{2}}, & d > 5\\
\end{cases}\\
A_{d-2,d-2} = & \frac{1-\rho_{d-5,d-2|d-4,d-3}\eta_{d-5,d-2|d-4,d-3}}{1-\eta_{d-5,d-2|d-4,d-3}^{2}}, \quad d > 5\\
A_{i,i} = & \frac{1-\rho_{i-3,i|i-2,i-1}\eta_{i-3,i|i-2,i-1}}{1-\eta_{i-3,i|i-2,i-1}^{2}}+\frac{1-\rho_{i,i+3|i+1,i+2}\eta_{i,i+3|i+1,i+2}}{1-\eta_{i,i+3|i+1,i+2}^{2}}, \quad i = 4,\ldots,d-3,
\end{align*}
which gives 
\begin{align*}
KL(p||q) = & \sum_{i=1}^{d-3}\frac{1-\rho_{i,i+3|i+1,i+2}\eta_{i,i+3|i+1,i+2}}{1-\eta_{i,i+3|i+1,i+2}^{2}}-\frac{d-6}{2}-\frac{1}{2}\log|\Bb{R}_{\Sigma}|+\frac{1}{2}\sum_{i=1}^{d-1}\log\left(1-R_{\Sigma,i,i+1}^{2}\right)\\
&+\frac{1}{2}\sum_{i=1}^{d-2}\log\left(1-\rho_{i,i+2|i+1}^{2}\right)+\frac{1}{2}\sum_{i=1}^{d-3}\log\left(1-\eta_{i,i+3|i+1,i+2}^{2}\right)
\end{align*}
Hence, the derivative of the KL divergence with respect to the parameters of the third
tree are given by
\[
\frac{\partial KL(p||q)}{\partial \eta_{i,i+3|i+1,i+2}} = \frac{1+\eta_{i,i+3|i+1,i+2}^{2}}{(1-\eta_{i,i+3|i+1,i+2}^2)^2}(\eta_{i,i+3|i+1,i+2}-\rho_{i,i+3|i+1,i+2}), \quad i=1,\ldots,d-1
\]
which means that the minimum is obtained for 
$\eta_{i,i+3|i+1,i+2}=\rho_{i,i+3|i+1,i+2}$, $i=1,\ldots,d-2$, which again means that
\begin{scriptsize}
\begin{align*}
&R_{\Psi,i,i+3} =\\ 
&\Bb R_{\Sigma,(i+1,i+2),i}^{T}\Bb R_{\Sigma,(i+1,i+2),(i+1,i+2)}^{-1}\Bb R_{\Sigma,(i+1,i+2),i+3}\\
&+\rho_{i,i+3|i+1,i+2}\sqrt{(1-\Bb R_{\Sigma,(i+1,i+2),i}^{T}\Bb R_{\Sigma,(i+1,i+2),(i+1,i+2)}^{-1}\Bb R_{\Sigma,(i+1,i+2),i})(1-\Bb R_{\Sigma(i+1,i+2),i+3}^{T}\Bb R_{\Sigma,(i+1,i+2),(i+1,i+2)}^{-1}\Bb R_{\Sigma,(i+1,i+2),i+3})}\\
& = R_{\Sigma,i,i+3},
\end{align*}
\end{scriptsize}
for $i=1,\ldots,d-3$.

The \textbf{$(t+1)$st step} (if $d \geq t+1$, otherwise the $t$th step is the last) is to find the 
parameters $\eta_{i,i+t-1|i+1\ldots i+t-2}$, $i=1,\ldots,d-t-1$ of the $t$th tree of the 
D-vine that maximise the forward KL when all the remaining $d-t-1$ trees are set to 
independence, which corresponds to setting $\eta_{i,i+j|i+1,\ldots,i+j-1}=0$, for 
$j=t,\ldots,d-t+1$, $i=1,\ldots,d-i$. Assume now that the stepwise procedure parameter
values obtained for the first $t$ trees correspond to the values from the true posterior,
so that $R_{\Psi,i,i+j}=R_{\Sigma,i,i+j}$, $j=1,\ldots,t-1$, $i=1,\ldots,d-i$. Then, we get

\begin{align*}
A_{i,i} = & A_{d-t+i,d-t+i} = \frac{1-\rho_{1,t+1|2,\ldots,t}\eta_{1,t+1|2,\ldots,t}}{1-\eta_{1,t+1|2,\ldots,t}^{2}}, \quad i=1,\ldots,t\\
A_{i,i} = & \frac{1-\rho_{i-t,i|i-t+1,\ldots,i-1}\eta_{i-t,i|i-t+1,\ldots,i-1}}{1-\eta_{i-t,i|i-t+1,\ldots,i-1}^{2}}+\frac{1-\rho_{i,i+t|i+1,\ldots,i+t-1}\eta_{i,i+t|i+1,\ldots,i+t-1}}{1-\eta_{i,i+t|i+1,\ldots,i+t-1}^{2}},\\ 
& i = t+1,\ldots,d-t.
\end{align*}
Hence, the forward KL divergence is given by
\begin{align*}
KL(p||q) = & \sum_{i=1}^{d-t+1}\frac{1-\rho_{i,i+t-1|i+1\ldots i+t-2}\eta_{i,i+t-1|i+1\ldots i+t-2}}{1-\eta_{i,i+t-1|i+1\ldots i+t-2}^{2}}-\frac{d-2(t+1)}{2}-\frac{1}{2}\log|\Bb{R}_{\Sigma}|\\
&+\frac{1}{2}\sum_{i=1}^{d-1}\log\left(1-R_{\Sigma,i,i+1}^{2}\right)+\ldots+\frac{1}{2}\sum_{i=1}^{d-t}\log\left(1-\rho_{i,i+t-2|i+1\ldots i+t-3}^{2}\right)\\
&+\frac{1}{2}\sum_{i=1}^{d-t+1}\log\left(1-\eta_{i,i+t-1|i+1\ldots i+t-2}^{2}\right),
\end{align*}
and the derivative of the KL divergence with respect to the parameters of the $t$th 
tree are given by
\[
\frac{\partial KL(p||q)}{\partial\eta_{i,i+t-1|i+1\ldots i+t-2}} = \frac{1+\eta_{i,i+t-1|i+1\ldots i+t-2}^{2}}{(1-\eta_{i,i+t-1|i+1\ldots i+t-2}^2)^2}(\eta_{i,i+t-1|i+1\ldots i+t-2}-\rho_{i,i+t-1|i+1\ldots i+t-2}), 
\]
for $i=1,\ldots,d-t+1$, which means that the minimum is obtained for 
$\eta_{i,i+t-1|i+1\ldots i+t-2}=\rho_{i,i+t-1|i+1\ldots i+t-2}$, which again means that
$R_{\Psi,i,i+t-1} = R_{\Sigma,i,i+t-1}$, $i=1,\ldots,d-t+1$.
\end{proof}

\begin{remark}
If the D-vine is truncated after tree $t$, so that the stepwise procedure stops after the 
$t+1$th step, the correlations $R_{\Psi,i,i+j}$, for $j > t$, will be approximated by
$R_{\Psi,i,i+j} = \Bb R_{\Psi,i+1\ldots i+j-1,i}^{T}\Bb R_{\Sigma,i+1\ldots i+j-1,i+1\ldots i+j-1}^{-1}\Bb R_{\Sigma,i+1\ldots i+j-1,i+j}$, which in practice will be close to
$R_{\Sigma,i,i+j}$ if $t$ is large enough.
\end{remark}

\subsection{Backward KL}\label{app:backwardKL_proof_and_proposition}

\begin{proof}[Proof of Theorem \ref{thm:backwardKL}]
In the \textbf{first step}, the marginal parameters $\B{\nu}$ and $\Bb{D}_{\Psi}$ are 
estimated by mean-field, minimizing the backward KL, assuming independence between 
$Z_{1},\ldots,Z_{d}$. As shown by \cite{margossian2024variational}, one then obtains 
$\B{\nu}=\B{\mu}$  and\\ 
$D_{\Psi,i,i}=1/\sqrt{(\B{\Sigma}^{-1})_{i,i}} = D_{\Sigma,i,i}/\sqrt{(\Bb R_{\Sigma}^{-1})_{i,i}} , i=1,\ldots,d$. 
The backward KL then reduces to
\begin{align*}
KL(q||p) = &\frac{1}{2}\left(tr\left(\Bb R_{\Sigma}^{-1}\Bb{D}_{\Sigma}^{-1}\Bb{D}_{\Psi}\Bb R_{\Psi}\Bb{D}_{\Psi}\Bb{D}_{\Sigma}^{-1}\right)-d+\sum_{i=1}^{d}\log\left((\Bb R_{\Sigma}^{-1})_{i,i}\right)-\log|\Bb R_{\Psi}|+\log|\Bb R_{\Sigma}|\right)
\end{align*}

The \textbf{second step} is to find the parameters 
$\eta_{i,i+1}=R_{\Psi,i,i+1}$, $i=1,\ldots,d-1$ of the first tree of the D-vine that maximize 
the backward KL when all the remaining $d-2$ trees are set to independence, so that 
the remaining elements of $\Bb{R}_{\Psi}$ are the constrained to be 
$R_{\Psi,i,i+j}=\prod_{k=1}^{j}R_{\Psi,i+k-1,i+k}$, for $j=1,\ldots,d-2$, $i=1,\ldots,d-i$,
and 
\[
|\Bb{R}_{\Psi}| = \prod_{i=1}^{d-1}\left(1-\eta_{i,i+1}^{2}\right).
\]

Let $\Bb B = \Bb R_{\Sigma}^{-1}\Bb{D}_{\Sigma}^{-1}\Bb{D}_{\Psi}\Bb R_{\Psi}\Bb{D}_{\Psi}\Bb{D}_{\Sigma}^{-1}$.
As $\Bb{D}_{\Sigma}^{-1}\Bb{D}_{\Psi}\Bb R_{\Psi}=\Bb R_{\Psi}\Bb{D}_{\Psi}\Bb{D}_{\Sigma}^{-1}$
is a diagonal matrix with diagonal entries $1//\sqrt{(\Bb R_{\Sigma}^{-1})_{i,i}}$, we get
\[
\Bb B_{ii} = 1+\sum_{j\neq i}\eta_{i,j}\frac{(\Bb R_{\Sigma}^{-1})_{i,j}}{\sqrt{(\Bb R_{\Sigma}^{-1})_{i,i}(\Bb R_{\Sigma}^{-1})_{j,j}}} = 1-\sum_{j\neq i}\eta_{i,j}\rho_{i,j|\{1,\ldots,d\}\setminus\{i,j\}}
\]
Plugging in the constraints on $\Bb R_{\Psi}$, this gives
\[
tr(\Bb B) = d-2\sum_{j=1}^{d-1}\sum_{i=1}^{d-j}\rho_{i,j|\{1,\ldots,d\}\setminus\{i,j\}}\prod_{k=1}^{j}\eta_{i+k-1,i+k},
\]
so that
\begin{align*}
KL(q||p) = &-\sum_{j=1}^{d-1}\sum_{i=1}^{d-j}\rho_{i,i+j|\{1,\ldots,d\}\setminus\{i,i+j\}}\prod_{k=1}^{j}\eta_{i+k-1,i+k}-\frac{1}{2}\sum_{i=1}^{d-1}\log\left(1-\eta_{i,i+1}^{2}\right)\\
&+\frac{1}{2}\log|\Bb R_{\Sigma}|+\frac{1}{2}\sum_{i=1}^{d}\log\left((\Bb R_{\Sigma}^{-1})_{i,i}\right)
\end{align*}
Hence, the derivatives of the KL divergence with respect to the parameters of the first 
tree are given by
\begin{align*}
\frac{\partial KL(p||q)}{\partial \eta_{i,i+1}} = &-\rho_{i,i+1|\{1,\ldots,d\}\setminus\{i,i+1\}}-\sum_{l=1}^{i-1}\prod_{k=1}^{l}\eta_{i-k,i-k+1}\rho_{i-l,i+1|\{1,\ldots,d\}\setminus\{i-l,i+1\}}\\
&-\sum_{m=2}^{d-i}\prod_{k=1}^{l}\eta_{i+k,i+k+1}\rho_{i+k,i+k+1|\{1,\ldots,d\}\setminus\{i+k,i+k+1\}}\\
&-\sum_{l=1}^{i-1}\prod_{k=1}^{l}\eta_{i-k,i-k+1}\sum_{m=2}^{d-i}\prod_{n=1}^{m}\eta_{i+n,i+n+1}\rho_{i-l+k,i-l+k+1|\{1,\ldots,d\}\setminus\{i-l+k,i-l+k+1\}}\\
&+\frac{\eta_{i,i+1}}{(1-\eta_{i,i+1}^2)}, \quad i=1,\ldots,d-1.
\end{align*}
Setting the $d-1$ above expressions to $0$, we obtain the parameters of the first tree,
buts these do not have an analytic expression in this case, but the solutions are generally
not the true correlations, unless all correlations, and thus partial correlations, in the true 
posterior are $0$, so that the derivatives of the KL divergence reduce to
\begin{align*}
\frac{\partial KL(p||q)}{\partial \eta_{i,i+1}} = \frac{\eta_{i,i+1}}{(1-\eta_{i,i+1}^2)}, \quad i=1,\ldots,d-1,
\end{align*}
which, when setting to $0$ and solving for the first tree parameters, results in
$\eta_{i,i+1}=0$, $i=1,\ldots,d-1$.
\end{proof}

\begin{proposition}\label{prop:backward_KL_known_std}
Even if the standard deviations $\Bb{D}_{\Sigma}$ are known, the true correlation matrix
$\Bb{R}_{\Sigma}$ will \textbf{not} be recovered with the stepwise procedure using
the backward KL divergence, unless a Gaussian D-vine with only one tree is the true model,
i.e. $\rho_{i,i+j|i+1,\ldots,i+j-1}=0$ for $j=1,\ldots,d-2$, $i=1,\ldots,d-j$.
\end{proposition}

\begin{proof}[Proof of Proposition \ref{prop:backward_KL_known_std}]
In the \textbf{first step}, only the marginal mean vector $\B{\nu}$ is estimated, as in the previous 
cases, the true mean vector $\B{\mu}$ is the solution. The resulting KL divergence is then
\[
KL(p||q) = \frac{1}{2}\left(tr\left(\Bb{R}_{\Sigma}^{-1}\Bb{R}_{\Psi}\right)-d+\log|\Bb{R}_{\Sigma}|-\log|\Bb{R}_{\Psi}|\right).
\]

Let $\Bb C=\Bb{R}_{\Sigma}^{-1}\Bb{R}_{\Psi}$. Now, moving on the the \textbf{second step}, 
we get
\[
\Bb C_{ii} = (\Bb R_{\Sigma}^{-1})_{i,i}+\sum_{j\neq i}\eta_{i,j}(\Bb R_{\Sigma}^{-1})_{i,j}.
\]
Plugging in the constraints on $\Bb R_{\Psi}$, this gives
\[
tr(\Bb C) = \sum_{i=1}^{d}(\Bb R_{\Sigma}^{-1})_{i,i}+2\sum_{j=1}^{d-1}\sum_{i=1}^{d-j}(\Bb R_{\Sigma}^{-1})_{i,j}\prod_{k=1}^{j}\eta_{i+k-1,i+k},
\]
so that
\begin{align*}
KL(q||p) = &\frac{1}{2}\sum_{i=1}^{d}(\Bb R_{\Sigma}^{-1})_{i,i}+\sum_{j=1}^{d-1}\sum_{i=1}^{d-j}(\Bb R_{\Sigma}^{-1})_{i,j}\prod_{k=1}^{j}\eta_{i+k-1,i+k}-\frac{d}{2}-\frac{1}{2}\sum_{i=1}^{d-1}\log\left(1-\eta_{i,i+1}^{2}\right)\\
	&+\frac{1}{2}\log|\Bb R_{\Sigma}|
\end{align*}
Hence, the derivatives of the KL divergence with respect to the parameters of the first 
tree are given by
\begin{align*}
\frac{\partial KL(p||q)}{\partial \eta_{i,i+1}} = &(\Bb R_{\Sigma}^{-1})_{i,i+1}+\sum_{l=1}^{i-1}\prod_{k=1}^{l}\eta_{i-k,i-k+1}(\Bb R_{\Sigma}^{-1})_{i-l,i+1}+\sum_{m=2}^{d-i}\prod_{k=1}^{l}\eta_{i+k,i+k+1}(\Bb R_{\Sigma}^{-1})_{i+k,i+k+1}\\
&+\sum_{l=1}^{i-1}\prod_{k=1}^{l}\eta_{i-k,i-k+1}\sum_{m=2}^{d-i}\prod_{n=1}^{m}\eta_{i+n,i+n+1}(\Bb R_{\Sigma}^{-1})_{i-l+k,i-l+k+1}\\
&+\frac{\eta_{i,i+1}}{(1-\eta_{i,i+1}^2)}, \quad i=1,\ldots,d-1.
\end{align*}
Again, there is no analytical solution, and it will in general be different from the true correlations, unless
$\rho_{i,i+j|i+1,\ldots,i+j-1}=0$ for $j=1,\ldots,d-2$, $i=1,\ldots,d-j$. Then 
$R_{\Sigma}^{-1})_{i,i+1}=-\frac{R{\Sigma,i,i+1}}{(1-R_{\Sigma,i,i+1}^2)}$
$R_{\Sigma}^{-1})_{i,i+j}=0$ for $j=1,\ldots,d-2$, $i=1,\ldots,d-j$, and the derivatives of the KL divergence
reduce to
\begin{align*}
\frac{\partial KL(p||q)}{\partial \eta_{i,i+1}} = -\frac{R_{\Sigma,i,i+1}}{(1-R_{\Sigma,i,i+1}^2)}+\frac{\eta_{i,i+1}}{(1-\eta_{i,i+1}^2)}, \quad i=1,\ldots,d-1,
\end{align*}
which, when setting to $0$ and solving for the first tree parameters, results in
$\eta_{i,i+1}=R_{\Sigma,i,i+1}$, $i=1,\ldots,d-1$.
\end{proof}

\section{Computational Complexity}\label{app:computational_complexity}
A $\tau$-truncated D-vine has $d(d-1)/2 - (d - \tau)(d - \tau - 1)/2$ parameters in comparison to $d(d-1)/2$ parameters of a full D-vine and Gaussian distribution with full-rank covariance matrix (ignoring the mean vector 0 for better comparison).

\section{Implementation of Stepwise VI with Vine Copulas}\label{app:implementation}
We implement our approach in \verb|pyro| \citep{bingham2018pyro} and \verb|PyTorch| \citep{paszke2017automatic}. More specifically, we implement D-vines and Gaussian pair copulas as \verb|TorchDistribution| class in \verb|pyro| with methods \verb|sample| and \verb|log_prob|. The \verb|log_prob| method is based on the likelihood function of a D-vine implemented in \verb|rvinecopulib| \citep{nagler2025rvinecopulib} and the \verb|sample| method is based on Algorithm 6.6. by \citet{czado2019analyzing}. We do not use the python package \verb|torchvinecopulib| \citep{cheng2025vine} which implements vine copulas in \verb|torch| enabling automatic differentiation. This is because \verb|torchvinecopulib| lacks an implementation of parametric pair copula families and we wanted to focus on simpler parameteric variational models.
As the reparametrized VR-IWAE bound yields the same SGD procedure as the reparametrized VR bound \citep{daudel2023alpha}, we use \verb|pyro|'s implementation of the VR bound and its reparametrized lower bound gradient estimators. 

Code will be published upon acceptance of the paper.

\section{Simulation Study on $\alpha$ in the VR-IWAE}\label{app:simulation_study_alpha}
We conduct an intensive simulation study to analyze the effect of $\alpha \in (0,1)$ in the VR-IWAE on the approximation capacity of the D-vine and the MF as a variational distributions. The MF and the D-vine are analyzed separately to find out whether the effect of $\alpha$ on the approximation capacity differs between D-vine and MF. 
% The goal is to find out whether tuning of $\alpha \in (0,1)$ is jointly or for the D-vine and MF each is necessary, or whether a joint $\alpha$ or two or separate $\alpha$ values can be pre-defined independent of the current example.

We create three different set ups and sample 10 different examples, i.e. data sets, per set up. On each we run VI with D-vine and MF as the variational distribution for each $\alpha \in \{0.01, 0.05, 0.1, 0.2, 0.3, 0.4, 0.5, 0.6, 0.7, 0.8, 0.9\} =: A$.

\subsection{Gaussian Set Up}
Following \citet{shen2025wild} we simulate a data set $(\mathbf{x}_i^T, y_i)_{i \in [n]}$ as follows: We first sample $n=300$ observations from the explanatory variables $(X_1, ..., X_4, Y)^T$:
\begin{align}
    (X_1, ..., X_4)^T &\sim \mathcal{N}(\bm{0}, C_k) \; ,
\end{align}
and then set:
\begin{align}
    Y := \beta_1 X_1 + \beta_2 X_2 + \beta_3 X_3 + \beta_4 X_4 \; .
\end{align}
We define the true value of the latent vector to:
\begin{align}
    (\beta_1, ..., \beta_4)^T := (10, -10, 5, 3) \; .
\end{align}
and set different values for the correlation matrix $C_k$ for each $k \in [10]$:
\begin{align}\label{eq:C1C2_correlation_simulation_study_definition}
    C_1 &:= \begin{pmatrix}
         1.0000 & -0.8590 & -0.8749 & -0.6122 \\
        -0.8590 &  1.0000 &  0.9154 &  0.8862 \\
        -0.8749 &  0.9154 &  1.0000 &  0.6948 \\
        -0.6122 &  0.8862 &  0.6948 &  1.0000
    \end{pmatrix} \; , \\
    C_2 &:= \begin{pmatrix}
        1.0000 & -0.6000 &  -0.0371 & 0.5559 \\
        -0.6000 & 1.0000 & 0.7000 & 0.9066 \\
        -0.0371 & 0.7000 & 1.0000 & -0.5000 \\
        0.5559 & 0.9066 & -0.5000 & 1.0000
    \end{pmatrix} \; ,
\end{align}
and for remaining $k \in \{3, ..., 10\}$ we sample 4 precision matrices each from:
\begin{align}
    P_k &\sim \mathcal{W}(\nu=5, C_1^{-1}) \; , \\
    P_k &\sim \mathcal{W}(\nu=10, C_2^{-1}) 
\end{align}
where $\mathcal{W}$ denotes the Wishart distribution, and obtain the correlation matrices by $C_k := Corr(P_k^{-1})$. As proposed by \citet{shen2025wild} we use the likelihood $y_i | \bm{\beta}, \mathbf{x}_i \sim \mathcal{N}(\mathbf{x}_i^T \bm{\beta}, 1)$\footnote{This likelihood as proposed by \citet{shen2025wild} is slightly mis-specified due to its standard deviation of 1. As the goal of the simulation study is to study the effect of $\alpha$ on the performance of variational approximation obtained with VR-IWAE, we take over the examples as is.} and chose $\mathcal{N}(\bm{0}, 10^4 \cdot I_4)$ as prior. We compare the variational distribution to samples from the true posterior obtained from NUTS \citep{hoffman2014no}.

\paragraph{Stepwise D-vine}
For comparing different $\alpha$ values, we compute the forward KL-divergence, i.e. $KL\big(p(\bm{\beta}|\mathbf{x}, \bm{y}) || q(\bm{\beta}; \bm{\lambda}^*, \bm{\eta}) \big)$, which is available in closed form for two multivariate Gaussians. Since the KL-divergence does not have a unified scale, values between different example cannot be compared. For this reason we assess:
\begin{align}\label{def:relKL}
    \Delta KL_{rel}^{(\alpha, k)} := \frac{KL^{(\alpha, k)}\big(p(\bm{\beta}|\mathbf{x}, \bm{y}) || q(\bm{\beta}; \bm{\lambda}^*, \bm{\eta}) \big) - KL^{(\alpha^{min}, k)}\big(p(\bm{\beta}|\mathbf{x}, \bm{y}) || q(\bm{\beta}; \bm{\lambda}^*, \bm{\eta}) \big)}{ | KL^{(\alpha^{min}, k)}\big(p(\bm{\beta}|\mathbf{x}, \bm{y}) || q(\bm{\beta}; \bm{\lambda}^*, \bm{\eta}) \big) |}
\end{align}
where:
\begin{align}
    \alpha^{min} \in \arg \min_{\alpha \in A} KL^{(\alpha, k)}\big(p(\bm{\beta}|\mathbf{x}, \bm{y}) || q(\bm{\beta}; \bm{\lambda}^*, \bm{\eta}) \big)
\end{align}
for the specific example $k$.
To assess the impact of $\alpha$ solely on the D-vine without any effect of the marginals/MF, we set $\bm{\lambda}^*$ to the correct MF parameters obtained from NUTS.
The results in Table \ref{tab:simulation_alpha_Gaussian_dvine_kl} indicate that consistently $\alpha \leq 0.2$ gives a smaller $KL\big(p(\bm{\beta}|\mathbf{x}, \bm{y}) || q(\bm{\beta}; \bm{\lambda}^*, \bm{\eta}) \big)$, where in 4 out of 10 examples $\alpha=0.05$ yields the best result.

\begin{table}[ht]
    \centering
    \caption{$\Delta KL_{rel}^{(\alpha, k)}$ defined in Equation \eqref{def:relKL} between $q$, a stepwise D-vine, and the true posterior obtained with NUTS, by $\alpha$ in the VR-IWAE (rows) and simulated observed data coming from a multivariate normal (columns). The minimum $\Delta KL_{rel}^{(\alpha, k)}$ value per Example (column) is marked in bold indicating the best $\alpha$ value per example.}
    \label{tab:simulation_alpha_Gaussian_dvine_kl}
    \scriptsize{
    \begin{tabular}{lrrrrrrrrrr}
        \toprule
        Example & 1 & 2 & 3 & 4 & 5 & 6 & 7 & 8 & 9 & 10 \\
    \midrule
    $\alpha=0.01$ & 0.9863 & 4.1335 & 2.2100 & 0.3619 & 0.3108 & 0.3112 & 7.6358 & \textbf{0.0000} & 9.1015 & \textbf{0.0000} \\
    $\alpha=0.05$ & \textbf{0.0000} & 3.3678 & 6.0011 & 0.0615 & \textbf{0.0000} & \textbf{0.0000} & \textbf{0.0000} & 0.7748 & 3.6409 & 0.8254 \\
    $\alpha=0.1$ & 2.4009 & \textbf{0.0000} & \textbf{0.0000} & 0.0050 & 0.1315 & 0.0606 & 5.1381 & 1.6675 & \textbf{0.0000} & 0.1515 \\
    $\alpha=0.2$ & 0.7186 & 1.2066 & 8.5026 & \textbf{0.0000} & 0.1960 & 0.0802 & 8.3610 & 0.9560 & 0.5140 & 1.0075 \\
    $\alpha=0.3$ & 3.4847 & 6.0026 & 5.8918 & 0.5133 & 0.5074 & 0.5757 & 25.4941 & 1.3103 & 6.0731 & 2.3233 \\
    $\alpha=0.4$ & 5.8391 & 12.2899 & 18.6095 & 0.4655 & 1.0352 & 0.6295 & 45.1336 & 2.1350 & 10.9958 & 3.9044 \\
    $\alpha=0.5$ & 10.0232 & 25.1026 & 26.0312 & 0.6685 & 1.0099 & 0.0275 & 59.9073 & 4.0688 & 23.9811 & 8.3800 \\
    $\alpha=0.6$ & 16.0396 & 38.3687 & 59.5540 & 0.9479 & 1.6862 & 0.9086 & 99.4024 & 11.2921 & 26.7213 & 17.4015 \\
    $\alpha=0.7$ & 24.8383 & 58.0596 & 95.2756 & 1.7309 & 3.1522 & 1.1258 & 118.9287 & 18.1933 & 52.6679 & 31.7101 \\
    $\alpha=0.8$ & 38.5228 & 80.9844 & 182.3937 & 3.5975 & 7.2812 & 1.9569 & 160.2008 & 28.5624 & 92.5257 & 45.7793 \\
    $\alpha=0.9$ & 66.5913 & 111.7053 & 252.3727 & 8.3199 & 13.2690 & 5.3173 & 181.2606 & 37.6151 & 162.3288 & 70.9501 \\
    \bottomrule
    \end{tabular}
    }%
\end{table}

\paragraph{Gaussian MF}
We compare the Gaussian MF as variational distribution to samples from the true posterior obtained from NUTS. From \citet{margossian2024variational} we know that the MF recovers the true posterior mean. Instead of computing $KL\big(p(\bm{\beta}|\mathbf{x}, \bm{y}) || q(\bm{\beta}; \bm{\lambda}) \big)$, we therefore compute the relative root means squared error (RMSE) of the standard deviations averaged over $d$:
\begin{align}\label{eq:mean_rel_RMSE_stds}
        \frac{1}{d} \sum_{j=1}^d \frac{\sqrt{(\sigma_j - \hat{\sigma}_j)^2}}{\sigma_j} \; ,
\end{align}
per $\alpha$ and simulated observed data example. Here $\sigma_j, \; j \in [d]$ are the standard deviations of the true posterior $p(\bm{\beta}|\mathbf{x}, \bm{y})$ and $\hat{\sigma}_j, \; j \in [d]$ are the standard deviations of the approximate posterior $q(\bm{\beta}; \bm{\lambda}^*, \bm{\eta})$.

The results of Table \ref{tab:simulation_alpha_Gaussian_MF_rel_RMSE_stds} indicate, that consistently a value of $\alpha \leq 0.2$ leads to a lowest relative RMSE of the estimated standard deviations.

\begin{table}[ht]
    \centering
    \caption{Mean relative RMSE of standard deviation \eqref{eq:mean_rel_RMSE_stds} of $q$, a Gaussian MF, and the true posterior obtained with NUTS, by $\alpha$ in the VR-IWAE (rows) and simulated observed data coming from a multivariate normal (columns). The minimum value per Example (column) is marked in bold indicating the best $\alpha$ value per example.}
    \label{tab:simulation_alpha_Gaussian_MF_rel_RMSE_stds}
    \scriptsize{
    \begin{tabular}{lrrrrrrrrrr}
        \toprule
        Example & 1 & 2 & 3 & 4 & 5 & 6 & 7 & 8 & 9 & 10 \\
        \midrule
        $\alpha=0.01$ & 0.2106 & 0.4120 & \textbf{0.2483} & 0.6312 & 0.5207 & 0.6121 & 0.5505 & 0.5237 & \textbf{0.4068} & \textbf{0.4344} \\
        $\alpha=0.05$ & \textbf{0.1935} & 0.4209 & 0.2723 & 0.6335 & 0.5287 & \textbf{0.6102} & \textbf{0.5416} & 0.5268 & 0.4187 & 0.4438 \\
        $\alpha=0.1$ & 0.2135 & \textbf{0.3794} & 0.2671 & 0.6220 & \textbf{0.5170} & 0.6317 & 0.5547 & \textbf{0.5161} & 0.4259 & 0.4399 \\
        $\alpha=0.2$ & 0.2265 & 0.4503 & 0.3100 & \textbf{0.6206} & 0.5428 & 0.6242 & 0.5914 & 0.5259 & 0.4512 & 0.4731 \\
        $\alpha=0.3$ & 0.2717 & 0.4572 & 0.3238 & 0.6337 & 0.5579 & 0.6354 & 0.6096 & 0.5569 & 0.4855 & 0.4461 \\
        $\alpha=0.4$ & 0.3219 & 0.5061 & 0.3784 & 0.6437 & 0.5691 & 0.6644 & 0.6585 & 0.5872 & 0.5195 & 0.5007 \\
        $\alpha=0.5$ & 0.3949 & 0.5515 & 0.4249 & 0.6780 & 0.5976 & 0.6654 & 0.6746 & 0.6351 & 0.5649 & 0.5463 \\
        $\alpha=0.6$ & 0.4567 & 0.6090 & 0.4838 & 0.7142 & 0.6587 & 0.7021 & 0.7019 & 0.6738 & 0.6121 & 0.6227 \\
        $\alpha=0.7$ & 0.5297 & 0.6427 & 0.5406 & 0.7799 & 0.7440 & 0.7575 & 0.7425 & 0.7145 & 0.6633 & 0.6761 \\
        $\alpha=0.8$ & 0.6115 & 0.6788 & 0.6196 & 0.8556 & 0.8043 & 0.8704 & 0.7522 & 0.7430 & 0.7118 & 0.7218 \\
        $\alpha=0.9$ & 0.6657 & 0.7066 & 0.6751 & 0.8913 & 0.8334 & 0.9117 & 0.7712 & 0.7783 & 0.7447 & 0.7642 \\
        \bottomrule
    \end{tabular}
    }%
\end{table}

% updated with psaver results til here

\subsection{Student-t Set Up}
We create a set up in similar fashion of the (Gaussian) needle example in \citet{shen2025wild} with a multivariate Student-t distribution. We sample $n=300$ observations from the explanatory variables $(X_1, ..., X_4, Y)^T$:
\begin{align}
    (X_1, ..., X_4)^T &\sim t_{d=4}(\nu=4, \bm{\mu}=\bm{0}, \Sigma=C_k) \; ,
\end{align}
and then set:
\begin{align}
    Y := \beta_1 X_1 + \beta_2 X_2 + \beta_3 X_3 + \beta_4 X_4 \; .
\end{align}
We define the true value of the latent vector to:
\begin{align}
    (\beta_1, ..., \beta_4)^T := (10, -10, 5, 3) \; .
\end{align}
and set $C_1$ and $C_2$ as in \eqref{eq:C1C2_correlation_simulation_study_definition} and for remaining $k \in \{3, ..., 10\}$ sample:
\begin{align}
    P_k &\sim \mathcal{W}(\nu=20, C_1^{-1}) \; , \\
    P_k &\sim \mathcal{W}(\nu=20, C_2^{-1}) 
\end{align}
where $\mathcal{W}$ denotes the Wishart distribution, and set $C_k := Corr(P_k^{-1})$. Similarly as in the Gaussian set up and as proposed by \citet{shen2025wild}, we use the likelihood $y_i | \bm{\beta}, \mathbf{x}_i \sim t_{d=1}(\nu=4, \mu = \mathbf{x}_i^T \bm{\beta}, \sigma=1) $ and chose $t_{d=4}(\nu=5, \bm{\mu}=\bm{0}, \Sigma=10^4 \cdot I_4)$ as prior.

For this set up, the KL-divergence is not available in closed form. However, it can be estimated up to an additive constant with samples from NUTS, if we take $p(\bm{\beta}|\mathbf{x}, \bm{y})$ to be the distribution NUTS converges to (the closest we come to the true posterior). More precisely, we can estimate:
\begin{align}
    KL\big(p(\bm{\beta}|\mathbf{x}, \bm{y}) || q(\bm{\beta}; \bm{\lambda}^*, \bm{\eta}) \big) = E_{\bm{\beta} \sim p(\bm{\beta}|\mathbf{x}, \bm{y})} \big[ \log p(\bm{\beta}, \bm{y} | \mathbf{x}) - \log q(\bm{\beta}; \bm{\lambda}^*, \bm{\eta}) \big] + \underbrace{\big( - \log p(\bm{y} | \mathbf{x} ) \big)}_{=: C}
\end{align}
up to the normalizing constant $C$ (that is fixed for a specific example $k$ and varies between examples). This constant cancels out in the $\Delta KL_{rel}^{(\alpha, k)}$ of Equation \eqref{def:relKL}.

\paragraph{Stepwise D-vine}
As in the Gaussian set up we set marginal/MF parameters $\bm{\lambda}^*$ to the values obtained from NUTS.
Table \ref{tab:simulation_alpha_StudentT_dvine_DeltaKL} again suggests that lower $\alpha$ values give better D-vine approximate posteriors. Specifically $\alpha \leq 0.1$ yields the lowest $\Delta KL_{rel}^{(\alpha, k)}$, where the relative differences to the lowest KL-divergence value (up to the additive constant) are quite low.

\begin{table}[ht]
    \centering
    \caption{$\Delta KL_{rel}^{(\alpha, k)}$ defined in Equation \eqref{def:relKL} between $q$, a stepwise D-vine, and the true posterior obtained with NUTS, by $\alpha$ in the VR-IWAE (rows) and simulated observed data coming from a multivariate Student-t distribution (columns). The minimum $\Delta KL_{rel}^{(\alpha, k)}$ value per Example (column) is marked in bold indicating the best $\alpha$ value per example.}
    \label{tab:simulation_alpha_StudentT_dvine_DeltaKL}
    \scriptsize{
    \begin{tabular}{lrrrrrrrrrr}
        \toprule
        Example & 1 & 2 & 3 & 4 & 5 & 6 & 7 & 8 & 9 & 10 \\
        \midrule
        $\alpha=0.01$ & 9.0e-05 & 2.6e-04 & 2.9e-04 & \textbf{0.0e+00} & 2.4e-04 & 1.5e-04 & 3.0e-04 & 2.6e-04 & 2.5e-04 & 7.7e-05 \\
        $\alpha=0.05$ & \textbf{0.0e+00} & 1.4e-04 & \textbf{0.0e+00} & 2.3e-04 & \textbf{0.0e+00} & 3.4e-04 & \textbf{0.0e+00} & \textbf{0.0e+00} & 3.2e-04 & 2.8e-05 \\
        $\alpha=0.1$ & 3.0e-04 & \textbf{0.0e+00} & 2.2e-05 & 8.6e-05 & 4.7e-04 & \textbf{0.0e+00} & 1.0e-03 & 1.6e-04 & \textbf{0.0e+00} & \textbf{0.0e+00} \\
        $\alpha=0.2$ & 3.8e-04 & 2.0e-05 & 4.5e-04 & 3.9e-06 & 4.4e-04 & 3.0e-04 & 1.8e-04 & 5.7e-05 & 9.0e-05 & 2.8e-04 \\
        $\alpha=0.3$ & 9.4e-04 & 1.3e-04 & 6.9e-04 & 3.9e-04 & 1.2e-03 & 1.3e-03 & 7.5e-04 & 2.5e-04 & 1.4e-04 & 1.9e-04 \\
        $\alpha=0.4$ & 1.6e-03 & 3.1e-04 & 1.5e-03 & 1.1e-03 & 1.9e-03 & 1.8e-03 & 1.5e-03 & 8.5e-04 & 5.3e-04 & 3.8e-04 \\
        $\alpha=0.5$ & 2.3e-03 & 5.2e-04 & 2.4e-03 & 2.5e-03 & 5.3e-03 & 3.8e-03 & 2.2e-03 & 1.0e-03 & 1.1e-03 & 4.9e-04 \\
        $\alpha=0.6$ & 3.4e-03 & 8.1e-04 & 4.0e-03 & 3.7e-03 & 7.0e-03 & 7.5e-03 & 3.4e-03 & 1.4e-03 & 1.8e-03 & 9.6e-04 \\
        $\alpha=0.7$ & 5.0e-03 & 1.4e-03 & 7.1e-03 & 6.6e-03 & 9.6e-03 & 1.5e-02 & 5.1e-03 & 2.2e-03 & 3.3e-03 & 1.2e-03 \\
        $\alpha=0.8$ & 9.8e-03 & 2.2e-03 & 1.5e-02 & 1.1e-02 & 1.9e-02 & 2.8e-02 & 7.2e-03 & 2.8e-03 & 4.4e-03 & 1.7e-03 \\
        $\alpha=0.9$ & 1.6e-02 & 2.4e-03 & 2.3e-02 & 3.3e-02 & 3.3e-02 & 6.2e-02 & 8.9e-03 & 3.8e-03 & 5.8e-03 & 1.9e-03 \\
        \bottomrule
    \end{tabular}
    }%
\end{table}

\paragraph{Gaussian MF}

As for the D-vine, we compare the $\Delta KL_{rel}^{(\alpha, k)}$ defined in Equation \eqref{def:relKL}.
Out of 10 examples a value of $\alpha=0.1$ yields the lowest $\Delta KL_{rel}^{(\alpha, k)}$ 7 times, see Table \ref{tab:simulation_alpha_StudentT_MF_DeltaKL}. In total an $\alpha \leq 0.1$ gives the best approximation in all examples, where the relative differences to the lowest KL-divergence value (up to the additive constant) are quite low.

\begin{table}[ht]
    \centering
    \caption{$\Delta KL_{rel}^{(\alpha, k)}$ defined in Equation \eqref{def:relKL} between $q$, a Gaussian MF, and the true posterior obtained with NUTS, by $\alpha$ in the VR-IWAE (rows) and simulated observed data coming from a multivariate Student-t distribution (columns). The minimum $\Delta KL_{rel}^{(\alpha, k)}$ value per Example (column) is marked in bold indicating the best $\alpha$ value per example.}
    \label{tab:simulation_alpha_StudentT_MF_DeltaKL}
    \scriptsize{
    \begin{tabular}{lrrrrrrrrrr}
        \toprule
        Example & 1 & 2 & 3 & 4 & 5 & 6 & 7 & 8 & 9 & 10 \\
        \midrule
        $\alpha=0.01$ & \textbf{0.0e+00} & 7.1e-04 & \textbf{0.0e+00} & 1.2e-03 & 2.2e-04 & 1.2e-03 & 1.6e-03 & \textbf{0.0e+00} & \textbf{0.0e+00} & \textbf{0.0e+00} \\
        $\alpha=0.05$ & 8.6e-05 & 1.6e-03 & 8.7e-04 & 3.7e-04 & \textbf{0.0e+00} & 1.3e-03 & 9.7e-04 & 8.0e-05 & 3.5e-04 & 3.0e-03 \\
        $\alpha=0.1$ & 3.0e-05 & \textbf{0.0e+00} & 1.1e-03 & \textbf{0.0e+00} & 1.4e-03 & \textbf{0.0e+00} & \textbf{0.0e+00} & 8.8e-04 & 8.2e-04 & 2.0e-03 \\
        $\alpha=0.2$ & 7.7e-04 & 2.2e-03 & 2.9e-03 & 7.8e-04 & 2.2e-03 & 1.6e-03 & 5.7e-03 & 8.5e-04 & 1.1e-03 & 2.5e-03 \\
        $\alpha=0.3$ & 9.9e-04 & 4.1e-03 & 2.6e-03 & 1.6e-03 & 3.5e-03 & 1.7e-03 & 5.9e-03 & 1.2e-03 & 3.7e-03 & 4.5e-03 \\
        $\alpha=0.4$ & 1.5e-03 & 1.1e-02 & 2.6e-03 & 2.5e-03 & 2.9e-03 & 2.3e-03 & 1.1e-02 & 2.6e-03 & 4.1e-03 & 4.8e-03 \\
        $\alpha=0.5$ & 2.0e-03 & 1.6e-02 & 9.1e-03 & 4.8e-03 & 7.2e-03 & 4.7e-03 & 2.7e-02 & 5.4e-03 & 7.9e-03 & 1.2e-02 \\
        $\alpha=0.6$ & 8.4e-03 & 2.6e-02 & 1.6e-02 & 9.1e-03 & 1.7e-02 & 1.2e-02 & 4.2e-02 & 9.5e-03 & 1.7e-02 & 1.6e-02 \\
        $\alpha=0.7$ & 1.7e-02 & 3.8e-02 & 2.7e-02 & 1.7e-02 & 3.8e-02 & 3.1e-02 & 7.9e-02 & 1.5e-02 & 3.0e-02 & 2.1e-02 \\
        $\alpha=0.8$ & 3.2e-02 & 4.7e-02 & 3.8e-02 & 3.0e-02 & 7.3e-02 & 7.1e-02 & 1.2e-01 & 2.3e-02 & 4.3e-02 & 2.5e-02 \\
        $\alpha=0.9$ & 5.7e-02 & 6.0e-02 & 6.4e-02 & 4.4e-02 & 1.1e-01 & 1.2e-01 & 1.7e-01 & 2.7e-02 & 5.9e-02 & 3.5e-02 \\
        \bottomrule
    \end{tabular}
    }%
\end{table}

\subsection{Vine Set Up}
Again, we create a set up in similar fashion of the (Gaussian) needle example in \citet{shen2025wild}, where this time the data comes from a vine copula with Gaussian marginals and Gaussian and Clayton pair copulas. We sample $n=300$ observations from the explanatory variables $(X_1, ..., X_4, Y)^T$:
\begin{align}
    (X_1, ..., X_4)^T &\sim \text{vine copula}\Big(\mathcal{F}, \mathcal{V}, \mathcal{B}(\mathcal{V}), \bm{\eta}\big(\mathcal{B}(\mathcal{V})\big)\Big)\; ,
\end{align}
and then set:
\begin{align}
    Y := \beta_1 X_1 + \beta_2 X_2 + \beta_3 X_3 + \beta_4 X_4 \; .
\end{align}
We set the marginals $\mathcal{F} := (F_1, ..., F_d)$ to be $F_j := \mathcal{N}(0, 1)$ for all $j \in [d]$ and the vine tree structure $\mathcal{V}$ to be a D-vine. For each edge in a tree in $\mathcal{V}$, we sample the pair copula family in $\mathcal{B}(\mathcal{V})$ from a Bernoulli to either be a Clayton copula (with $p=0.6$) or a Gaussian pair copula ($1 - p$). For each Clayton pair copula we sample its parameter $\eta \sim U\big([1,8]\big)$, for each Gaussian pair copula $\eta \sim U\big([0,1]\big)$.

We define the true value of the latent vector to:
\begin{align}
    (\beta_1, ..., \beta_4)^T := (10, -10, 5, 3) \; .
\end{align}
As in the Gaussian set up, we use the likelihood $y_i | \bm{\beta}, \mathbf{x}_i \sim \mathcal{N}(\mathbf{x}_i^T \bm{\beta}, 1)$, which in this case is clearly mis-specified, and chose $\mathcal{N}(\bm{0}, 10^4 \cdot I_4)$ as prior. 
The resulting posterior is a multivariate normal and we can compute the forward KL-divergence $KL\big(p(\bm{\beta}|\mathbf{x}, \bm{y}) || q(\bm{\beta}; \bm{\lambda}^*, \bm{\eta}) \big)$ between $p(\bm{\beta}|\mathbf{x}, \bm{y})$, the posterior obtained by NUTS, and $q(\bm{\beta}; \bm{\lambda}^*, \bm{\eta})$, the approximate posterior to assess the performance of VR-IWAE under different $\alpha$ values.

\paragraph{Stepwise D-vine}
As in the previous two set ups, we set marginal/MF parameters $\bm{\lambda}^*$ to the values obtained from NUTS. 
In Table \ref{tab:simulation_alpha_Vine_dvine_DeltaKL} a $\alpha \leq 0.4$ yields the lowest KL-divergence in all examples. This value is slightly higher than in the set ups before. Still in the majority of examples (7 out of 10) an $\alpha \leq 0.2$ yields the lowest $\Delta KL_{rel}^{(\alpha, k)}$.

\begin{table}[ht]
    \centering
    \caption{$\Delta KL_{rel}^{(\alpha, k)}$ defined in Equation \eqref{def:relKL} between $q$, a stepwise D-vine, and the true posterior obtained with NUTS, by $\alpha$ in the VR-IWAE (rows) and simulated observed data coming from a vine copula with Gaussian marginals and Clayton and Gaussian pair copulas (columns). The minimum $\Delta KL_{rel}^{(\alpha, k)}$ value per Example (column) is marked in bold indicating the best $\alpha$ value per example.}
    \label{tab:simulation_alpha_Vine_dvine_DeltaKL}
    \scriptsize{
    \begin{tabular}{lrrrrrrrrrr}
        \toprule
        Example & 1 & 2 & 3 & 4 & 5 & 6 & 7 & 8 & 9 & 10 \\
        \midrule
        $\alpha=0.01$ & \textbf{0.0000} & 1.4084 & 0.7960 & \textbf{0.0000} & \textbf{0.0000} & 0.0663 & 0.0016 & 4.3844 & 0.3249 & 0.6476 \\
        $\alpha=0.05$ & 0.0876 & 1.4489 & \textbf{0.0000} & 0.0093 & 0.0228 & 0.0726 & \textbf{0.0000} & 0.7574 & 1.3431 & 1.8231 \\
        $\alpha=0.1$ & 0.6093 & 1.4973 & 0.0443 & 0.0015 & 0.0282 & 0.0579 & 0.0327 & 0.3312 & \textbf{0.0000} & \textbf{0.0000} \\
        $\alpha=0.2$ & 0.2182 & 1.5363 & 0.1981 & 0.0031 & 0.3240 & 0.1062 & 0.0041 & \textbf{0.0000} & 1.4434 & 0.4692 \\
        $\alpha=0.3$ & 0.3218 & 1.7829 & 0.2536 & 0.0054 & 0.7951 & \textbf{0.0000} & 0.0159 & 2.5615 & 6.8313 & 1.9801 \\
        $\alpha=0.4$ & 1.0896 & \textbf{0.0000} & 0.9100 & 0.0135 & 1.0328 & 0.4799 & 0.0192 & 4.6749 & 8.2737 & 1.8281 \\
        $\alpha=0.5$ & 3.8322 & 0.7775 & 2.2266 & 0.0255 & 2.3891 & 1.2552 & 0.0965 & 10.4346 & 28.1116 & 8.4882 \\
        $\alpha=0.6$ & 8.7510 & 1.3662 & 3.5056 & 0.0982 & 5.2396 & 3.0301 & 0.1873 & 25.5516 & 63.2145 & 17.7040 \\
        $\alpha=0.7$ & 2.7649 & 9.5905 & 4.9209 & 0.7148 & 1.8280 & 5.6362 & 0.4041 & 42.3457 & 111.8122 & 34.4603 \\
        $\alpha=0.8$ & 21.2089 & 17.4623 & 7.7865 & 1.8747 & 1.8287 & 6.1535 & 0.5688 & 56.7802 & 169.4176 & 49.1580 \\
        $\alpha=0.9$ & 29.8824 & 21.4766 & 22.0038 & 4.5072 & 47.9719 & 8.3250 & 0.9030 & 81.0791 & 238.8662 & 84.9565 \\
        \bottomrule
    \end{tabular}
    }%
\end{table}

\paragraph{Gaussian MF}
For the mis-specified case we have no statement in the style of \citet{margossian2024variational} of whether the correct posterior mean is recovered by the Gaussian MF or not. For this reason it makes sense to assess the KL-divergence. However, for several examples several $\alpha$ yield the same KL-divergence value.\footnote{This could indicate that even though the statement of \citet{margossian2024variational} does not hold for the current, mis-specified case, the posterior mean obtained by NUTS might be recovered by the Gaussian MF in this set up.} For this reason we instead assess the mean relative RMSE of the standard deviations averaged over $d$ as defined in Equation \eqref{eq:mean_rel_RMSE_stds}.

Table \ref{tab:simulation_alpha_Vine_MF_rel_RMSE_stds} reports that a value of $\alpha \leq 0.2$ yields the lowest VR-IWAE loss for all examples reported.

\begin{table}[ht]
    \centering
    \caption{Mean relative RMSE of standard deviation \eqref{eq:mean_rel_RMSE_stds} of $q$, a Gaussian MF, and the true posterior obtained with NUTS, by $\alpha$ in the VR-IWAE (rows) and simulated observed data coming from a vine copula with Gaussian marginals and Clayton and Gaussian pair copulas (columns). The minimum value per Example (column) is marked in bold indicating the best $\alpha$ value per example.}
    \label{tab:simulation_alpha_Vine_MF_rel_RMSE_stds}
    \scriptsize{
    \begin{tabular}{lrrrrrrrrrr}
        \toprule
        Example & 1 & 2 & 3 & 4 & 5 & 6 & 7 & 8 & 9 & 10 \\
        \midrule
        $\alpha=0.01$ & \textbf{0.1288} & \textbf{0.1356} & 0.1245 & 0.0762 & 0.3238 & 0.1872 & 0.2329 & 0.0650 & \textbf{0.1316} & \textbf{0.2368} \\
        $\alpha=0.05$ & 0.1804 & 0.1598 & \textbf{0.1143} & 0.0752 & 0.3049 & 0.2171 & \textbf{0.2084} & 0.0579 & 0.1684 & 0.2538 \\
        $\alpha=0.1$ & 0.1441 & 0.1447 & 0.1632 & 0.0861 & \textbf{0.2875} & \textbf{0.1165} & 0.2638 & \textbf{0.0519} & 0.1506 & 0.3116 \\
        $\alpha=0.2$ & 0.1798 & 0.1618 & 0.1442 & \textbf{0.0544} & 0.3399 & 0.1891 & 0.2713 & 0.1058 & 0.1784 & 0.3075 \\
        $\alpha=0.3$ & 0.1642 & 0.1793 & 0.1613 & 0.1088 & 0.3360 & 0.1865 & 0.3046 & 0.1648 & 0.2274 & 0.3230 \\
        $\alpha=0.4$ & 0.2232 & 0.2288 & 0.2155 & 0.1628 & 0.3766 & 0.2459 & 0.3860 & 0.2227 & 0.2794 & 0.3673 \\
        $\alpha=0.5$ & 0.2851 & 0.3040 & 0.2690 & 0.1944 & 0.4379 & 0.2854 & 0.4454 & 0.2463 & 0.2928 & 0.4455 \\
        $\alpha=0.6$ & 0.3578 & 0.3573 & 0.3936 & 0.2626 & 0.5079 & 0.3913 & 0.5279 & 0.3320 & 0.3865 & 0.5140 \\
        $\alpha=0.7$ & 0.4706 & 0.4713 & 0.4609 & 0.3401 & 0.5724 & 0.4908 & 0.5755 & 0.3869 & 0.4571 & 0.5690 \\
        $\alpha=0.8$ & 0.5335 & 0.5876 & 0.5769 & 0.4196 & 0.6605 & 0.6223 & 0.6207 & 0.4456 & 0.5276 & 0.6268 \\
        $\alpha=0.9$ & 0.5814 & 0.6551 & 0.6396 & 0.4888 & 0.7255 & 0.7081 & 0.6482 & 0.4744 & 0.5670 & 0.6639 \\
        \bottomrule
    \end{tabular}
    }%
\end{table}

\subsection{Discussion of Simulation Study Results}
In total, we ran 10 examples for 3 different data generating set ups with each 11 different $\alpha \in [0,1]$ for the stepwise D-vine and a Gaussian MF as variational distributions. For all examples we found that $\alpha$ values of at most 0.4 in the VR-IWAE loss give the best approximation to the true posterior with both the Gaussian MF and the stepwise D-vine. For all examples, but 2 examples we found that an $\alpha$ of at most 0.2 gives the best approximation.

The results are consistent in the sense that the $\Delta KL_{rel}^{(\alpha, k)}$ value for the stepwise D-vine and the mean relative RMSE of the standard deviations for the Gaussian MF exhibit a certain degree of monotonicity: the closer the current $\alpha$ is to $\alpha^{min} := \arg \min_{\alpha}\Delta KL_{rel}^{(\alpha, k)}$, the lower $\Delta KL_{rel}^{(\alpha, k)}$ value and similarly for the mean relative RMSE of the standard deviations.

In Section \ref{sec:Renyi-divergence_VI} we elaborate how $\alpha$ closer to 0 gives higher variance in the VR-IWAE. For this reason we conclude from the simulation study that an $\alpha=0.1$ is an overall good choice that gives good approximation for both the Gaussian MF and stepwise D-vine.

\section{Competitor Models}\label{app:competitor_models}

\subsection{Gaussian Copula VI (GC-VI)}\label{app:GC-VI}
Gaussian copula VI (GC-VI) is proposed by \citet{tran2015copula}, where the mean-field (MF) variational distribution is augmented with a vine copula to capture the dependence among the latent variables and a black-box VI approach is obtained. The authors assume the vine tree structure and the pair copula families to be fixed. 
In experiments they attempt to learn the tree structure and pair copula families from synthetic data of the latent variables generated from an estimate of the variational distribution. However, for a good reason\footnote{This is a very unrealistic setting: Either there is a high-quality variational posterior available from which we could generate synthetic data and meaningfully learn the vine tree structure. However, then there is no need to estimate a vine copula variational posterior. Or we only have a low-quality variational posterior available. The synthetic data generated from the latter will be useless for vine structure learning.}  the authors do not further explore this route. \citet{tran2015copula} use the score estimator to re-express the ELBO gradient w.r.t. the MF parameters as an expectation over the full variational distribution. In the case of differentiable latent variables, they additionally apply the reparametrization trick for the ELBO gradient w.r.t. the MF and the pair copula parameters. The proposed copula VI algorithm alternates between optimizing the MF parameters $\bm{\lambda}$ until convergence while the pair copula parameters $\bm{\eta}$ are held fixed, and optimizing $\bm{\eta}$ until convergence with $\bm{\lambda}$ fixed, until convergence of the whole routine.

Neither source code nor full details or parameters of the experiments (e.g. the vine tree structure, pair copula types or threshold value for parameter convergence, etc.) are made available by the authors, our request by email was not answered. Therefore we need to fill the gaps, where we are as faithful to \citet{tran2015copula} as possible.

We set the vine tree structure to be a D-vine and pair copula families to be Gaussian. We take 2 optimization rounds on MF and D-vine parameters. We specify the threshold on change of the parameter values as stopping criterion for the optimization to be ADVI's default, i.e. $0.001$ \citep{kucukelbir2015automatic}.\footnote{Detecting convergence of variational parameter estimates is hard: Setting the value for the threshold on the parameter change is difficult and has a big impact on the quality of the variational posterior approximation.} \citet{tran2015copula} use 1024 MC samples to estimate the ELBO gradient in each step. We find that drawing 1024 MC samples are slows the optimization down to an impractical degree. For this reason we reduce the number of MC samples drawn for the lower bound gradient estimate to 10.

\subsection{Masked Auto-Regressive Flow Model (MAF)}
A normalizing flow model for $\mathbf{X} \in \mathbb{R}^d$ consists of a simple base distribution $\pi$ of $\mathbf{U}$ and an invertible transformation $T$, where $\mathbf{x} := T(\bm{\mathbf{u}})$ and both $T$ and $T^{-1}$ are differentiable. Then with the change of variables:
\begin{align}
    p_{\mathbf{X}}(\mathbf{x}) = \pi\big(T^{-1}(\mathbf{x})\big) \cdot | \det J_{T^{-1}}(\mathbf{x})| \; .
\end{align}
Generally, $T$ is a composition of a finite number of simple transformations:
\begin{align}
    T := T_K \circ \dots \circ T_1 \; ,
\end{align}
where with $\mathbf{z}_0 := \mathbf{u}$ and $\mathbf{z}_K := \mathbf{x}$. The forward evaluation yields:
\begin{align}
    \mathbf{z}_k := T_k(\mathbf{z}_{k-1})
\end{align}
and consequently:
\begin{align}
    \log | \det J_T(\mathbf{z}_0)| = \log | \prod_{k=1}^K \det J_{T_k}(\mathbf{z}_{k-1})| = \sum_{k=1}^K \log | \det J_{T_k}(\mathbf{z}_{k-1}) | \; .
\end{align}

Masked auto-regressive flows (MAFs) proposed by \citet{papamakarios2017masked} belong to the class of auto-regressive flows. In an auto-regressive flow a transformation $T_k$ takes on the form:
\begin{align}
    z_j' := \tau \big( z_j; \bm{h}_j \big) \quad \text{where} \quad \bm{h}_j := c_j(\mathbf{z}_{<j}) \; , \quad  j \in [d]
\end{align}
where we left out the index $k$ of $T_k$ to ease understanding and use the notation $\mathbf{z}'$ to defined the output of $T_k$, i.e. $\mathbf{z}' := T_k (\mathbf{z})$. Here $\tau$ is the so-called transformer and $c_j$ the $j$th conditioner. The conditioner is auto-regressive since $c_j$ depends on $\mathbf{z}_{<j}$. It is easy to check that this auto-regressive form yields a triangular $J_{T_k}$, for which $\log |\det J_{T_k} (\mathbf{z}) |$ can be computed efficiently in linear time. 
% \citet{huang2018neural} set $\tau$ to a neural network and obtain neural auto-regressive flows. 

A masked conditioner \citep{papamakarios2017masked} is an auto-regressive conditioner and was introduced to increase speed of evaluating and inverting $T_k$. Having $d$ separate conditioner models $c_j(\mathbf{z}_{<j}), \; j \in [d]$ each with separate parameters scales poorly with $d$. Instead, the idea of a masked conditioner is to share parameters across conditioners by having a single conditioner model $c$ and obtain $(\bm{h}_l, ..., \bm{h}_d)$ in a single forward pass $(\bm{h}_l, ..., \bm{h}_d) = c(\mathbf{z})$. Typically, $c$ is a feed-forward neural network (NN) that satisfies an auto-regressive structure. The latter is achieved by starting off with the NN and removing all paths from $z_j$ to $(\bm{h}_l, ..., \bm{h}_j)$ from it. This is done by masking, namely by multiplying the according NN weights with 0 \citep{papamakarios2021normalizing}.

To obtain MAFs, \citet{papamakarios2017masked} combine masked conditioners with affine transformers:
\begin{align}
    \tau \big( z_j; \bm{h}_j \big) := \exp(\alpha_j) z_j + \beta_j \quad \text{with} \quad (\alpha_j, \beta_j)^T := \bm{h}_j
\end{align}

In our experiments we use the MAF implementation of \verb|Zuko| \citep{zuko_pypi} with $K=20$ transformations $T_k, \; k \in [K]$ and Adam as optimizer \citep{kingma2014adam}.

In our experiments we also tried the more flexible neural spline flow models (NSFs) proposed by \citet{durkan2019neural}. However, they did not give sensible approximations, so we switched to the simpler MAFs.

\section{Simulated Examples: Results}\label{app:simulated_examples_results}

\begin{figure}[H]
    \centering
    \includegraphics[width=\linewidth]{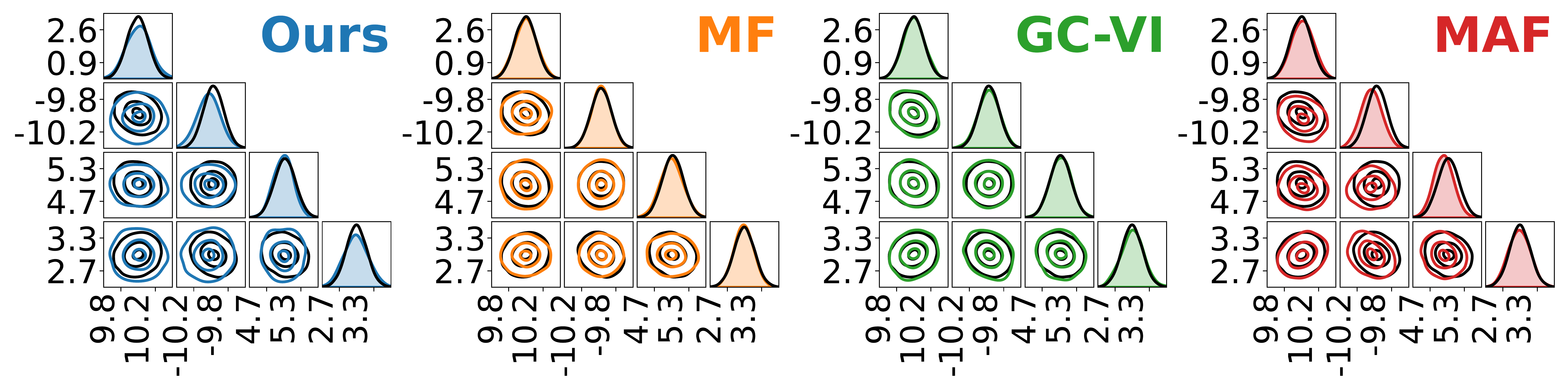}
    \includegraphics[width=\linewidth]{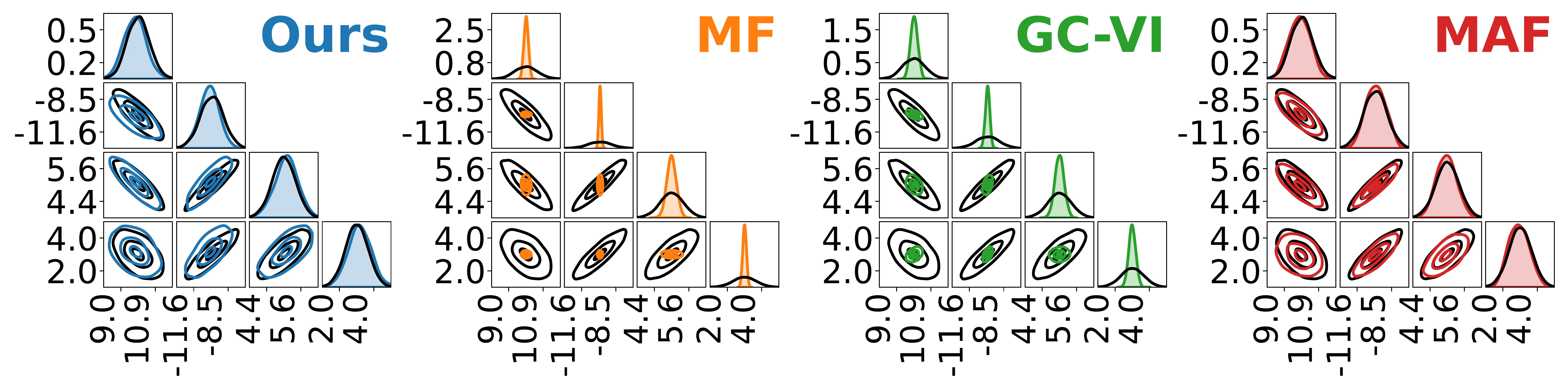}
    \caption{Contour plots of samples from NUTS (black) regarded as ground truth and variational approximations obtained with stepwise VI with vines (blue), MFVI (orange), GC-VI (green) and MAF (red) for the independence example (top) and the needle example (bottom) of Section \ref{sec:simulated_examples_results}.}
    \label{fig:example_independence_pairsplots}
\end{figure}

\section{GP Example: Results}\label{app:GP_results}

\begin{figure}[H]
\centering\includegraphics[width=0.5\linewidth]{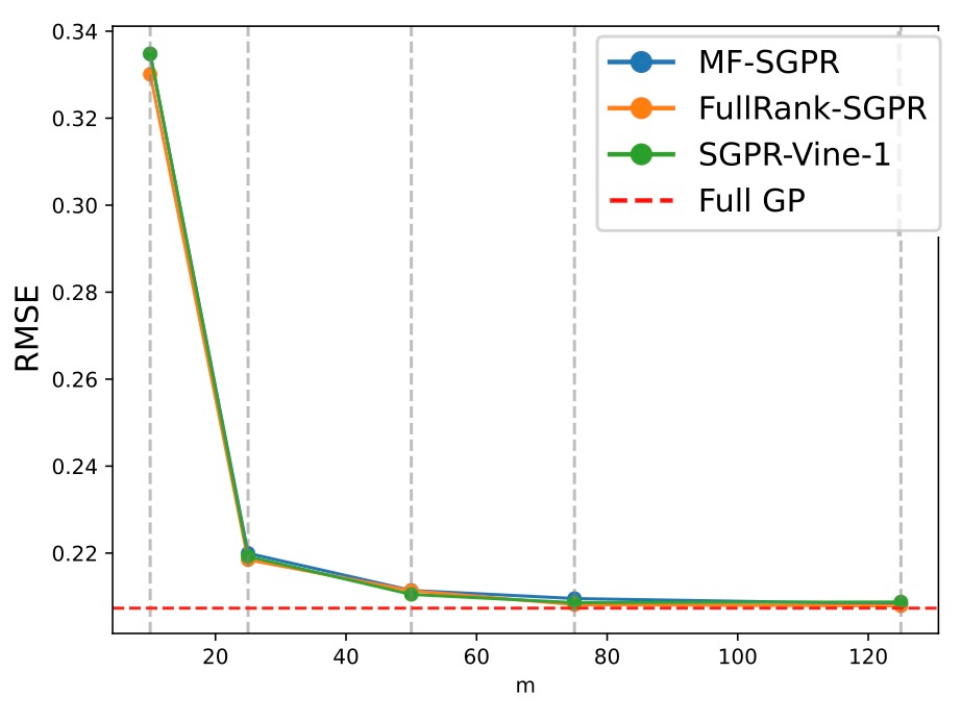}
    \caption{RMSE across a range of inducing point values for the \texttt{pumadyn32nm} dataset. We compare the mean-field SGPR (MF-SGPR), a full rank SGPR (FullRank-SGPR), and the vine at tree level $\tau=1$ to a full (non-sparse) GP fit (red dashed line).}
    \label{fig:GP_results_rmse}
\end{figure}

% \paragraph{Lotka-Volterra}

% We consider a two-dimensional Lotka-Volterra predator-prey system following \citet{carpenter2018predator}. The latent state $\mathbf{z}(t) = (u(t), v(t))^\top$ evolves according to the ODEs:
% \begin{align}
%     \frac{du(t)}{dt} &= (\alpha - \beta v(t))\,u(t), \\
%     \frac{dv(t)}{dt} &= (-\gamma + \delta u(t))\,v(t),
% \end{align}
% where $\boldsymbol{\theta} = (\alpha, \beta, \gamma, \delta)^\top \in \mathbb{R}_+^4$. Given fixed initial condition $\mathbf{z}(0) := (30,5)^T$ and noise parameter $\sigma := 0.3$, we simulate $n=50$ noisy observations $(\bm{y}_t)_{t \in [n]}$ as follows: First, we numerically solve the ODEs to obtain the trajectory $\mathbf{z}(t)$. Then we set $\bm{y}_t := \mathbf{z}(t) \cdot \exp(\sigma \bm{\epsilon}_t)$ where we draw $\bm{\epsilon}_t \sim \mathcal{N}(\bm{0}, I_2)$ i.i.d. for $t \in [n]$.

% We set a $\mathcal{N}(\bm{0}, I_4)$ prior on $\log(\bm{\theta})$ and use the likelihood $\bm{y}_t | \bm{\theta} \sim \mathrm{LogNormal}\Big(\log \big( \mathbf{z}(t; \boldsymbol{\theta}) \big), \sigma I_2\Big)$.

\section{Compute Resources}
The independence and the needle example in Section \ref{sec:simulated_examples_results} were conducted on an Apple Macbook Pro with macOS Sequoia 15.6.1, Apple M2 Pro chip, 16 GB RAM and 10 cores. We used \verb|Python| 3.12.2 and \verb|conda| 24.7.1 for virtual environment managing.

The Gaussian Process examples were performed on an Apple Macbook Pro with macOS Sequoia 15.1.1, Apple M4 chip, 16GB RAM and 10 cores on \verb|Python| version was 3.10.19.

\end{document}